%% file: main.tex
\documentclass[letterpaper]{article} 
\usepackage{aaai23}  
\usepackage{times}  
\usepackage{helvet}  
\usepackage{courier}  
\usepackage[hyphens]{url}  
\usepackage{graphicx} 
\usepackage{natbib}  
\usepackage{caption} 
\usepackage{algorithm}
\usepackage{algorithmic}

\usepackage{newfloat}
\usepackage{listings}
\DeclareCaptionStyle{ruled}{labelfont=normalfont,labelsep=colon,strut=off} 
\floatstyle{ruled}
\newfloat{listing}{tb}{lst}{}
\floatname{listing}{Listing}
\usepackage{float}
\usepackage{multicol}
\usepackage{microtype}
\usepackage{graphicx}
\usepackage{subfigure}
\usepackage{booktabs} 
\usepackage{epsfig,epsf,fancybox}
\usepackage{amsmath}
\usepackage{mathrsfs}
\usepackage{amssymb}
\usepackage{color}
\usepackage{xcolor}
\usepackage{multirow}
\usepackage{paralist}
\usepackage{verbatim}
\usepackage{algorithm}
\usepackage{algorithmic}
\usepackage{galois}
\usepackage{boxedminipage}
\usepackage{accents}
\usepackage{stmaryrd}
\usepackage[bottom]{footmisc}
\definecolor{light-gray}{gray}{0.85}
\usepackage{colortbl}
\usepackage{makecell}
\usepackage{hhline}
\usepackage{bbm}
\usepackage{mathtools}
\usepackage{xcolor}         
\usepackage{bbold}
\usepackage[utf8]{inputenc} 
\usepackage[T1]{fontenc}    
\usepackage{url}            
\usepackage{booktabs}       
\usepackage{amsfonts}       
\usepackage{nicefrac}       
\usepackage{microtype}      
\usepackage{MnSymbol, graphicx}
\usepackage{xr}    
\usepackage{hyperref}
\usepackage{enumitem}
\usepackage{nameref}

\newtheorem{theorem}{Theorem}
\newtheorem{corollary}[theorem]{Corollary}
\newtheorem{lemma}[theorem]{Lemma}

\newtheorem{remark}[theorem]{Remark}
\newtheorem{assumption}[theorem]{Assumption}

\newcommand{\norm}[1]{\left\lVert#1\right\rVert}
\newcommand\inner[2]{\left\langle #1, #2 \right\rangle}

\newcommand{\EE}{\mathbb{E}}

\newcommand{\argmax}{\mathop{\rm argmax}}

\newcommand{\ceil}[1]{\lceil {#1} \rceil}

\newcommand{\ba}{\begin{array}}
\newcommand{\ea}{\end{array}}

\makeatletter
\newcommand*{\rom}[1]{\expandafter\@slowromancap\romannumeral #1@}

\title{Provably Efficient Primal-Dual Method for CMDPs with Non-stationary Objectives and Constraints}
\author {
    Yuhao Ding,
    Javad Lavaei 
}
\affiliations {
    University of California, Berkeley, 
    \{yuhao\_ding, lavaei\}@berkeley.edu
}

\usepackage{bibentry}

\begin{document}

\maketitle

\begin{abstract}
We consider primal-dual-based reinforcement learning (RL) in episodic constrained Markov decision processes (CMDPs) with non-stationary objectives and constraints, which plays a central role in ensuring the safety of RL in time-varying environments. In this problem, the reward/utility functions and the state transition functions are both allowed to vary arbitrarily over time as long as their cumulative variations do not exceed certain known variation budgets. Designing safe RL algorithms in time-varying environments is particularly challenging because of the need to integrate the constraint violation reduction, safe exploration, and adaptation to the non-stationarity. To this end, we identify two alternative conditions on the time-varying constraints under which we can guarantee the safety in the long run. We also propose the \underline{P}eriodically \underline{R}estarted \underline{O}ptimistic \underline{P}rimal-\underline{D}ual \underline{P}roximal \underline{P}olicy \underline{O}ptimization (PROPD-PPO) algorithm that can coordinate with both two conditions. Furthermore, a dynamic regret bound and a constraint violation bound are established for the proposed algorithm in both the linear kernel CMDP function approximation setting and the tabular CMDP setting under two alternative  conditions. This paper provides the first provably efficient algorithm for non-stationary CMDPs with safe exploration.
\end{abstract}

\input{files/Intro}
\input{files/Related_work}

\input{files/Prelim}

\input{files/Method}

\input{files/Main_results}

\input{files/safe_exp_non}

\input{files/Conclusion}


\bibliography{main_bib}
\clearpage

\input{files/appe_notation}
\input{files/appe_linearCMDP}
\input{files/appe_PolicyUpdate}
\input{files/appe_policy_eva}
\input{files/appe_regret}
\input{files/appe_constraint}

\input{files/appe_regret_uniform_slater}

\input{files/appe_tabular_local_budget}

\input{files/appe_tabular_uniform_slater}
\input{files/appe_model_pred_error}
\input{files/appe_other}

\end{document}

%% file: files/Intro.tex
\section{Introduction}\label{sec:Intro}


Safe reinforcement learning (RL) studies how an agent
learns to maximize its expected total reward by interacting
with an unknown environment over time while dealing with restrictions/constraints arising from real-world problems \cite{amodei2016concrete,dulac2019challenges,garcia2015comprehensive}. A standard approach for modeling the safe RL is based on Constrained Markov Decision Processes (CMDPs) \cite{altman1999constrained}, where one seeks
to maximize the expected total reward under a safety-related constraint
on the expected total utility.

While classical safe RL and CMDPs assume that an agent interacts with a time-invariant (stationary) environment, both the reward/utility functions and transition kernels can be time-varying for many real-world safety-critical applications. For example, in autonomous driving \cite{sallab2017deep}, it is essential to guarantee the safety, such as collision-avoidance and traffic rules, while handling time-varying conditions related to weather and traffic. Similarly, in most safety-critical human-computer interaction applications, e.g., automated medical care, human behavior changes over time. In such scenarios, if the automated system is not adapted to take such changes into account, then the system could quickly violate the safety constraint and incur a severe loss \cite{chandak2020optimizing,moore2014reinforcement}. 
Despite the importance of non-stationary safe RL problems, the literature lacks provably efficient algorithms and theoretical results. 

In this work, we formulate a general non-stationary safe exploration  problem as an episodic CMDP in which the transition model is unknown and non-stationary, the reward/utility feedback after each episode is bandit and non-stationary, and the variation budget is known. The goal is to design an algorithm that can perform a non-stationary safe exploration, that is, to adaptively explore the unknown and time-varying environment and learn to satisfy time-varying constraints in the long run. 

The safe exploration in non-stationary CMDPs is more challenging since the utilities and
dynamics are time-varying and unknown a priori. Thus, it is difficult/impossible to guarantee a small/zero constraint violation without knowing how CMDPs will change. Previous constraint violation analyses \cite{ding2021provably, liu2021learning} strongly rely on the conditions of having the same transition dynamics and rewards over all episodes, which are not applicable to non-stationary CMDPs.
In view of the aforementioned challenges, we propose a new primal-dual method and develop novel techniques to decouple the optimality gap and the constraint violation. Our main contributions are summarized below:

\begin{itemize}[leftmargin=*]
\item We identify two alternative conditions on the time-varying constraints under which we can guarantee the safety in the long run. The first assumption requires the knowledge of the local variation budgets of the constraint for each epoch, while the second assumption needs the strict feasibility of the constraint at each episode and the knowledge of a uniform strict feasibility threshold.

\item We develop a new periodically restarted policy-based primal-dual method, which can coordinate with both two conditions, for  general non-stationary CMDP problems.

\item We study the proposed algorithm under two alternative  conditions that require  different amounts of  knowledge on the constraints. Our results are summarized in Table \ref{table: comparison} and our method is the first provably efficient  algorithm for non-stationary CMDPs with safe exploration.
\end{itemize}

\begin{table*}[ht]
\centering
{\footnotesize
\renewcommand{\arraystretch}{1.5}
 \begin{tabular}{c| c|c |c } 
 \hline
 \textbf{Setting} &  \textbf{Assumption}  & \textbf{Dynamic regret} & \textbf{Constraint violation}  \\\hline\hline
   Tabular  & $B_{g,\mathcal{E}}$, $B_{\mathbb{P},\mathcal{E}}$   & $\widetilde{\mathcal{O}}\left( |\mathcal{S}|^{\frac{2}{3}} |\mathcal{A}|^{\frac{1}{3}} H^{\frac{5}{3}} M^{\frac{1+\rho}{2}} (B_\Delta +B_\ast)^{\frac{1}{3}}\right)$  &   $\widetilde{\mathcal{O}}\left( |\mathcal{S}|^{\frac{2}{3}} |\mathcal{A}|^{\frac{1}{3}} H^{\frac{5}{3}} M^{\frac{2-\rho}{2}}  (B_\Delta +B_\ast)^{\frac{1}{3}}\right)$  \\\hline
    Tabular  & $\gamma$ &   $\widetilde{\mathcal{O}}\left( \gamma^{-1} |\mathcal{S}|^{\frac{2}{3}} |\mathcal{A}|^{\frac{1}{3}}   H^{\frac{5}{2}} M^{\frac{2}{3}} (B_\Delta+B_\star)^{\frac{1}{3}}  \right)$   &  $\widetilde{\mathcal{O}}\left( \gamma^{-1} |\mathcal{S}|^{\frac{2}{3}} |\mathcal{A}|^{\frac{1}{3}}   H^{\frac{5}{2}} M^{\frac{2}{3}} (B_\Delta+B_\star)^{\frac{1}{3}}  \right)$   \\ \hline
    \hline
 Linear kernel   &  $B_{g,\mathcal{E}}$, $B_{\mathbb{P},\mathcal{E}}$  &   $   \widetilde{\mathcal{O}}\left( d^{\frac{9}{8}} H^{\frac{5}{2}} M^{\frac{3}{4}} (\sqrt{d}B_\Delta +B_\ast)^{\frac{1}{3}}\right)$  &  $\widetilde{\mathcal{O}}\left( d^{\frac{9}{8}} H^{\frac{5}{2}} M^{\frac{3}{4}} (\sqrt{d}B_\Delta +B_\ast)^{\frac{1}{3}}\right)$   \\\hline
     Linear kernel  & $\gamma$ &  $\widetilde{\mathcal{O}}\left( \gamma^{-1}d^{\frac{9}{8}} H^{\frac{5}{2}} M^{\frac{3}{4}} (\sqrt{d}B_\Delta +B_\ast)^{\frac{1}{3}}\right)$   &   $\widetilde{\mathcal{O}}\left( \gamma^{-1}d^{\frac{9}{8}} H^{\frac{5}{2}} M^{\frac{3}{4}} (\sqrt{d}B_\Delta +B_\ast)^{\frac{1}{3}}\right)$    \\\hline 
 \end{tabular}
 \caption{We summarize the dynamic regrets and constraint violations obtained in this paper for tabular and linear kernel CMDPs under different assumptions. Here, $B_{g,\mathcal{E}}$ and $B_{\mathbb{P},\mathcal{E}}$ are the local variation budgets for the constraints and are defined in Assumption \ref{ass: local budget}, $\gamma$ is the strict feasibility threshold of the constraints and is defined in Assumption \ref{ass: Feasibility}, $H$ is the horizon of each episode, $M$ is the total number of episodes, $d$ is the dimension of the feature mapping, $|\mathcal{S}|$ and $|\mathcal{A}|$ are the cardinalities of the state and action spaces,
 and $B_\Delta, B_*$ are the variation budgets defined in \eqref{eq: variation budget diamond+P} and \eqref{eq: variation budget star}. There is a trade-off controlled by $\rho\in[\frac{1}{3}, \frac{1}{2}]$ between the dynamic regret and constraint violation for the tabular CMDP under Assumption \ref{ass: local budget}. }
 \label{table: comparison}
 }
\end{table*}

%% file: files/Related_work.tex
\subsection{Related work}\label{sec:related}
\textbf{Non-stationary RL.}
Non-stationary RL has been mostly studied in the unconstrained setting \cite{jaksch2010near, auer2019adaptively, ortner2020variational, domingues2021kernel,mao2020near,zhou2020nonstationary,touati2020efficient,fei2020dynamic,zhong2021optimistic,cheung2020reinforcement,wei2021non}.
Our work is related to  policy-based methods for non-stationary RL since the optimal solution of CMDP is usually a stochastic policy \cite{altman1999constrained} and thus a policy-based method is preferred. When the variation budget is known a prior,
\cite{fei2020dynamic} propose the first policy-based method for non-stationary RL, but they assume stationary transitions and adversarial full-information rewards in the tabular setting. \cite{zhong2021optimistic} extends the above results to a more  general setting where both the transitions and rewards can vary over episodes.
To eliminate the assumption of having prior knowledge on variation budgets,
\cite{wei2021non} recently outline that an adaptive restart approach can be used to convert any upper-confidence-bound-type stationary RL algorithm to a dynamic-regret-minimizing algorithm. 
Beyond the non-stationary unconstrained RL, \cite{qiu2020upper} consider the online CMDPs where the reward is adversarial but the transition model is fixed and the constraints are stochastic over episodes. In summary, the above papers only consider the non-stationarity in the objective and may not work for the more general safe RL problems where there is also time-varying constraints.


\textbf{CMDP.} The study of RL algorithms for CMDPs has received
considerable attention due to the safety requirement \cite{altman1999constrained, paternain2019safe,yu2019convergent,dulac2019challenges,garcia2015comprehensive}. 
Our work is closely related to Lagrangian-based CMDP algorithms with optimistic
policy evaluations \cite{efroni2020exploration,singh2020learning,ding2021provably,liu2021learning,qiu2020upper}.
In particular, \cite{efroni2020exploration,singh2020learning}
leverage upper confidence bound (UCB) bonus on fixed reward/utility and transition probability to propose sample efficient algorithms for tabular CMDPs. \cite{ding2021provably} generalize the above results to the linear kernel CMDPs. Under some mild conditions and additional computation cost, \cite{liu2021learning} propose two algorithms to learn policies with a zero or bounded constraint violation for CMDPs. Beyond the stationary CMDP, \cite{qiu2020upper}
consider the online CMDPs where only the rewards in objective can vary over episodes.
In contrast, our work focuses on a more general and realistic safe RL setting where the dynamics and rewards/utilities can all change over episodes, and thus we significantly extend the existing results.


Due to space restrictions, we introduce the notations in Section \nameref{sec:notation} of the appendix.

%% file: files/Prelim.tex
\section{Preliminaries}\label{sec:prelim}
\textbf{Model.}
In this paper, we study safe RL in non-stationary environments via episodic CMDPs with adversarial bandit-information reward/utility feedback and unknown adversarial transition kernels. At each episode $m$, a  CMDP is defined by the state space $\mathcal{S}$, the action space $\mathcal{A}$, the fixed length of each episode $H$, a collection of transition probability measure $\{\mathbb{P}_h^m\}_{h=1}^H$, a collection of reward functions $\{r_h^m\}_{h=1}^H$, a collection of utility functions $\{g_h^m\}_{h=1}^H$ and the  constraint offset $b_m$. 
We assume that $\mathcal{S}$ is a measurable space with a possibly infinite number of elements, and that $\mathcal{A}$ is a finite set.  In addition, we assume $r_h^m: \mathcal{S}\times \mathcal{A} \rightarrow [0,1]$ and $g_h^m: \mathcal{S}\times \mathcal{A} \rightarrow [0,1]$ are deterministic reward and utility functions.
Our analysis readily generalizes to the setting where the reward/utility functions are random.  In this paper, we focus on a bandit setting where the agent only observes the values of reward and utility functions, $r_h^m(x_h^m, a_h^m)$ and $g_h^m(x_h^m, a_h^m)$ at the visited state-action pair $(x_h^m, a_h^m)$.  To avoid triviality, we take $b_m\in(0,H]$ and assume that it is known to the agent.


Let the policy space $\Delta(\mathcal{A}|\mathcal{S}, H)$ be $\{ \{\pi_h(\cdot|\cdot)\}_{h=1}^H: \pi_h(\cdot|s) \in \Delta(\mathcal{A}), \forall x\in \mathcal{S}, h\in[H] \}$, where $\Delta(\mathcal{A})$ denotes a probability simplex over the action space. Let $\pi^m \in \Delta(\mathcal{A}|\mathcal{S}, H)$ be a policy taken by the agent at episode $m$, where  $\pi^m_h(\cdot|x_h^m): \mathcal{S} \rightarrow \mathcal{A}$ is the action that the agent takes at state $x_h^m$. For simplicity, we assume the initial state $x_1^m$ to be fixed as $x_1$ in different episodes. 
The episode terminates at state $x_H^m$ in which no control action is needed and both reward and utility functions are equal to zero.

Given a policy $\pi\in \Delta(\mathcal{A}|\mathcal{S},H)$ and the episode $m$, the value function $V_{r,h}^{\pi,m}$ associated with the reward function $r$ at step $h$ in episode $m$ is the expected value of the total reward, 
$V_{r,h}^{\pi,m}(x)=\EE_{\pi, \mathbb{P}^m} \left[ \sum_{i=h}^H r_i^m(x_i,a_i) |x_h=x\right]$, 
for all $ x\in \mathcal{S}$ and $h\in [H]$,
where the expectation $\EE_{\pi, \mathbb{P}^m}$ is taken over the random state-action sequence $\{(x_i^m, a^m_i)\}_{i=h}^H$, the action $a^m_h$ follows the policy $\pi^m_h(\cdot|x_h^m)$, and the next state $x_{h+1}$ follows the transition dynamics $\mathbb{P}_h^m(\cdot|x_h^m, a_h^m)$. 

The action-value function is defined as
$Q_{r,h}^{\pi,m}(x,a)=\EE_{\pi, \mathbb{P}^m}\left[\sum_{i=h}^H r_i^m(x_i^m,a_i^m) |x_h^m=x, a_h^m=a \right]$,
for all $ x\in \mathcal{S}, a\in \mathcal{A}$ and $h\in [H]$. Similarly, we define the value function $ V_{g,h}^{\pi,m}: \mathcal{S} \rightarrow \mathbb{R}$ and the action-value function $ Q_{g,h}^{\pi,m}: \mathcal{S} \times \mathcal{A} \rightarrow \mathbb{R}$  associated with the utility function $g$. For brevity, we use the symbol $\diamond$ to denote $r \text{ or } g$. 
and take the shorthand $\mathbb{P}_h^m V_{\diamond,h}^{\pi,m}(x,a)\coloneqq \EE_{x^\prime \sim  \mathbb{P}_h^m (\cdot|x,a)} \left[V_{\diamond,h+1}^{\pi,m}(x^\prime) \right]$. The Bellman equation associated with a policy $\pi$ is given by
\begin{subequations}\label{eq: belmman equation}
\begin{eqnarray}
& Q_{\diamond,h}^{\pi,m}(x,a)=(\diamond_h ^m+ \mathbb{P}_h^m V_{\diamond,h+1}^{\pi,m}) (x,a), \\
& V_{\diamond,h}^{\pi,m}(x)= \inner{Q_{\diamond,h}^{\pi,m} (x, \cdot)}{ \pi_h(\cdot|x)}_\mathcal{A},
\end{eqnarray}
\end{subequations}
for all $(x,a) \in \mathcal{S} \times \mathcal{A}$, where $\inner{\cdot}{\cdot}_\mathcal{A}$ denotes the inner product over $\mathcal{A}$ and we will omit the subscript  $\mathcal{A}$ in the sequel when it is clear from the context.

\textbf{Constrained MDP.}
In constrained MDPs, the agent aims to approximate the optimal non-stationary
policy by interacting with the environment. In each episode $m$, the agent aims to maximize the expected total reward while satisfying the constraints on the expected total utility
\begin{align} \label{eq: CMDP at episode m}
\max_{\pi \in \Delta(\mathcal{A}|\mathcal{S},H)} V_{r,1}^{\pi,m} \text{ subject to }  V_{g,1}^{\pi,m} \geq b_m
\end{align}
for all $m=1,2,\ldots$, where the reward/utility functions and the transition kernels are potentially different across the episodes.
The associated Lagrangian of problem \eqref{eq: CMDP at episode m} is given by 
\begin{align} \label{eq: langrangian equation}
    \mathcal{L}^m(\pi, \mu)\coloneqq  V_{r,1}^{\pi,m}+ \mu \left(V_{g,1}^{\pi,m}-b_m \right)
\end{align}
where the policy $\pi$ is the primal variable and $\mu \geq 0$ is the dual variable. We can reformulate the constrained optimization problem \eqref{eq: CMDP at episode m} as the saddle-point problem
$\max_{\pi \in \Delta(\mathcal{A}|\mathcal{S},H)} \min_{\mu \geq 0} \  \mathcal{L}^m(\pi, \mu)$.
Let $\mathcal{D}^m(Y):=$ maximize $_{\pi} \mathcal{L}^m(\pi, \mu)$ be the dual function, $\mu^{\star,m}:=\operatorname{argmin}_{\mu \geq 0} \mathcal{D}^m(\mu)$ be an optimal dual variable and $\pi^{\star, m}$ be a globally optimal solution of \eqref{eq: CMDP at episode m} at episode $m$. 

Unlike the unconstrained MDP, the optimal solution of CMDP is usually a stochastic policy and the best deterministic policy can lose as much as the difference between the respective values of the best and the worst policies \cite{altman1999constrained}. 
As a consequence, RL methods that implicitly rely on the existence of a deterministic optimal policy (e.g., Q learning) may not be suitable for this type of problem. This further inspires the study of randomized policies and take on a policy gradient approach for non-stationary CMDP.



\textbf{Performance metrics.}
%
Suppose that the agent executes policy $\pi^m$ in episode $m$. 
We now define the dynamic regret 
and the constraint violation in the long run as:
\begin{align}\label{eq: d regret}
& \text{DR}(M) \coloneqq \sum_{m=1}^M \left( V_{r,1}^{\pi^{\star,m},m}- V_{r,1}^{\pi^m,m} \right), \\ 
& \text{CV}(M) \coloneqq \left[\sum_{m=1}^M \left(b_m- V_{g,1}^{\pi^m,m} \right)\right]_+.
\end{align}
There are two main reasons for considering the constraint violation in the long run. Firstly, in many applications such as supply chain and energy systems, the requirements of balancing the time-varying and unknown demands with the supply are formulated as some time-varying constraints. As long as the supply and the demand can be balanced in the long run, the policy is considered safe.
Secondly, since the utility function $g^m_h$ is unknown \textit{a priori} and time-varying, the constraint $V_{g,1}^{\pi,m} \geq b_m$ may not be satisfied in every episode $m$. Rather, the agent strives to satisfy the constraints in the long run. In other words, the agent aims to ensure the long-term constraint $\sum_{m=1}^M ( V_{g,1}^{\pi,m}- b_m)\geq 0$ over some given period of episodes $M$.

\textbf{Linear function approximation}
We focus on a class of CMDPs, where transition kernels and reward/utility functions are linear in feature maps.
\begin{assumption}[Linear Kernel CMDP] \label{ass: linear fun approx}
For every $m \in [M]$, the CMDP$(\mathcal{S}, \mathcal{A}, H, \mathbb{P}^m, r^m, g^m)$ satisfies the following conditions:
(1) there exist  a  kernel feature map $\psi: \mathcal{S} \times \mathcal{A} \times \mathcal{S} \rightarrow \mathbb{R}^{d_1}$ and a vector $\theta_h^m \in \mathbb{R}^{d_1}$ with $\norm{\theta_h^m}_2 \leq \sqrt{d_1}$ such that
\begin{align*}
    \mathbb{P}_h^m(x^\prime \mid x,a) = \inner{\psi(x,a,x^\prime)}{\theta_h^m}
\end{align*}
for all $(x,a,x^\prime) \in \mathcal{S} \times \mathcal{A} \times \mathcal{S}$ and $h\in[H]$;
(2) there exist a feature map $\varphi: \mathcal{S} \times \mathcal{A} \rightarrow \mathbb{R}^{d_2}$ and vectors
$\theta_{r,h}^m,\theta_{g,h}^m \in \mathbb{R}^{d_2}$  such that
\begin{align*}
&r_{h}^m(x,a) = \inner{\varphi(x,a)}{\theta_{r,h}^m}\text{ and } g_{h}^m(x,a) = \inner{\varphi(x,a)}{\theta_{g,h}^m}
\end{align*}
for all $(x,a) \in \mathcal{S} \times \mathcal{A}$ and $h\in[H]$,
where $\max\left(\norm{\theta_{r,h}^m}_2,\norm{\theta_{g,h}^m}_2 \right)\leq \sqrt{d_2}$; (3) for every function $V: \mathcal{S}\rightarrow[0,H]$, $\norm{\int_\mathcal{S} \psi(x,a,x^\prime) V(x^\prime) dx^\prime} \leq \sqrt{d_1}H$ for all $(x,a) \in \mathcal{S} \times \mathcal{A}$ and $\max(d_1,d_2)\leq d$.
\end{assumption}
This assumption adapts the definition of linear kernel MDP \cite{ayoub2020model,cai2020provably,zhou2021provably} to CMDP and has also been used in \cite{ding2021provably} for stationary  constrained MDP problems. We refer the reader to Appendix \nameref{appe: linear CMDP} for more discussions on this assumption.

\textbf{Variation budget.}  
Note that the transition function $\mathbb{P}_h^m$ and the reward/utility functions $r_h^m, g_h^m$ are determined by the unknown measures $\{\theta_h^m\}_{h\in[H], m\in[M]}$ and the latent vectors $\{\theta_{\diamond, h}^m\}_{h\in[H], m\in[M]}$ for $\diamond=r \text{ or } g$ which can vary across the indexes $(m,h)\in [M] \times [H]$ in general.  We measure the non-stationarity of the CMDP in terms of its variation in $\theta_h^m, \theta_{r,h}^m$ and $\theta_{g,h}^m$:
\begin{subequations}  \label{eq: variation budget diamond+P}
\begin{eqnarray}
&  B_{\mathbb{P}} \coloneqq \sum_{m=2}^M\sum_{h=1}^H \norm{\theta_h^m-\theta_h^{m-1}}_2, \\
& B_{\diamond} \coloneqq \sum_{m=2}^M\sum_{h=1}^H \norm{\theta_{\diamond,h}^m-\theta_{\diamond,h}^{m-1}}_2, \text{ for } \diamond=r \text{ or } g,  
\end{eqnarray}
\end{subequations}
and denote $B_\Delta=B_{\mathbb{P}}+B_{r}+B_g$.
Note that our definition of variation only imposes restrictions on the summation of non-stationarity across two different episodes, and it does not put any restriction on the difference between two consecutive steps in the same episode.
In addition to the variations defined above, we  introduce the total variation in the optimal policies of adjacent episodes:
\begin{align}
    &B_{\star}\coloneqq \sum_{m=2}^M\sum_{h=1}^H \max_{x\in\mathcal{S}} \norm{\pi_h^{\star,m}(\cdot \mid x)-\pi_h^{\star,m-1}(\cdot \mid x)}_1  \label{eq: variation budget star}.
\end{align}
The notion of $B_{\star}$ is also used for online convex optimization with a dynamic regret criterion \cite{besbes2015non, hall2013dynamical, hall2015online,cao2018online} and for policy-based methods in non-stationary unconstrained MDPs \cite{fei2020dynamic,zhong2021optimistic}. It is worth noting that the variations $(B_{\mathbb{P}}, B_{\diamond})$ and $B_{\star}$ do not imply each other.

A special but important example of the non-stationarity is the system with piece-wise constant dynamics and rewards/utilities where the number of switches is $S$. In this case, all variation budgets $(B_{\mathbb{P}}, B_{\diamond})$ and $B_{\star}$ can be upper bounded by $\mathcal{O}(SH)$. 
As one of the first works to investigate the non-stationary CMDP, we assume that we have access to quantities $B_\Delta$ and $B_{\star}$ or some upper bounds on them via an oracle. 

\section{Assumptions on the time-varying constraints}\label{sec: assumption}
In this paper, we consider two scenarios for the non-stationary CMDPs, each requiring some specific knowledge to enable safe exploration under the non-stationarity.

The first scenario assumes the knowledge of local variation budgets of constraints. We first define local variation budgets of constraints. To adapt the non-stationarity, the restart estimation of the value function is used (see Section \nameref{sec: Optimistic policy evaluation with restart strategy}), which
breaks the $M$ episodes into $\ceil{\frac{M}{L}}$ epochs.
For every $\mathcal{E}\in \left[ \ceil{\frac{M}{L}}\right]$, define $B_{g,\mathcal{E}}$ and $B_{\mathbb{P},\mathcal{E}}$ to be the local variation budgets of the utility function and transitions within epoch $\mathcal{E}$. By definition, we have $\sum_{\mathcal{E}=1}^{\ceil{\frac{M}{L}}} B_{g,\mathcal{E}} \leq B_g$ and $\sum_{\mathcal{E}=1}^{\ceil{\frac{M}{L}}} B_{\mathbb{P},\mathcal{E}} \leq B_\mathbb{P}$.

\begin{assumption}[Local variation budgets of constraints] \label{ass: local budget}
We have access to the local variation budget $B_{g,\mathcal{E}}$ and $B_{\mathbb{P},\mathcal{E}}$ for every $\mathcal{E}\in \left[ \ceil{\frac{M}{L}}\right]$,  and also the constrained optimization problems given in \eqref{eq: CMDP at episode m} are uniformly feasible.
\end{assumption}

The second scenario extends the strict feasibility (also known as Slater condition) for problem \eqref{eq: CMDP at episode m} to non-stationary constrained optimization problems.
\begin{assumption}[Uniformly strict feasibility]\label{ass: Feasibility}
We have access to a sequence of constraint thresholds $\{b_m\}_{m=1}^M$ and a constant $\gamma$ such that the constrained optimization problems in \eqref{eq: CMDP at episode m} are $\gamma-$uniformly strictly feasible, i.e., there exist $\gamma>0$ and $\bar{\pi}^m \in \Delta(\mathcal{A} \mid \mathcal{S}, H)$ such that $V_{g, 1}^{\bar{\pi}^m,m}\left(x_{1}\right) \geq b_m+\gamma$ for all $m=1,\ldots,M$.
\end{assumption}
Under this assumption, one can establish the strong duality and the boundedness of the optimal dual variable. 
\begin{lemma}[Lemma 1 in \cite{ding2021provably}] \label{lemma: lemma 1 in dongsheng}
Under Assumption \ref{ass: Feasibility}, it holds that $V_{r, 1}^{\pi^{\star,m},m}\left(x_{1}\right)=\mathcal{D}^m\left(\mu^{\star,m}\right)$ and $0 \leq \mu^{\star,m} \leq H/\gamma $ for all $m=1,\ldots,M$.
\end{lemma}

\begin{remark}
We require either Assumption \ref{ass: local budget} or Assumption \eqref{ass: Feasibility}, and both of them need not hold simultaneously. 
Assumption \ref{ass: local budget} requires the local variation budgets of constraints, but does not enforce  every instance problem \eqref{eq: CMDP at episode m} to be strictly feasible. It is suitable for the case with a forecasting oracle for the constraints. For example, in supply chain or energy systems, the supply is desired to match the time-varying and unknown demands where a forecasting oracle for the demands is usually available. In addition, it is also suitable for the case with only non-stationary rewards such as collision avoidance in a maze with a moving target.
On the other hand, Assumption \ref{ass: Feasibility} needs the knowledge of strict feasible constraint thresholds, but does not require the local variation budgets of constraints. It is suitable for the case with a relatively large feasibility threshold $\gamma$.
\end{remark}


%% file: files/Method.tex
\section{Safe exploration under the non-stationarity}\label{sec:method}
In Algorithm \ref{alg:algoirthm 1}, we develop a new efficient method named \underline{P}eriodically \underline{R}estarted \underline{O}ptimistic \underline{P}rimal-\underline{D}ual \underline{P}roximal \underline{P}olicy \underline{O}ptimization (PROPD-PPO) algorithm.
In each episode, our algorithm consists of three main stages: periodically restarted policy improvement, dual update, and periodically restarted policy evaluation. We first present the high-level idea behind our method.
\vspace{-0.2cm}
\subsection{High-level idea}
\vspace{-0.1cm}
Safe exploration in non-stationary CMDPs is more challenging in that we need to reduce the constraint violation even when the constraints vary over the episodes. To overcome this issue, we develop our method based on some assumed knowledge on the constraints. Under Assumption \ref{ass: local budget}, since the optimal dual variables may not be well-bounded, we need to add a dual regularization to stabilize the dual updates and fully utilize the convexity of the dual function. In addition, the knowledge of local variation of the constraints is needed to obtain an optimistic estimator of constraint functions, so that a large dual variable cannot amplify the estimation error of the constraint functions.
This is different from the dual update that has been used in Lagrangian-based stationary CMDPs under the strict feasible condition \cite{ding2020natural, ding2021provably, ying2021dual, efroni2020exploration, liu2021learning,qiu2020upper}. On the other hand, under Assumption \ref{ass: Feasibility}, the optimal dual variables can be bounded by Lemma \ref{lemma: lemma 1 in dongsheng}.
Then, the dual regularization and an optimistic estimator for the constraint functions are not necessary. Thus, a standard dual update will be enough.
\vspace{-0.2cm}
\subsection{Periodically restarted policy improvement}
\vspace{-0.1cm}
One way to update the policy $\pi^m$ is to solve the Lagrangian-based policy optimization problem
$\max_{\pi \in \Delta(\mathcal{A}|\mathcal{S},H)} \mathcal{L}_\xi^m(\pi,\mu^{m-1})$,
where $\mathcal{L}_\xi^m(\pi,\mu^{m-1})$ is defined in \eqref{eq: langrangian equation regu} and the dual variable $\mu^{m-1}$ is from episode $m-1$. 
Motivated by the policy improvement step in NPG \cite{kakade2001natural}, TRPO \cite{schulman2015trust}, and PPO \cite{schulman2017proximal}, 
 we perform a simple policy update in the online mirror descent fashion by
\begin{align}\label{eq: PPO step}
\nonumber \argmax_{\pi\in \Delta(\mathcal{A}\mid \mathcal{S}, H)}& \sum_{h=1}^H\inner{ \left(Q_{r, h}^{m-1}+ \mu ^{m-1} Q_{g, h}^{m-1}\right) (x_h,\cdot)}{\pi_h-\pi_h^{m-1}}\\
&-\frac{1}{\alpha} \sum_{h=1}^H D\left(\pi_h(\cdot|x_h) \mid {\pi}_h^{m-1}(\cdot|x_h)\right).
\end{align}

\begin{algorithm}[tb]
   \caption{Periodically Restarted Optimistic Primal-Dual Proximal Policy Optimization}
   \label{alg:algoirthm 1}
\begin{algorithmic}[1]
   \STATE {\bfseries Inputs:} Time horizon $M$, restart period $W, L$, $\{Q_{r,h}^0, Q_{g,h}^0\}_{h=1}^H$ and $V_{g,1}^0$ being zero functions, initial policy $\{\pi_h^0\}_{h\in[H]}$ being uniform distributions on $\mathcal{A}$, initial dual variable $\mu^0=0$, dual regularization parameter $\xi$, learning rates $\alpha, \eta>0$, $\chi $.
   \FOR{$m=1, \ldots, M$}
    \STATE Set the initial state $x_1^m=x_1$,  $\ell_\pi^m=(\ceil{\frac{m}{L}}-1)L+1$, $\ell_Q^m=(\ceil{\frac{m}{W}}-1)W+1$.
    \IF{$m=\ell_\pi^m$}
    \STATE Set $\{Q_{r,h}^{m-1}, Q_{g,h}^{m-1}\}_{h=1}^H$ as zero functions and set $\{\pi_{h}^{m-1}\}_{h=1}^H$ as uniform distributions on $\mathcal{A}$.
    \ENDIF
    \FOR{ $h=1,2,\ldots,H$}
    \STATE Update the policy $\pi_{h }^m (\cdot \mid \cdot)\propto $ \\${\pi}_{h}^{m-1}(\cdot \mid \cdot) \exp \left(\alpha \left(Q_{r,h}^{m-1}+\mu^{m-1}Q_{g,h}^{m-1} \right) (\cdot, \cdot)\right)$.
    \STATE Take an action $a_h^m \sim\pi_h^m (\cdot\mid x_h^m)$ and receive reward/utility $r_h(x_h^m,a_h^m),g_h(x_h^m,a_h^m) $.
    \STATE Observe the next state $x_{h+1}^m$.
    \ENDFOR
    \STATE Update the dual variable by $\mu^m=\text{Proj}_{[0,\chi ]} \left( \mu^{m-1}+\eta \left(b_m -V_{g,1}^{m-1}(x_1) -\xi \mu^{m-1} \right)  \right)$.
    \STATE Estimate $\{Q_{r,h}^m, Q_{g,h}^m\}_{h=1}^H$ and  $V_{g,1}^{m}$ via LSTD$\left(\{x_h^\tau, a_h^\tau, r_h^\tau(x_h^\tau,a_h^\tau), g_h^\tau(x_h^\tau, a_h^\tau)\}_{h=1,\tau=\ell_Q^m}^{H,m}\right)$.
   \ENDFOR
\end{algorithmic}
\end{algorithm}

Since the above update is separable over $H$ steps, we can update the policy $\pi^m$ as line 8 in Algorithm \ref{alg:algoirthm 1}, leading to a closed-form solution for each step $h\in[H]$.  Furthermore, in order to guarantee the policy to be exploratory enough in new environments, our policy improvement step also features a periodic restart mechanism, which resets its policy to a uniform distribution over the action space $\mathcal{A}$ every $L$ episodes. 
\begin{remark}
Although policy improvement step \eqref{eq: PPO step} has been used in stationary CMDPs \cite{ding2021provably}, our method differs in the sense that we remove the requirement to mix the policy with a uniform policy at every iteration. This is due to a technical improvement in the analysis by replacing the ``pushback property of KL-divergence lemma'' (Lemma 14 in \cite{ding2021provably}) with the ``one-step descent lemma'' (Lemma \ref{lemma: One-step descent lemma}) for the KL-regularized optimization.
\end{remark}
\vspace{-0.2cm}
\subsection{Dual update}
\vspace{-0.1cm}
We first define the modified Lagrangian of \eqref{eq: langrangian equation} to be
\begin{align}\label{eq: langrangian equation regu}
\mathcal{L}_\xi^m(\pi, \mu)\coloneqq  V_{r,1}^{\pi,m}+ \mu \left(V_{g,1}^{\pi,m}-b_m \right) +\frac{\xi}{2}\norm{\mu}_2^2
\end{align}
where $\xi\geq 0$ is the dual regularization parameter to be determined later. 
Since the value function $V_{g,1}^{\pi,m}$ is unknown, in order to infer the constraint violation for the dual update, we estimate $V_{g,1}^{\pi^m,m} (x_1)$ via an optimistic policy evaluation. We update the Lagrange multiplier $\mu$ by moving $\mu^m$ to the direction of minimizing the estimated Lagrangian $\mathcal{L}(\pi, \mu)$:
\begin{align}\label{eq: estimtaed langrangian equation regu}
\widetilde{\mathcal{L}}_\xi^m(\pi, \mu)\coloneqq  V_{r,1}^{m}+ \mu \left(V_{g,1}^{m}-b_m \right) +\frac{\xi}{2}\norm{\mu}_2^2.
\end{align}
over $\mu\geq 0$ in line 14 of Algorithm \ref{alg:algoirthm 1}, where $\eta>0$ is a stepsize and $\text{Prof}_{[0,\chi ]}$ is a projection onto $[0,\chi ]$ with an upper bound $\chi $ on $\mu^m$. 
The choices of the parameters $\chi $ and $\xi$ depend on the assumption:
\begin{align*}
&\xi>0, \ \chi  =  \infty, \  \text{under Assumption \ref{ass: local budget}}, \\ 
&\xi=0, \ \chi  =\frac{2H}{\gamma}, \text{under Assumption \ref{ass: Feasibility}}.
\end{align*}
Under Assumption \ref{ass: local budget}, since the strictly feasibility may not hold for all episodes (corresponding to $\gamma=0$), we may not have a finite upper bound on the dual variable $\mu$. Thus, a dual regularization with $\xi>0$ is needed to stabilize the dual updates under the non-stationarity. The value of $\xi$ depends on the number of episodes $M$ and the variation budgets ${B}_\mathbb{P}, {B}_g$. On the other hand, under Assumption \ref{ass: Feasibility}, we choose $\chi = \frac{2H}{\gamma} \geq 2 \mu^{*,m}$ similarly as \cite{ding2021provably, efroni2020exploration}, so that the projection interval $[0,\chi ]$ includes all optimal dual variables $\{ \mu^{*,m}\}_{m=1}^M$ in light of Lemma \ref{lemma: lemma 1 in dongsheng}.

\subsection{Periodically restarted optimistic policy evaluation}\label{sec: Optimistic policy evaluation with restart strategy}
To evaluate the policy under the unknown nonstationarity, we take the Least-Squares Temporal Difference (LSTD) \cite{bradtke1996linear,lazaric2010finite} with UCB to properly handle the exploration-exploitation trade-off and apply the restart strategy to adapt to the unknown nonstationarity. In particular, we apply the restart strategy and evaluate the policy $\pi^m$ only based on the previous historical trajectories from the episode $\ell_Q^m$ to the episode $m$ instead of the all previous historical trajectories.
The method is standard and summarized in Algorithm \ref{alg:algoirthm 2} in Section \nameref{sec:appe_PE} of Appendix. 

After obtaining the estimates of $\mathbb{P}_h^m V_{\diamond, h+1}^m$ and $\diamond_h^m(\cdot,\cdot)$ for $\diamond =r \text{ or } g$, we update the estimated action-value function $\left\{ Q_{\diamond,h}^m\right\}_{h=1}^H$ iteratively 
and add UCB bonus terms $\Gamma_h^m(\cdot,\cdot)$, $\Gamma_{\diamond,h}^m(\cdot,\cdot): \mathcal{S}\times\mathcal{A} \rightarrow \mathbb{R}^+$ so that
\begin{align*}
\Omega_{1,\diamond} \coloneqq \left(\varphi^m\right)^\top u_{\diamond,h}^m+ \Gamma_h^m  \ \text{ and } \ \Omega_{2,\diamond} \coloneqq\left(\phi_{\diamond,h}^m\right)^\top w_{\diamond,h}^m + \Gamma_{\diamond,h}^m
\end{align*}
all become upper bounds on $\mathbb{P}_h^m V_{\diamond, h+1}^m$ and $\diamond_h^m(\cdot,\cdot)$ (up to some errors due to the non-stationarity). Here, the  weights $ u_{\diamond,h}^m, w_{\diamond,h}^m$ and the bonus terms $\Gamma_h^m, \Gamma_{\diamond,h}^m$ are defined in Algorithm \nameref{alg:algoirthm 2} in Section \nameref{sec:appe_PE} of Appendix. Moreover,
\begin{align*}
&Q_{r, h}^{m}(\cdot, \cdot)=\min \left(H-h+1, \Omega_{1,r}(\cdot, \cdot)+\Omega_{2,r}(\cdot, \cdot) \right)_+,\\
&Q_{g, h}^{m}(\cdot, \cdot)=\min \left(H-h+1, \Omega_{1,g}(\cdot, \cdot)+\Omega_{2,g}(\cdot, \cdot)+LV\right)_+
\end{align*}
where $LV>0$ depends on the local variation budgets of the constraint $B_{\mathbb{P},\mathcal{E}}$,  $B_{g,\mathcal{E}}$ under Assumption \ref{ass: local budget},  $LV=0$ under Assumption \ref{ass: Feasibility}, and $(x)_+$ denotes the maximum between $x$ and $0$. The reason for introducing a  positive $LV$ term under Assumption \ref{ass: local budget} is to guarantee that the model prediction error in $Q_{g, h}^{m}$ is non-positive when the dual variable $\mu$ is very large.




%% file: files/Main_results.tex
\section{Main results}\label{sec:main results}
We now present the dynamic regret and the constraint violation bounds for Algorithm \ref{alg:algoirthm 1} under the two alternative assumptions introduced in Section \nameref{sec: assumption}. The choices of the algorithm parameters will depend on the assumption used for the analysis. When both assumptions are satisfied, one can check which one yields a tighter bound, and this depends on the value of the strict feasibility threshold $\gamma$ (and the values of $H, M$ if in the tabular CMDP setting).
\vspace{-0.2cm}
\subsection{Linear kernal CMDP}
\vspace{-0.1cm}
We first present the results for linear Kernal CMDP under each of Assumptions \ref{ass: local budget} and \ref{ass: Feasibility}.
\begin{theorem}[Linear Kernal CMDP + Assumption \ref{ass: local budget}] \label{thm: Linear Kernal MDP under local budget}
Let Assumptions \ref{ass: linear fun approx} and \ref{ass: local budget} hold. Given $p\in(0,1)$, we set $\alpha=  H^{-1} M^{-\frac{1}{2}} (\sqrt{d}B_\Delta+B_\star)^{\frac{1}{3}} $, $L= {M^{\frac{3}{4}}} (\sqrt{d}B_\Delta+B_\star)^{-\frac{2}{3}} $, $\eta=M^{-\frac{1}{2}}$, $\xi=2H (\sqrt{d}B_\Delta+B_\star)^{\frac{1}{3}} M^{-\frac{1}{2}}$,
$W=d^{-\frac{1}{4}} H^{-1} {M}^{\frac{1}{2}} B_\Delta^{-\frac{1}{2}}$,
in Algorithm \ref{alg:algoirthm 1} and set set $\beta=$ $C_{1} \sqrt{d H^{2} \log (d W / p)}$,  $LV=B_{\mathbb{P},\mathcal{E}} H^2 d_1\sqrt{d_1 W} +B_{g,\mathcal{E}}\sqrt{d_2 W}$ in Algorithm \ref{alg:algoirthm 2}. Then, with probability $1-p$, the dynamic regret and the constraint violation satisfy 
\begin{align*}
&\operatorname{DR}(M)\leq \widetilde{\mathcal{O}}\left( d^{\frac{9}{8}} H^{\frac{5}{2}} M^{\frac{3}{4}} (\sqrt{d}B_\Delta +B_\ast)^{\frac{1}{3}}\right), \\
&\operatorname{CV}(M)\leq  \widetilde{\mathcal{O}}\left( d^{\frac{9}{8}} H^{\frac{5}{2}} M^{\frac{3}{4}} (\sqrt{d}B_\Delta +B_\ast)^{\frac{1}{3}}\right).
\end{align*}
\end{theorem}

\begin{theorem}[Linear Kernal CMDP + Assumption \ref{ass: Feasibility}] \label{thm: Linear Kernal MDP under feasible}
Let Assumptions \ref{ass: linear fun approx} and \ref{ass: Feasibility} hold. Given $p\in(0,1)$, we set $\alpha=\gamma H^{-\frac{3}{2}} M^{-\frac{1}{3}} (\sqrt{d}B_\Delta+B_\star)^{\frac{1}{3}} $, $L= {M^{\frac{2}{3}}} (\sqrt{d}B_\Delta+B_\star)^{-\frac{2}{3}} $, $\eta=M^{-\frac{1}{2}}$, $\xi=0$,
$W=d^{-\frac{1}{4}} H^{-1} {M}^{\frac{1}{2}} B_\Delta^{-\frac{1}{2}}$
in Algorithm \ref{alg:algoirthm 1} and set set $\beta=$ $C_{1} \sqrt{d H^{2} \log (d W / p)}$,  $LV=0$ in Algorithm \ref{alg:algoirthm 2}. Then, with probability $1-p$, the dynamic regret and the constraint violation satisfy 
\begin{align*}
&\operatorname{DR}(M)\leq \widetilde{\mathcal{O}}\left( \gamma^{-1}d^{\frac{9}{8}} H^{\frac{5}{2}} M^{\frac{3}{4}} (\sqrt{d}B_\Delta +B_\ast)^{\frac{1}{3}}\right), \\ & \operatorname{CV}(M)\leq  \widetilde{\mathcal{O}}\left( \gamma^{-1}d^{\frac{9}{8}} H^{\frac{5}{2}} M^{\frac{3}{4}} (\sqrt{d}B_\Delta +B_\ast)^{\frac{1}{3}}\right).
\end{align*}
\end{theorem}
The proofs for Theorems \ref{thm: Linear Kernal MDP under local budget} and \ref{thm: Linear Kernal MDP under feasible} can be found in Appendix \nameref{sec:appe-linear MDP under local budget} and \nameref{sec:appe-linear MDP under feasibility}, respectively.
Our dynamic regret bounds in Theorems \ref{thm: Linear Kernal MDP under local budget} and \ref{thm: Linear Kernal MDP under feasible} have the optimal dependence on the total number of episodes $M$. This matches the existing bounds in the general non-stationary linear kernel MDP setting without any constraints \cite{zhong2021optimistic, zhou2020nonstationary, touati2020efficient}. The dependence on the variation budgets $(B_\Delta, B_\star)$ also matches the existing bound in policy-based method for the non-stationary linear kernel MDP setting \cite{zhong2021optimistic}.
Regarding the long-term safe exploration, we provide the first finite-time constraint violation result in the non-stationary CMDP setting. 

In the linear kernel CMDP setting, the same dynamic regret and constraint violation bounds are obtained under either of Assumptions \ref{ass: local budget} and \ref{ass: Feasibility}, except that the  dynamic regret and constraint violation under Assumption \ref{ass: Feasibility} also depend on the strict feasibility threshold $\gamma$. When $\gamma$ is small, i.e., there exist some episodes for which the CMDP problem \eqref{eq: CMDP at episode m} does not have a large enough strict feasibility threshold, the  dynamic regret and constraint violation bounds in Theorem \ref{thm: Linear Kernal MDP under feasible} may be large.

\subsection{Tabular CMDP}
A special case of the linear kernel CMDP in Assumption \ref{ass: linear fun approx} is the tabular CMDP with $|\mathcal{S}|<\infty$ and $|\mathcal{A}|<\infty$. In the tabular case,  improved results can be obtained by incorporating Algorithm \ref{alg:algoirthm 1} with a variant of the optimistic policy evaluation method in Algorithm \ref{alg:algoirthm 2}. We refer the reads to Algorithm \ref{alg:algoirthm 3} in Section \nameref{sec:appe_PE} for such procedures and state the result below:

\begin{theorem}[Tabular CMDP + Assumption \ref{ass: local budget}] \label{thm: Tabular Case MDP with local budget}
Let Assumption \ref{ass: local budget} hold and consider a tabular CMDP. Given $p\in(0,1)$ and $\rho\in[\frac{1}{3},\frac{1}{2}]$, we set $\alpha=  H^{-\frac{1}{3}} M^{-\rho} (B_\Delta+B_\star)^{\frac{1}{3}} $, $L= H^{-\frac{1}{3}} {M^{\frac{1+\rho}{2}}} (B_\Delta+B_\star)^{-\frac{2}{3}} $, $\eta=H^{-\frac{1}{3}} M^{-\frac{1}{2}}$, $\xi=2H^{\frac{5}{3}} (B_\Delta+B_\star)^{\frac{1}{3}} M^{-\rho}$,
$W= H^{\frac{2}{3}}|\mathcal{S}|^{\frac{2}{3}} |\mathcal{A}|^{\frac{1}{3}} \left(\frac{M}{B_\Delta}\right)^{\frac{2}{3}}$ in Algorithm \ref{alg:algoirthm 1} and $\beta=C_{4} H \sqrt{|\mathcal{S}| \log (|\mathcal{S}||\mathcal{A}| W / p)}$, $LV=B_{\mathbb{P},\mathcal{E}} H  +B_{g,\mathcal{E}}$ in Algorithm \ref{alg:algoirthm 3}. Then, with probability $1-p$, the dynamic regret and the constraint violation 
satisfy 
\begin{align*}
&\operatorname{DR}(M)\leq  \widetilde{\mathcal{O}}\left( |\mathcal{S}|^{\frac{2}{3}} |\mathcal{A}|^{\frac{1}{3}} H^{\frac{5}{3}} M^{\frac{1+\rho}{2}} (B_\Delta +B_\ast)^{\frac{1}{3}}\right), \\
& \operatorname{CV}(M)\leq  \widetilde{\mathcal{O}}\left( |\mathcal{S}|^{\frac{2}{3}} |\mathcal{A}|^{\frac{1}{3}} H^{\frac{5}{3}} M^{\frac{2-\rho}{2}}  (B_\Delta +B_\ast)^{\frac{1}{3}}\right).
\end{align*}
\end{theorem}
\begin{theorem}[Tabular CMDP + Assumption \ref{ass: Feasibility}] \label{thm: Tabular Case MDP under feasible}
Let Assumption \ref{ass: Feasibility} hold and consider a tabular CMDP. Given $p\in(0,1)$, we set $\alpha= \gamma H^{-\frac{3}{2}} M^{-\frac{1}{3}} (B_\Delta+B_\star)^{\frac{1}{3}} $, $L= {M^{\frac{2}{3}}} (B_\Delta+B_\star)^{-\frac{2}{3}} $, $\eta=M^{-\frac{1}{2}}$, $\xi=0$,
$W= |\mathcal{S}|^{\frac{2}{3}} |\mathcal{A}|^{\frac{1}{3}} \left(\frac{M}{B_\Delta}\right)^{\frac{2}{3}}$ in Algorithm \ref{alg:algoirthm 1} and $\beta=C_{4} H \sqrt{|\mathcal{S}| \log (|\mathcal{S}||\mathcal{A}| W / p)}$, $LV=0$ in Algorithm \ref{alg:algoirthm 3}. Then, with probability $1-p$, the dynamic regret and the constraint violation 
satisfy 
\begin{align*}
&\operatorname{DR}(M)\hspace{-0.05cm}\leq \hspace{-0.05cm} \widetilde{\mathcal{O}}\left(\gamma^{-1} |\mathcal{S}|^{\frac{2}{3}} |\mathcal{A}|^{\frac{1}{3}}   H^{\frac{5}{2}} M^{\frac{2}{3}} (B_\Delta+B_\star)^{\frac{1}{3}} \right), \\
& \operatorname{CV}(M) \hspace{-0.05cm} \leq \hspace{-0.05cm} \widetilde{\mathcal{O}}\left(\gamma^{-1} |\mathcal{S}|^{\frac{2}{3}} |\mathcal{A}|^{\frac{1}{3}}   H^{\frac{5}{2}} M^{\frac{2}{3}} (B_\Delta+B_\star)^{\frac{1}{3}} \right).
\end{align*}
\end{theorem}

The proofs for Theorems \ref{thm: Tabular Case MDP with local budget} and \ref{thm: Tabular Case MDP under feasible} can be found in Appendix \nameref{sec:appe-tabular MDP under local budget} and \nameref{sec:appe-tabular MDP under Feasibility}, respectively. For the tabular CMDP under Assumption \ref{ass: local budget}, there is a trade-off for the dependence on the total number of episodes $M$ between the dynamic regret and the constraint violation. This trade-off is controlled by the primal update parameter $\alpha$ and the dual regularization parameter $\xi$. Such trade-off does not appear in the linear kernel CMDP setting because the dynamic regret and constraint violation in the linear kernel CMDP are bottlenecked by the error in the non-stationary policy evaluation.

The dynamic regret and constraint violation bounds in Theorem \ref{thm: Tabular Case MDP under feasible} have an improved dependence on the total number of episodes $M$ compared to Theorems \ref{thm: Linear Kernal MDP under local budget} and \ref{thm: Linear Kernal MDP under feasible}. This improvement is due to the improved result of the policy evaluation step in the tabular setting. The dependence on $M$ in Theorem \ref{thm: Tabular Case MDP under feasible} is also better than that of Theorem \ref{thm: Tabular Case MDP with local budget}. This is due to a sharper analysis for the constraint violation under Assumption \ref{ass: Feasibility} based on \cite[Proposition 3.60]{beck2017first}. However, the dynamic regret and constraint violation bounds in Theorem \ref{thm: Tabular Case MDP under feasible} have a worse dependence on the horizon $H$ and are also dependent on the feasibility threshold $\gamma$ compared to Theorem \ref{thm: Tabular Case MDP with local budget}. In addition, the dependence of the dynamic regret on $M$ and $(B_\Delta, B_\star)$ matches the existing bound in the non-stationary tabular MDP setting without any constraints \cite{mao2020near}.

%% file: files/safe_exp_non.tex
\section{Proof sketch}

The safe exploration in non-stationary CMDP is more challenging since the utilities and dynamics are time-varying and unknown a priori. 
In this section, we outline some of the key ideas behind the proof, especially how to decouple the dynamic regret and constraint violation under the non-stationarity.  We defer the full proof to Appendix.

\subsection{Dynamic regret}

By combining the primal-dual analysis of stationary CMDPs \cite{ding2021provably} and the analysis for the non-stationary MDP \cite{zhong2021optimistic,fei2020dynamic}, we can obtain the following bound on the Lagrangian function:
\begin{align} \label{eq: analysis 1}
 \nonumber &\sum_{m=1}^M \hspace{-0.1cm} \left(V_{r, 1}^{\pi^{\star,m},m}\hspace{-0.1cm}-\hspace{-0.06cm} V_{r, 1}^{\pi^m,m}\hspace{-0.06cm}+\hspace{-0.06cm} \mu^m\left(b_m\hspace{-0.06cm}-\hspace{-0.06cm}V_{g, 1}^{\pi^m,m}\right)\hspace{-0.06cm} \right) \\
 &\leq \delta_1 \hspace{-0.06cm}+\hspace{-0.06cm} \alpha H^2 \hspace{-0.06cm} \sum_{m=1}^M |\mu^m|^2 \hspace{-0.06cm}+\hspace{-0.06cm} \sum_{m=1}^M \hspace{-0.06cm}\sum_{h=1}^H \mu^m  \EE_{{\pi}^{\star,m}, \mathbb{P}^m} \left[ \iota_{g, h}^{m}\right]
\end{align}
where $\delta_1$ contains all terms irrelated to the dual variables $\{\mu^m\}_{m=1}^M$ and  $\iota_{g, h}^{m}$ is the model prediction error of the constraint. Furthermore, with the dual regularization and the dual update, it holds that
\begin{align} \label{eq: analysis 2}
\nonumber &-\sum_{m=1}^{M} \mu^{m}\left(V_{g, 1}^{\pi^{\star,m},m}-V_{g, 1}^{m} \right) \\
 & \leq  {\eta H^{2}(M+1)} +(\eta \xi^2-\xi) \sum_{m=1}^{M}  (\mu^{m})^2.
\end{align}

Combining the  inequalities \eqref{eq: analysis 1} and \eqref{eq: analysis 2} yields 
 \begin{align*}
    \operatorname{DR}(M)\leq& \delta_1+ (\alpha H^2+\eta \xi^2-\xi) \sum_{m=1}^M (\mu^m)^2 \\
    & +\eta H^{2}(M+1) + \sum_{m=1}^M\sum_{h=1}^H \mu^m \EE_{{\pi}^{\star,m}, \mathbb{P}^m} \left[ \iota_{g, h}^{m}\right].
\end{align*}
Under Assumption \ref{ass: local budget}, since $\mu^m$ is not well-bounded (when CMDP is not strictly feasible), a positive dual regularization $\xi$ is needed to guarantee $\alpha H^2+\eta \xi^2-\xi\leq 0$ and the knowledge of the local variation budgets of the constraint $B_{\mathcal{P}, \mathcal{E}}, B_{g, \mathcal{E}}$ are needed to guarantee that $\iota_{g, h}^{m}$ is non-positive. On the other hand, under Assumption \ref{ass: Feasibility}, $\mu^m$ is bounded by $\frac{2H}{\gamma}$ and the dynamic regret can be well-controlled without any additional requirement on $\xi$ and  $\iota_{g, h}^{m}$.
\vspace{-0.2cm}
\subsection{Constraint violation}
\vspace{-0.1cm}
The techniques used for our analyses under Assumptions \ref{ass: local budget} and \ref{ass: Feasibility} are different. We first consider Assumption \ref{ass: local budget}. If we set $\xi\eta\leq \frac{1}{2}$ and $\chi=\infty$ in Algorithm \ref{alg:algoirthm 1}, then from the convexity of the Lagrangian function with respect to the dual variable, we have 
\begin{align*}
 \sum_{m=1}^M (\mu-\mu^m) \left(b_m-V_{g,1}^{m}  \right)-(\frac{\xi M}{2}+\frac{1}{2\eta})\mu^2    \leq \eta  H^2 M
\end{align*}
for every $\mu\geq 0$. By combining the above inequality with the inequality \eqref{eq: analysis 1} and using the fact that $\left|V_{r, 1}^{\pi^{\star,m},m}-V_{r, 1}^{\pi^m,m}\right|\leq H$, it holds that
\begin{align*}
    \mu \sum_{m=1}^M  \left(b_m-V_{g, 1}^{\pi^m,m} \right) -(\frac{\xi M}{2}+\frac{1}{2\eta})\mu^2 \leq \delta_2,
\end{align*}
where $\delta_2$ is irrelated to $\mu$. Then, by maximizing both sides of the above inequality over $\mu\geq 0$, we can obtain the constraint violation under Assumption \ref{ass: local budget}. 
On the other hand, the analysis of the constraint violation under Assumption \ref{ass: Feasibility} relies on the extension of \cite[Proposition 3.60]{beck2017first} or \cite[Lemma 10]{ding2021provably}. In particular, it shows that if there exist $\delta_3$ and $\bar{C}^{\star} \geq 2 \max_{m \in [M]}\mu^{\star,m}$ such that 
$$
\sum_{m=1}^M V_{r, 1}^{\pi^{\star,m},m}-V_{r, 1}^{\pi^m,m}+\bar{C}^{\star} \sum_{m=1}^M \left(b_m-V_{g, 1}^{\pi^m,m}\right) \leq \delta_3,
$$
then the constraint violation can be bounded by
$
\sum_{m=1}^M \left(b_m-V_{g, 1}^{\pi^m,m} \right) \leq \frac{2 \delta_3}{\bar{C}^{\star}}.
$

%% file: files/Conclusion.tex
\section{Conclusion}\label{sec:conclusion}

In this paper, we formulate a general non-stationary safe RL problem as a non-stationary episodic CMDP. To solve this problem, we identify two alternative conditions on the time-varying constraints under which we can guarantee the safety in the long run. We also develop a new algorithm named PROPD-PPO, which consists of three main mechanisms: periodic-restart-based policy improvement, dual update with dual regularization, and periodic-restart-based optimistic policy evaluation. 
We establish the dynamic regret bound and a constraint violation bounds for the proposed algorithm in both the linear kernel CMDP function approximation setting and the tabular CMDP setting under two alternative  assumptions. This paper provides the first provably efficient algorithm for non-stationary CMDPs with safe exploration. An interesting future direction is to relax the assumption on the prior knowledge of the variation budgets and generalize the non-stationarity detection mechanism  in \cite{wei2021non} to our CMDP setting.

%% file: files/appe_notation.tex
\section{Notation}\label{sec:notation}
Let $\left|\mathcal{S}\right|$ denote the cardinality of set $\mathcal{S}$.
When the variable $s$ follows the distribution $\rho$, we write it as $s\sim \rho$.
Let $\mathbb{E}[\cdot]$ and $\mathbb{E}[\cdot\mid \cdot]$ denote the expectation and conditional expectation of a random variable, respectively.
Let $\mathbb{R}$ represent the set of real numbers.
We use the shorthand notation $[n]$ for the set $\{1,2,\dots, n\}$.
For a vector $x$, we use $x^T$ to denote the transpose of $x$,
and use $x_i$ or $(x)_i$ to denote the $i$-th entry of $x$.
When applying a scalar function to a vector $x$, e.g. $\log x$, the operation is understood as entry-wise. 
For vectors $x$ and $y$, we use $x \geq y$ to denote an entry-wise inequality.
We use the standard notations that $\|x\|_1 = \sum_i |x_i|$, $\|x\|_2 = \sqrt{\sum_i x_i^2}$, and $\|x\|_\infty = \max_i |x_i|$.
For a matrix $A$, we use $A_{ij}$ to denote its $(i,j)$-th entry and use $\lambda_{\text{min}}(A)$ to denote its minimum eigenvalue. We use $\norm{v}_A$ to denote the norm induced by a positive definite matrix $A$ for vector $v$, i.e., $\norm{v}_A=\sqrt{v^{\top}Av}$.
Let $I_n$ denote the $n\times n$ identity matrix.
For a function $f(x)$, let $\nabla_x f(x)$ denote its gradient with respect to $x$, and we may omit $x$ in the subscript when it is clear from the context.
Let $\operatorname{arg}\min f(x)$ (resp. $\operatorname{arg}\max f(x)$) denote any arbitrary global minimum (resp. global maximum) of $f(x)$.
We denote $\text{Proj}_{[a,b]}(x)$ as the projection of $x$ onto the interval $[a,b]$ and $(x)_+$ as the maximum between $x$ and $0$.
 Lastly, given a variable $x$, the notation $a=\mathcal{O}(b(x))$ means that $a \leq C \cdot b(x)$ for some constant $C>0$ that is independent of $x$. {Similarly, $a=\widetilde{\mathcal{O}}(b(x))$ indicates that the previous inequality may also depend on the function $\log(x)$, where $C>0$ is again independent of $x$.}  

%% file: files/appe_linearCMDP.tex
\section{Linear Kernel CMDP}\label{appe: linear CMDP}
Some examples of linear kernel MDPs 
include tabular MDPs \cite{zhou2021provably}, feature embedded
transition models \cite{yang2020reinforcement}, and linear combinations of base
models \cite{modi2020sample}. The linear kernal MDP defined in Assumption \ref{ass: linear fun approx} is different from linear MDP \cite{yang2019sample,jin2020provably} since they define transition dynamics using the different feature maps, although both of them encapsulate the tabular MDP as the special case. They are not comparable since one cannot
be implied by the other \cite{zhou2021provably}.

%% file: files/appe_PolicyUpdate.tex
\section{Policy update}\label{appe: policy update}
For this purpose, we first use the performance difference lemma \cite{kakade2002approximately} to expand the value function $V_{\diamond,1}^{\pi,m}(x_1)$ at the previously known policy $\pi^{m-1}$ as follows:
\begin{align*}
   & V_{\diamond,1}^{\pi,m}= V_{\diamond,1}^{\pi^{m-1},m}+\\
   &\EE_{\pi^{m-1},\mathbb{P}^{m-1}}\left[ \sum_{h=1}^H \inner{Q_{\diamond, h}^{\pi,m}(x_h,\cdot)}{\pi_h-\pi_h^{m-1} (\cdot|x_h)}\right],
\end{align*}
where $\EE_{\pi^{m-1},\mathbb{P}^{m-1}}$ is taken over the random state-action sequence $\{(x_h,a_h)\}_{h=1}^H$ at episode $m-1$. Then, we introduce an approximation of $V_{\diamond,1}^{\pi,m}$ for any state-action sequence $\{(x_h,a_h)\}_{h=1}^H$ induced by $\pi$:
\begin{align*}
L_{\diamond}^{m-1}(\pi)\coloneqq &V_{\diamond,1}^{m-1}(x_1)+\sum_{h=1}^H \inner{Q_{\diamond, h}^{m-1}(x_h,\cdot)}{\pi_h-\pi_h^{m-1} (\cdot|x_h)}
\end{align*}
where $V_{\diamond,1}^{m-1}$ and $Q_{\diamond, h}^{m-1}$ are the approximations of $V_{\diamond,1}^{\pi,m-1}$ and $Q_{\diamond, h}^{\pi,m-1}$ and can be estimated from an optimistic policy evaluation procedure that will be discussed in Section \nameref{sec: Optimistic policy evaluation with restart strategy}. With this notation, in each episode, we perform a simple policy update in the online mirror descent fashion with the KL-divergence regularization,
\begin{align*}
 \max_{\pi\in\Delta(\mathcal{A}|\mathcal{S}, H)} & L_{r}^{m-1}(\pi)-\mu^{m-1} (b-L_{g}^{m-1}(\pi))\\
 &-\frac{1}{\alpha} \sum_{h=1}^H D\left(\pi_h(\cdot|x_h) \mid \pi_h^{m-1}(\cdot|x_h)\right)
\end{align*}
where the constant $\alpha>0$ is a trade-off parameter and $D(\pi\mid\pi^{m-1})$ is the KL divergence between $\pi$ and $\pi^{m-1}$. 

%% file: files/appe_policy_eva.tex
\section{Policy evaluation algorithm} \label{sec:appe_PE}

\subsection{Policy evaluation algorithm for linear kernel MDP setting}

For episode $m$ and  each step $h\in[H]$,
we estimate $\mathbb{P}_h^m V_{r,h+1}^m$ in the Bellman equation \eqref{eq: belmman equation} by
${\phi_{r,h}^m} ^\top w_{r,h}^m$,  where  $w_{r,h}^m$ is updated by the minimizer of the regularized least-squares problem over $w$,
\begin{align} \label{eq: regression w}
\sum_{\tau=\ell_Q^m}^{m-1} \left( V_{r,h+1}^\tau (x_{h+1}^\tau) - \phi_{r,h}^\tau (x_h^\tau,a_h^\tau) ^\top w\right)^2+\lambda\norm{w}^2
\end{align}
where 
\begin{subequations}
\begin{eqnarray}
&\phi_{r,h}^\tau(\cdot,\cdot)\coloneqq \int_{\mathcal{S}} \psi(\cdot,\cdot,x^\prime) V_{r,h+1}^\tau (x^\prime) dx^\prime \\
& V_{r,h+1}^\tau (\cdot)=\inner{Q_{r,h+1}^\tau (\cdot,\cdot)}{\pi_{h+1^\tau(\cdot\mid\cdot)}}_{\mathcal{A}}
\end{eqnarray}
\end{subequations}
for all $h\in[H-1]$ and $V_{r,H+1}^\tau=0$, and $\lambda>0$ is the regularization parameter. 

Similarly, we estimate $\mathbb{P}^m_h V_{g,h+1}^m$ by $(\phi_{g,h}^m)^\top w_{g,h}^m$. We display the least-squares solution in lines 3-5 of Algorithm \ref{alg:algoirthm 2} where the symbol $\diamond$ denotes $r \text{ or } g$. 
In addition, since we consider the bandit reward/utility feedback in the linear function approximation setting, we also need to estimate $r_h^m(\cdot,\cdot)$ by $\left(\varphi^m(\cdot,\cdot)\right)^\top u_{r,h}^m$, where $u_{r,h}^m$ is updated by the minimizer of another regularized least-squares problem,
\begin{align}\label{eq: regression u}
    \sum_{\tau=\ell_Q^m}^{m-1}\left(r_h^\tau(x_h^\tau, a_h^\tau) - \left(\varphi^\tau(x_h^\tau, a_h^\tau)\right)^\top u\right)^2+\lambda\norm{u}_2^2
\end{align}
where $\lambda$ is the regularization parameter. Similarly, we estimate $g^m_h (\cdot,\cdot)$ by $ \left(\varphi^m(\cdot,\cdot)\right)^\top u_{g,h}^m$. The least-squares solutions
lead to lines 8-9 of Algorithm \ref{alg:algoirthm 2}.

\begin{algorithm}[H]
   \caption{Least-Squares Temporal Difference with UCB exploration (LSTD)}
   \label{alg:algoirthm 2}
\begin{algorithmic}[1]
 \STATE {\bfseries Inputs:} $\{x_h^\tau, a_h^\tau, r_h^\tau(x_h^\tau,a_h^\tau), g_h^\tau(x_h^\tau, a_h^\tau)\}_{h=1,\tau=\ell_Q^m}^{H,m}$, regularization parameter $\lambda$, UCB parameter $\beta$, local variation budgets $B_{\mathbb{P},\mathcal{E}}, B_{g,\mathcal{E}}$.
    \FOR{ $h=H, H-1,\ldots,1$}
    \STATE $\Lambda_{\diamond, h}^{m}=\sum_{\tau=\ell_Q^m}^{m-1} \phi_{\diamond, h}^{\tau}\left(x_{h}^{\tau}, a_{h}^{\tau}\right) \phi_{\diamond, h}^{\tau}\left(x_{h}^{\tau}, a_{h}^{\tau}\right)^{\top}+\lambda I$.
    \STATE $w_{\diamond, h}^{m}=\left(\Lambda_{\diamond, h}^{m}\right)^{-1} \sum_{\tau=\ell_Q^m}^{m-1} \phi_{\diamond, h}^{\tau}\left(x_{h}^{\tau}, a_{h}^{\tau}\right) V_{\diamond, h+1}^{\tau}\left(x_{h+1}^{\tau}\right)$. 
    \STATE $\phi_{\diamond, h}^{m}(\cdot, \cdot)=\int_{\mathcal{S}} \psi\left(\cdot, \cdot, x^{\prime}\right) V_{\diamond, h+1}^{m}\left(x^{\prime}\right) d x^{\prime} $. 
    \STATE $\Gamma_{\diamond, h}^{m}(\cdot, \cdot)=\beta\left(\phi_{\diamond, h}^{m}(\cdot, \cdot)^{\top}\left(\Lambda_{\diamond, h}^{m}\right)^{-1} \phi_{\diamond, h}^{m}(\cdot, \cdot)\right)^{1 / 2}$.
    \STATE 
    $LV=\begin{cases} B_{\mathbb{P},\mathcal{E}} H^2 d_1\sqrt{d_1 W} +B_{g,\mathcal{E}}\sqrt{d_2 W}, \\ \hspace{2cm} \text{Under Assumption \ref{ass: local budget}},\\
    0,  \\
    \hspace{2cm} \text{Under Assumption \ref{ass: Feasibility}.} \end{cases}$
    \STATE $\Lambda_{h}^{m}=\sum_{\tau=\ell_Q^m}^{m-1} \varphi\left(x_{h}^{\tau}, a_{h}^{\tau}\right) \varphi\left(x_{h}^{\tau}, a_{h}^{\tau}\right)^{\top}+\lambda I$ . 
    \STATE $u_{\diamond, h}^{m}=\left(\Lambda_{h}^{m}\right)^{-1} \sum_{\tau=\ell_Q^m}^{m-1} \varphi\left(x_{h}^{\tau}, a_{h}^{\tau}\right) \diamond_{h}^m\left(x_{h}^{\tau}, a_{h}^{\tau}\right) $. 
    \STATE $\Gamma_{h}^{m}(\cdot, \cdot)=\beta\left(\varphi(\cdot, \cdot)^{\top}\left(\Lambda_{h}^{m}\right)^{-1} \varphi(\cdot, \cdot)\right)^{1 / 2}$ . 
    \STATE $Q_{r, h}^{m}(\cdot, \cdot)=\min \left(H-h+1, \varphi(\cdot, \cdot)^{\top} u_{r, h}^{m} + \right.$\\
    $\left. \phi_{r, h}^m(\cdot, \cdot)^\top w_{r, h}^{m} +\left(\Gamma_{h}^{m}+\Gamma_{r, h}^{m}\right)(\cdot, \cdot)\right)_+$,\\
    $Q_{g, h}^{m}(\cdot, \cdot)=\min \left(H-h+1,\varphi(\cdot, \cdot)^{\top} u_{g, h}^{m} + \right.$  \\  $\left. \phi_{g, h}^m(\cdot, \cdot)^\top w_{g, h}^{m} +\left(\Gamma_{h}^{m}+\Gamma_{g, h}^{m}\right)(\cdot, \cdot)+LV\right)_+$. 
    \STATE $V_{\diamond, h}^{m}(\cdot)=\left\langle Q_{\diamond, h}^{m}(\cdot, \cdot), \pi_h^m(\cdot\mid \cdot) \right\rangle^{+}_\mathcal{A}$ .
    \ENDFOR
   \STATE {\bfseries Output:}  $\{Q_{r,h}^m, Q_{g,h}^m\}_{h=1}^H$ and $V_{g,1}^{m}$.
\end{algorithmic}
\end{algorithm}

\subsection{Policy evaluation algorithm for tabular MDP setting}
\subsubsection{Tabular case of Assumption \ref{ass: linear fun approx}}
A special case of the linear CMDP in Assumption \ref{ass: linear fun approx} is the tabular CMDP with $|\mathcal{S}|<\infty$ and $|\mathcal{A}|<\infty$. We take the following feature maps and parameter vectors:
\begin{align*}
   & d_1=|\mathcal{S}|^2|\mathcal{A}|, \ \psi(x,a,x^\prime)=\boldsymbol{e}_{x,a,x'} \in \mathbb{R}^{d_1},  \\ 
   & \theta_h^m =\mathbb{P}_h^m(\cdot,\cdot,\cdot) \in \mathbb{R}^{d_1},  d_2=|\mathcal{S}||\mathcal{A}|, \varphi(x,a)=\boldsymbol{e}_{x,a} \in \mathbb{R}^{d_2}, \\
   & \theta_{r,h}^m =r_h^m(\cdot,\cdot) \in \mathbb{R}^{d_2}, \ \theta_{g,h}^m =g_h^m(\cdot,\cdot) \in \mathbb{R}^{d_2}
\end{align*}
where $\mathbf{e}_{\left(x, a, x^{\prime}\right)}$ is a canonical basis of $\mathbb{R}^{d_{1}}$ associated with $\left(x, a, x^{\prime}\right)$ and the notation $\theta_{h}^m=\mathbb{P}_{h}^m(\cdot, \cdot, \cdot)$ means that the $\left(x, a, x^{\prime}\right)$-th entry of $\theta_{h}$ is $\mathbb{P}^m\left(x^{\prime} \mid x, a\right)$ for every $\left(x, a, x^{\prime}\right) \in \mathcal{S} \times \mathcal{A} \times \mathcal{S}$; Similarly we define $\mathbf{e}_{(x, a)}, \theta^m_{r, h}$, and $\theta^m_{g, h} .$ Thus, 
$$
\begin{aligned}
&\mathbb{P}^m_{h}\left(x^{\prime} \mid x, a\right)=\left\langle\psi\left(x, a, x^{\prime}\right), \theta^m_{h}\right\rangle, \\
& \text { for every }\left(x, a, x^{\prime}\right) \in \mathcal{S} \times \mathcal{A} \times \mathcal{S}; \\
& r^m_{h}(x, a)=\left\langle\varphi(x, a), \theta^m_{r, h}\right\rangle \text { and } g^m_{h}(x, a)=\left\langle\varphi(x, a), \theta^m_{g, h}\right\rangle, \\
&\text { for every }(x, a) \in \mathcal{S} \times \mathcal{A}.
\end{aligned}
$$

One can also verify that
$$
\begin{aligned}
&\left\|\theta_{h}^m\right\|=\left(\sum_{\left(x, a, x^{\prime}\right)}\left|\mathbb{P}^m_{h}\left(x^{\prime} \mid x, a\right)\right|^{2}\right)^{1 / 2} \leq \sqrt{|\mathcal{S}|^{2}|\mathcal{A}|}=\sqrt{d_{1}}, \\
&\left\|\theta_{r, h}^m\right\|=\left(\sum_{(x, a)}\left(r^m_{h}(x, a)\right)^{2}\right)^{1 / 2} \leq \sqrt{|\mathcal{S}||\mathcal{A}|}=\sqrt{d_{2}}, \\
&\left\|\theta_{g, h}^m\right\|=\left(\sum_{(x, a)}\left(g^m_{h}(x, a)\right)^{2}\right)^{1 / 2} \leq \sqrt{|\mathcal{S}||\mathcal{A}|}=\sqrt{d_{2}}
\end{aligned}
$$
and for every $V: \mathcal{S} \rightarrow[0, H]$ and  $(x, a) \in \mathcal{S} \times \mathcal{A}$, we have
\begin{align*}
\left\|\sum_{x^{\prime} \in \mathcal{S}} \psi\left(x, a, x^{\prime}\right) V\left(x^{\prime}\right)\right\|=&\left(\sum_{x^{\prime} \in \mathcal{S}}\left(V\left(x^{\prime}\right)\right)^{2}\right)^{1 / 2} \\
&\leq \sqrt{|\mathcal{S}|} H \leq \sqrt{d_{1}} H .
\end{align*}
Therefore, the tabular CMDP is a special case of Assumption \ref{ass: linear fun approx} with $d:=\max \left(d_{1}, d_{2}\right)=|\mathcal{S}|^{2}|\mathcal{A}|$.

\subsubsection{Tabular Case of Algorithm \ref{alg:algoirthm 1} and Algorithm \ref{alg:algoirthm 2}}

In the tabular case, the improved results can be obtain by incorporating Algorithm \ref{alg:algoirthm 1} with a variant of the optimistic policy evaluation method in Algorithm \ref{alg:algoirthm 2}. We refer the read to Algorithm \ref{alg:algoirthm 3} for such procedures and explain the details below.

We first review some notations for the reader's  convenience. For every $(h, m) \in[H] \times[M]$, every $\left(x, a, x^{\prime}\right) \in \mathcal{S} \times \mathcal{A} \times \mathcal{S}$, $(x, a) \in \mathcal{S} \times \mathcal{A}$, we define two visitation counters $n_{h}^m\left(x, a, x^{\prime}\right)$ and $n_{h}^m(x, a)$ at step $h$ in episode $m$ as follows:
\begin{subequations}
\begin{eqnarray}\label{eq: n h m}
&\hspace{-0.6cm} n_{h}^m\left(x, a, x^{\prime}\right)=\sum_{\tau=\ell^m_Q}^{m-1} 1\left\{\left(x, a, x^{\prime}\right)=\left(x_{h}^{\tau}, a_{h}^{\tau}, a_{h+1}^{\tau}\right)\right\} \\
& n_{h}^m(x, a)=\sum_{\tau=\ell^m_Q}^{m-1} 1\left\{(x, a)=\left(x_{h}^{\tau}, a_{h}^{\tau}\right)\right\} . 
\end{eqnarray}
\end{subequations}
This allows us to estimate the transition kernel $\mathbb{P}_{h}^m$, reward function $r^m$, and utility function $g^m$ for episode $m$ by
\begin{subequations}
\begin{eqnarray}
&\widehat{\mathbb{P}}_{h}^m\left(x^{\prime} \mid x, a\right)=\frac{n_{h}^m\left(x, a, x^{\prime}\right)}{n_{h}^m(x, a)+\lambda}, \\
\nonumber & \hspace{3cm}\text { for all }\left(x, a, x^{\prime}\right) \in \mathcal{S} \times \mathcal{A} \times \mathcal{S}  \label{eq: hat P}\\
&\widehat{\diamond}_{h}^m(x, a)=\frac{\sum_{\tau=\ell^m_Q}^{m-1} 1\left\{(x, a)=\left(x_{h}^{\tau}, a_{h}^{\tau}\right)\right\} \diamond^\tau_{h}\left(x_{h}^{\tau}, a_{h}^{\tau}\right)}{n_{h}^m(x, a)+\lambda}, \\
\nonumber  & \hspace{3cm}\text { for all }(x, a) \in \mathcal{S} \times \mathcal{A},  \diamond = r \text{ or } g \label{eq: hat diamond} 
\end{eqnarray}
\end{subequations}

where $\lambda>0$ is the regularization parameter. Moreover, we introduce the bonus term $\Gamma_{h}^m: \mathcal{S} \times \mathcal{A} \rightarrow \mathbb{R}$ as
\begin{align*}
\Gamma_{h}^m(x, a)&=\beta\left(n_{h}^m(x, a)+\lambda\right)^{-1 / 2},
\end{align*}
which adapts the counter-based bonus terms in the literature \cite{azar2017minimax,jin2018q}, where $\beta>0$ is to be determined later. In addition, the local variation budget term is defined as
$$LV=\begin{cases} B_{\mathbb{P},\mathcal{E}} H  +B_{g,\mathcal{E}}, \quad \text{Under Assumption \ref{ass: local budget}},\\
    0,  \hspace{2.29cm} \text{Under Assumption \ref{ass: Feasibility}.} \end{cases}$$
    
Using the estimated transition kernels $\left\{\widehat{\mathbb{P}}_{h}^m\right\}_{h=1}^{H}$, the estimated reward/utility functions $\left\{\widehat{r}_{h}^m, \widehat{g}_{h}^m\right\}_{h=1}^{H}$, and the bonus terms $\left\{\Gamma_{h}^m\right\}_{h=1}^{H}$, one can estimate the action-value function via
\begin{align*}
&Q_{r, h}^m(x, a)=\min \left(\widehat{r}_{h}^m(x, a)+\sum_{x^{\prime} \in \mathcal{S}} \widehat{\mathbb{P}}_{h}\left(x^{\prime} \mid x, a\right) V_{r, h+1}^m\left(x^{\prime}\right) \right. \\
&\left.+ \Gamma_{h}^m(x, a) +\Gamma_{h,r}^m(x, a) , H-h+1\right)_{+},\\
&Q_{g, h}^m(x, a)=\min \left(\widehat{g}_{h}^m(x, a)+\sum_{x^{\prime} \in \mathcal{S}} \widehat{\mathbb{P}}_{h}\left(x^{\prime} \mid x, a\right) V_{g, h+1}^m\left(x^{\prime}\right)\right.\\
& \left. + \Gamma_{h}^m(x, a) +\Gamma_{h,g}^m(x, a)+LV , H-h+1\right)_{+},\\
\end{align*}
for every $(x, a) \in \mathcal{S} \times \mathcal{A}$, where $\diamond=r$ or $g .$ Thus, $V_{\diamond, h}^m(x)=\left\langle Q_{\diamond, h}^m(x, \cdot), \pi_{h}^m(\cdot \mid x)\right\rangle_{\mathcal{A}} .$ We summarize the above procedure in Algorithm \ref{alg:algoirthm 3}. Using the already estimated values $\left\{Q_{r, h}^m(\cdot, \cdot), Q_{g, h}^m(\cdot, \cdot)\right\}_{h=1}^{H}$ and $V_{g,1}^{m}$, we can execute the policy improvement and the dual update in Algorithm \ref{alg:algoirthm 1}.

\begin{algorithm}[tb]
   \caption{Optimistic Policy Evaluation (OPE)}
   \label{alg:algoirthm 3}
\begin{algorithmic}[1]
 \STATE {\bfseries Inputs:} $\{x_h^\tau, a_h^\tau, r_h^\tau(x_h^\tau,a_h^\tau), g_h^\tau(x_h^\tau, a_h^\tau)\}_{h=1,\tau=\ell_Q^m}^{H,m}$, regularization parameter $\lambda$, UCB parameter $\beta$, local variation budgets $B_{\mathbb{P},\mathcal{E}}, B_{g,\mathcal{E}}$.
\FOR{ $h=H, H-1,\ldots,1$}
\STATE Compute counters $n_{h}^{m}\left(x, a, x^{\prime}\right)$ and $n_{h}^{m}(x, a)$ via \eqref{eq: n h m} for all $\left(x, a, x^{\prime}\right) \in \mathcal{S} \times \mathcal{A} \times \mathcal{S}$ and $(x, a) \in$ $\mathcal{S} \times \mathcal{A}$.
\STATE  Estimate reward/utility functions $\widehat{r}_{h}^{m}, \widehat{g}_{h}^{m}$ via \eqref{eq: hat diamond} for all $(x, a) \in \mathcal{S} \times \mathcal{A}$.
\STATE Estimate transition $\widehat{\mathbb{P}}_{h}^{m}$ via \eqref{eq: hat P} for all $\left(x, a, x^{\prime}\right) \in \mathcal{S} \times \mathcal{A} \times \mathcal{S}$, and take bonus $\Gamma_{h}^{m}=\beta\left(n_{h}^{m}(x, a)+\lambda\right)^{-1 / 2}$ for all $(x, a) \in \mathcal{S} \times \mathcal{A} .$
\STATE 
    $LV=\begin{cases} B_{\mathbb{P},\mathcal{E}} H  +B_{g,\mathcal{E}}, \quad \text{Under Assumption \ref{ass: local budget}},\\
    0,  \hspace{2.29cm} \text{Under Assumption \ref{ass: Feasibility}.} \end{cases}$
\STATE $ Q_{r, h}^{m}(\cdot, \cdot)=\min \left(H-h+1, \widehat{r}_{h}^{m}(\cdot, \cdot)+\right.$
$\left.\sum_{x^{\prime} \in \mathcal{S}} \widehat{\mathbb{P}}_{h}\left(x^{\prime} \mid \cdot, \cdot\right) V_{r, h+1}^{m}\left(x^{\prime}\right)+2\Gamma_{h}^{m}(\cdot, \cdot) \right)_+$, \\
$ Q_{g, h}^{m}(\cdot, \cdot)=\min \left(H-h+1, \widehat{g}_{h}^{m}(\cdot, \cdot)+\right.$
$\left.\sum_{x^{\prime} \in \mathcal{S}} \widehat{\mathbb{P}}_{h}\left(x^{\prime} \mid \cdot, \cdot\right) V_{g, h+1}^{m}\left(x^{\prime}\right)+2\Gamma_{h}^{m}(\cdot, \cdot) +LV\right)_+$.
\STATE $V_{\diamond, h}^{m}(\cdot)=\left\langle Q_{\diamond, h}^{m}(\cdot, \cdot), \pi_{h}^{m}(\cdot \mid \cdot)\right\rangle_{\mathcal{A}} .$
    \ENDFOR
   \STATE {\bfseries Output:}  $\{Q_{r,h}^m, Q_{g,h}^m\}_{h=1}^H$ and $V_{g,1}^{m}$.
\end{algorithmic}
\end{algorithm}

%% file: files/appe_regret.tex
\section{Proof for linear kernel CMDP case under Assumption \ref{ass: local budget}}\label{sec:appe-linear MDP under local budget}

\subsection{Proof of dynamic regret bound in Theorem \ref{thm: Linear Kernal MDP under local budget}}\label{sec:appe-regret bound}
Our analysis for the dynamic regret begins with the decomposition of the regret given in \eqref{eq: d regret}:
\begin{lemma}[Dynamic regret decomposition]
The dynamic regret in \eqref{eq: d regret} can be expanded as
\begin{align*}
&\operatorname{DR}(M)\\
=&\sum_{m=1}^M\sum_{h=1}^H \EE_{{\pi}^{\star,\ell_\pi^m}, \mathbb{P}^{\ell_\pi^m}} \left[ \inner{Q_{r,h}^m(x_h,\cdot)}{{\pi}_h^{\star,m}(\cdot\mid x_h) - \pi_h^m(\cdot\mid x_h)}\right] \\
&+\sum_{m=1}^M\sum_{h=1}^H \left( \EE_{{\pi}^{\star,m}, \mathbb{P}^m} - \EE_{{\pi}^{\star,\ell_\pi^m}, \mathbb{P}^{\ell_\pi^m}}  \right) \\
&\hspace{2cm} \left[ \inner{Q_{r,h}^m(x_h,\cdot)}{{\pi}_h^{\star,m}(\cdot\mid x_h) - \pi_h^m(\cdot\mid x_h)}\right] \\
& +\sum_{m=1}^M\sum_{h=1}^H \EE_{{\pi}^{\star,m}, \mathbb{P}^m} \left[ \iota_{r, h}^{m}(x_h,a_h) \right] \\
&\hspace{2cm} - \sum_{m=1}^{M} \sum_{h=1}^{H} \iota_{r, h}^{m}\left(x_{h}^{m}, a_{h}^{m}\right)+S_{r, H, 2}^{M}
\end{align*}
where $\left\{S_{r, h, k}^{m}\right\}_{(m, h, k) \in[M] \times[H] \times[2]}$ is a martingale.
\end{lemma}
\begin{proof}
We have
\begin{align} \label{eq: decomposed d regret}
\operatorname{DR}(M)=&
\underbrace{\sum_{m=1}^M\left(V_{r,1}^{{\pi}^{\star,m},m}(x_1)-V_{r,1}^{m}(x_1) \right)}_\text{(R.\rom{1})}\\ \nonumber &+\underbrace{\sum_{m=1}^M\left(V_{r,1}^{m}(x_1) -V_{r,1}^{\pi^m,m}(x_1)\right)}_\text{(R.\rom{2})},
\end{align}
where the policy $\pi^{\star,m}$ is the best policy in hindsight for the constrained optimization problem \eqref{eq: CMDP at episode m} at episode $m$;
the policy $\pi^m$ is the policy updated in line 10 of Algorithm \ref{alg:algoirthm 1};
$V_{r,1}^{\pi^{\star,m},m}, V_{r,1}^{\pi^m,m}(x_1)$ are the value functions corresponding to the policies $\pi^{\star,m}$ and $\pi^m$, and the value function  $V_{r,1}^{m}(x_1)$ is estimated from an optimistic policy evaluation by Algorithm \ref{alg:algoirthm 2}. To bound the total regret \eqref{eq: decomposed d regret}, we need to analyze the two terms \text{(R.\rom{1})} and \text{(R.\rom{2})} separately.

To analyze the first term (R.\rom{1}), we define the model prediction error for the reward at episode $m$ as 
\begin{align} \label{eq: model prediction error r}
   \iota_{r, h}^{m} \coloneqq r_h^m+ \mathbb{P}_h^m V_{r,h+1}^m - Q_{r,h}^m
\end{align}
for all $(m,h) \in [M] \times [H]$, which describes the error in the Bellman equation \eqref{eq: belmman equation} using $V_{r,h+1}^m$ instead of $V_{r,h+1}^{\pi^m,m}$ and using the sample estimation of $ \mathbb{P}_h^m$. With this notation, we expand the term (R.\rom{1}) in \eqref{eq: decomposed d regret} into 

\begin{align}  \label{eq: R.2}
\nonumber&\sum_{m=1}^M\sum_{h=1}^H \EE_{\pi^{\star,m}, \mathbb{P}^m} \left[ \inner{Q_{r,h}^m(x_h,\cdot)}{\pi_h^{\star,m}(\cdot\mid x_h) - \pi_h^m(\cdot\mid x_h)}\right]\\
&+\sum_{m=1}^M\sum_{h=1}^H \EE_{\pi^{\star,m}, \mathbb{P}^m} \left[ \iota_{r, h}^{m}(x_h,a_h) \right]
\end{align}
where the first double sum is linear in terms of the policy difference and the second one describes the total model prediction errors. The above expansion is proved in Lemma \ref{lemma: expansion of V star m -V m}. We can further decompose the first term in \eqref{eq: R.2} as

\begin{align}
\nonumber &\sum_{m=1}^M\sum_{h=1}^H \EE_{\pi^{\star,m}, \mathbb{P}^m} \left[ \inner{Q_{r,h}^m(x_h,\cdot)}{\pi_h^{\star,m}(\cdot\mid x_h) - \pi_h^m(\cdot\mid x_h)}\right]\\
\nonumber =& \sum_{m=1}^M\sum_{h=1}^H \EE_{\pi^{\star,\ell_\pi^m}, \mathbb{P}^{\ell_\pi^m}} \left[ \inner{Q_{r,h}^m(x_h,\cdot)}{\pi_h^{\star,m}(\cdot\mid x_h) - \pi_h^m(\cdot\mid x_h)}\right] \\
\nonumber&+\sum_{m=1}^M\sum_{h=1}^H \left( \EE_{\pi^{\star,m}, \mathbb{P}^m} - \EE_{\pi^{\star,\ell_\pi^m}, \mathbb{P}^{\ell_\pi^m}}  \right) \cdot \\
&\hspace{1cm}\left[ \inner{Q_{r,h}^m(x_h,\cdot)}{\pi_h^{\star,m}(\cdot\mid x_h) - \pi_h^m(\cdot\mid x_h)}\right] \label{eq: R.2 further decomp}.
\end{align}

To analyze the second term (R.\rom{2}) in \eqref{eq: decomposed d regret}, we will first introduce some notations. 
for every $(m,h)\in [M] \times [H]$, we define $\mathcal{F}_{h,1}^m$ as a $\sigma-$algebra generated by state-action sequences, reward and utility functions,
\begin{align*}
    \left\{ (x_i^\tau,a_i^\tau) \right\}_{(\tau,i)\in[m-1]\times[H]} \bigcup \left\{ (x_i^m,a_i^m) \right\}_{i\in[H]}.
\end{align*}

Similarly, we define $\mathcal{F}_{h,2}^m$ as an $\sigma-$algebra generated by
\begin{align*}
    \left\{ (x_i^\tau,a_i^\tau) \right\}_{(\tau,i)\in[m-1]\times[H]} \bigcup \left\{ (x_i^m,a_i^m) \right\}_{i\in[H]} \bigcup \left\{ x_{h+1}^m \right\}.
\end{align*}

Here, $x_{H+1}^m$ is a null state for every $m\in[M]$. A filtration is a sequence of $\sigma-$algebras $\{\mathcal{F}^m_{h,k}\}_{(k,h,m) \in [K] \times [H] \times [2]}$ in terms of the time index

\begin{align} \label{eq: definition t}
    t(m,h,k)\coloneqq 2(m-1)H+2(h-1)+k
\end{align}

such that ${F}^m_{h,k} \subset {F}^{m'}_{h',k'}$ for every $t(m,h,k) \leq t(m',h',k')$. The estimated reward/utility value functions $V_{r,h}^m, V_{g,h}^m$ and the associated Q-functions $Q_{r,h}^m, Q_{g,h}^m$ are $\mathcal{F}_{1,1}^m$  measurable since they are obtained from previous $m-1$ historical
trajectories. With these notations, we can expand the term (R.\rom{2}) in \eqref{eq: decomposed d regret} into 
\begin{align}  \label{eq: R.3}
-\sum_{m=1}^{M} \sum_{h=1}^{H} \iota_{r, h}^{m}\left(x_{h}^{m}, a_{h}^{m}\right)+S_{r, H, 2}^{M}
\end{align}
where $\left\{S_{r, h, k}^{m}\right\}_{(m, h, k) \in[M] \times[H] \times[2]}$ is a martingale adapted to the filtration $\left\{\mathcal{F}_{h, k}^{m}\right\}_{(m, h, k) \in[M] \times[H] \times[2]}$ in terms of the time index $t$. We define $S_{r, H, 2}^{M}$ and  prove \eqref{eq: R.3} in Lemma \ref{lemma: expansion of V m -V pi m}.
\end{proof}

In the following proofs,  we use the shorthand notation $\left\langle Q_{r, h}^{m-1}+\mu^{m-1} Q_{g, h}^{m-1}, \pi_{h}\right\rangle$ for $\left\langle\left(Q_{r, h}^{m-1}+\mu^{m-1} Q_{g, h}^{m-1}\right)\left(x_{h}, \cdot\right), \pi_{h}\left(\cdot \mid x_{h}\right)\right\rangle$
and the shorthand notation $D\left(\pi_{h} \mid \pi_{h}^{m-1}\right)$ for $D\left(\pi_{h}\left(\cdot \mid x_{h}\right) \mid \pi_{h}^{m-1}\left(\cdot \mid x_{h}\right)\right)$ 
if dependence on the state-action sequence $\left\{x_{h}, a_{h}\right\}_{h=1}^{H}$ is clear from the context.

\begin{lemma}[Primal step for dynamic regret] \label{lemma: policy Improvement: Primal-Dual Mirror Descent Step}
Let Assumption \ref{ass: linear fun approx} hold. For the primal update rule in line 10 of Algorithm \ref{alg:algoirthm 1}, we have
\begin{align}\label{eq: 22}
\nonumber & \sum_{h=1}^{H} \left\langle Q_{r, h}^{m-1}, \pi_{h}^{\star,m-1}-\pi_{h}^{m-1}\right\rangle\\
\nonumber &\leq-\mu^{m-1} \sum_{h=1}^{H}\left\langle Q_{g, h}^{m-1},\pi_{h}^{\star,m-1}-\pi_{h}^{m-1}\right\rangle + \frac{\alpha(1+\mu^{m-1})^{2} H^{2}}{2}\\
&+\frac{1}{\alpha} \sum_{h=1}^{H} \left[D\left(\pi_{h}^{\star,m-1} \mid \pi_{h}^{m-1}\right)-D\left(\pi_{h}^{\star,m-1} \mid \pi_{h}^{m}\right)\right] 
\end{align}
\end{lemma}
\begin{proof}
This result follows immediately from the "one-step descent" lemma in Lemma \ref{lemma: One-step descent lemma} and the fact $Q_{r, h}^{m-1} + \mu^{m-1} Q_{g, h}^{m-1} \in [0,(1+\mu^{m-1}) H]$.

\end{proof}


\begin{lemma}[Bound for the first term in \eqref{eq: R.2 further decomp}]\label{lemma: Bound for equation  R.2-1}
Let  Assumption \ref{ass: linear fun approx} hold. Then
\begin{align*}
&\sum_{m=1}^M\sum_{h=1}^H \EE_{{\pi}^{\star,\ell_\pi^m}, \mathbb{P}^{\ell_\pi^m}} \left[ \inner{Q_{r,h}^m}{{\pi}_h^{\star,m}- \pi_h^m(\cdot\mid x_h)}\right] \\
&\leq-\sum_{m=1}^M \mu^{m} \sum_{h=1}^{H}  \EE_{{\pi}^{\star,\ell_\pi^m}, \mathbb{P}^{\ell_\pi^m}} \left[\left\langle Q_{g, h}^{m}, {\pi}_{h}^{\star,m}-\pi_{h}^{m}\right\rangle  \right]\\
&+ \alpha H^{2} \sum_{m=1}^M(1+|\mu^{m}|^{2}) + \frac{1}{\alpha} H M L^{-1} \log{|\mathcal{A}|} +H^2L B_{\star}
\end{align*}
\end{lemma}
\begin{proof}
By further decomposing the first term in \eqref{eq: R.2 further decomp}, we obtain
\begin{align} \label{eq: R.2-1 decomposition}
\nonumber &\sum_{m=1}^M\sum_{h=1}^H \EE_{{\pi}^{\star,\ell_\pi^m}, \mathbb{P}^{\ell_\pi^m}} \left[ \inner{Q_{r,h}^m}{{\pi}_h^{\star,m}- \pi_h^m(\cdot\mid x_h)}\right] \\
\nonumber &= \sum_{m=1}^M\sum_{h=1}^H \EE_{{\pi}^{\star,\ell_\pi^m}, \mathbb{P}^{\ell_\pi^m}} \left[ \inner{Q_{r,h}^m}{{\pi}_h^{\star,\ell_\pi^m}- \pi_h^m(\cdot\mid x_h)}\right]\\
&+\sum_{m=1}^M\sum_{h=1}^H \EE_{{\pi}^{\star,\ell_\pi^m}, \mathbb{P}^{\ell_\pi^m}} \left[ \inner{Q_{r,h}^m}{{\pi}_h^{\star,m}- \pi_h^{\star,\ell_\pi^m}(\cdot\mid x_h)}\right]
\end{align}
where $\pi_h^{\star,\ell_\pi^m}$ is the optimal policy at episode $\ell_\pi^m$.

For the first term in \eqref{eq: R.2-1 decomposition}, we have
\begin{align*}
&\sum_{m=1}^M\sum_{h=1}^H \EE_{{\pi}^{\star,\ell_\pi^m}, \mathbb{P}^{\ell_\pi^m}} \left[ \inner{Q_{r,h}^m}{{\pi}_h^{\star,\ell_\pi^m}- \pi_h^m(\cdot\mid x_h)}\right] \\
\leq&-\sum_{m=1}^M \mu^{m} \sum_{h=1}^{H}  \EE_{{\pi}^{\star,\ell_\pi^m}, \mathbb{P}^{\ell_\pi^m}} \left[\left\langle Q_{g, h}^{m}, {\pi}_{h}^{\star,m}-\pi_{h}^{m}\right\rangle  \right]\\
&+ \alpha H^{2} \sum_{m=1}^M(1+|\mu^m|^2)\\
&+\frac{1}{\alpha}  \sum_{h=1}^{H} \sum_{\mathcal{E}=1}^{\ceil{\frac{M}{L}}}  \EE_{{\pi}^{\star,(\mathcal{E}-1)L}, \mathbb{P}^{(\mathcal{E}-1)L}} \left[ \sum_{m=(\mathcal{E}-1)L}^{\mathcal{E}L} D\left({\pi}_{h}^{\star,m} \mid {\pi}_{h}^{m}\right)\right.\\ 
&\left. -D\left({\pi}_{h}^{\star,m} \mid \pi_{h}^{m+1}\right)\right] \\
\leq&-\sum_{m=1}^M \mu^{m} \sum_{h=1}^{H}  \EE_{{\pi}^{\star,\ell_\pi^m}, \mathbb{P}^{\ell_\pi^m}} \left[\left\langle Q_{g, h}^{m}, {\pi}_{h}^{\star,m}-\pi_{h}^{m}\right\rangle  \right]\\
& +  \alpha H^{2} \sum_{m=1}^M(1+|\mu^m|^2)\\
&+\frac{1}{\alpha}  \sum_{h=1}^{H} \sum_{\mathcal{E}=1}^{\ceil{\frac{M}{L}}}  \EE_{{\pi}^{\star,(\mathcal{E}-1)L}, \mathbb{P}^{(\mathcal{E}-1)L}} \left[ D\left({\pi}_{h}^{\star,(\mathcal{E}-1)L} \mid {\pi}_{h}^{(\mathcal{E}-1)L}\right)\right] \\
\leq&-\sum_{m=1}^M \mu^{m} \sum_{h=1}^{H}  \EE_{{\pi}^{\star,\ell_\pi^m}, \mathbb{P}^{\ell_\pi^m}} \left[\left\langle Q_{g, h}^{m}, {\pi}_{h}^{\star,m}-\pi_{h}^{m}\right\rangle  \right]\\
& +  \alpha H^{2} \sum_{m=1}^M(1+|\mu^m|^2)+ \frac{1}{\alpha} H M L^{-1} \log{|\mathcal{A}|}
\end{align*}
where the first inequality follows from Lemma \ref{lemma: policy Improvement: Primal-Dual Mirror Descent Step} and the fact that $(1+|\mu^m|)^2\leq 2+2|\mu^m|^2$, the second inequality results from the telescoping, and the last inequality is due to
\begin{align*}
&D\left({\pi}_{h}^{\star,(\mathcal{E}-1)L} \mid {\pi}_{h}^{(\mathcal{E}-1)L}\right)\\
=& \sum_{a\in \mathcal{A}} {\pi}_{h}^{\star,(\mathcal{E}-1)L}   \cdot \log \left( |\mathcal{A}| \cdot {\pi}_{h}^{\star,(\mathcal{E}-1)L}\right) \leq \log |\mathcal{A}|. 
\end{align*}
For the second term in \eqref{eq: R.2-1 decomposition}, it holds that
\begin{align*}
&\sum_{m=1}^M\sum_{h=1}^H \EE_{{\pi}^{\star,\ell_\pi^m}, \mathbb{P}^{\ell_\pi^m}} \left[ \inner{Q_{r,h}^m}{{\pi}_h^{\star,m}(\cdot\mid x_h)- \pi_h^{\star,\ell_\pi^m}(\cdot\mid x_h)}\right]\\
\leq & \sum_{m=1}^M\sum_{h=1}^H \EE_{{\pi}^{\star,\ell_\pi^m}, \mathbb{P}^{\ell_\pi^m}} \left[ H \norm{ {\pi}_h^{\star,m}(\cdot\mid x_h)- \pi_h^{\star,\ell_\pi^m}(\cdot\mid x_h) }_1\right]\\
\leq & \sum_{m=1}^M\sum_{h=1}^H  H \max_{x_h \in \mathcal{S}} \norm{ {\pi}_h^{\star,m}(\cdot\mid x_h)- \pi_h^{\star,\ell_\pi^m}(\cdot\mid x_h) }_1 \\
\leq & \sum_{\mathcal{E}=1}^{\ceil{\frac{M}{L}}}\sum_{h=1}^H \sum_{m=(\mathcal{E}-1)L+1} ^{\mathcal{E}L} H B_{\star,\mathcal{E}}\\
\leq & H^2 L B_{\star}
\end{align*}
where the first policy holds by Holder's inequality and the fact that $\norm{Q_h^m(s,\cdot)}_\infty \leq H$, the third step  is due to the definition of $B_{\star,\mathcal{E}}=\sum_{m=(\mathcal{E}-1)L+1}^{\mathcal{E}L}\sum_{h=1}^H \norm{\pi_h^{\star,m}-\pi_h^{\star,m-1}}_\infty$ and the last inequality follows from the definition of $B_\star=\sum_{m=1}^M \sum_{h=1}^H \norm{\pi_h^{\star,m}-\pi_h^{\star,m-1}}_\infty$. This completes the proof.
\end{proof}

\begin{lemma}[Bound for the second term in \eqref{eq: R.2 further decomp}] \label{lemma: Bound for equation  R.2-2}
Let Assumption \ref{ass: linear fun approx} hold. Then
\begin{align*}
    & \sum_{m=1}^M\sum_{h=1}^H \left( \EE_{\pi^{\star,m}, \mathbb{P}^m} - \EE_{\pi^{\star,\ell_\pi^m}, \mathbb{P}^{\ell_\pi^m}}  \right) \cdot\\
    &\left[ \inner{Q_{r,h}^m(x_h,\cdot)}{\pi_h^{\star,m}(\cdot\mid x_h) - \pi_h^m(\cdot\mid x_h)}\right]
   \\
   &\leq
    2H^2L \left(\sqrt{d_1} B_\mathbb{P}+ B_\star \right).
\end{align*}
\end{lemma}
\begin{proof}
We denote by $\mathbb{1}(x_h)$ the indicator function for state $x_h$. It holds that
\begin{align*}
& \sum_{m=1}^M\sum_{h=1}^H \left( \EE_{\pi^{\star,m}, \mathbb{P}^m} - \EE_{\pi^{\star,\ell_\pi^m}, \mathbb{P}^{\ell_\pi^m}}  \right)\cdot\\
&\left[ \inner{Q_{r,h}^m(x_h,\cdot)}{\pi_h^{\star,m}(\cdot\mid x_h) - \pi_h^m(\cdot\mid x_h)}\right] \\
\leq & \sum_{m=1}^M\sum_{h=1}^H \left( \EE_{\pi^{\star,m}, \mathbb{P}^m} - \EE_{\pi^{\star,\ell_\pi^m}, \mathbb{P}^{\ell_\pi^m}}  \right)\left[ 2 H \mathbb{1}(x_h)\right] \\
= & 2 H \sum_{\mathcal{E}=1}^{\ceil{\frac{M}{L}}} \sum_{m=(\mathcal{E}-1)L+1}^{\mathcal{E}L}\sum_{h=1}^H  \sum_{j=(\mathcal{E}-1)L+2}^m \\
&\left( \EE_{\pi^{\star,j}, \mathbb{P}^j} - \EE_{\pi^{\star,j-1}, \mathbb{P}^{j-1}}  \right)\left[  \mathbb{1}(x_h)\right] \\
\leq & 2 HL  \sum_{\mathcal{E}=1}^{\ceil{\frac{M}{L}}} \sum_{h=1}^H  \sum_{j=(\mathcal{E}-1)L+1}^{\mathcal{E}L}\left( \EE_{\pi^{\star,j}, \mathbb{P}^j} - \EE_{\pi^{\star,j-1}, \mathbb{P}^{j-1}}  \right)\left[  \mathbb{1}(x_h)\right] \\
\leq & 2 H^2L \left(\sqrt{d_1} B_\mathbb{P}+ B_\star \right)
\end{align*}
where the first step follows from $\left| \inner{Q_{r,h}^m(x_h,\cdot)}{\pi_h^{\star,m}(\cdot\mid x_h) - \pi_h^m(\cdot\mid x_h)}\right| \leq 2  H  \mathbb{1}(x_h)$, the third steps holds by telescoping, and the last step follows from Lemma \ref{lemma: probability difference in expectation}. This completes the proof.
\end{proof}

\begin{lemma}[Dual step for dynamic regret]  \label{lemma: dual step}
It holds that
\begin{align*}
  & -\sum_{m=1}^{M} \mu^{m}\left(V_{g, 1}^{\pi^{\star,m},m}\left(x_{1}\right)-V_{g, 1}^{m}\left(x_{1}\right)\right)  \\
 \leq &  {\eta H^{2}(M+1)} +\sum_{m=1}^{M+1} (\eta \xi^2-\xi) | \mu^{m-1}|^2.
\end{align*}
\end{lemma}
\begin{proof}
By the dual update in line 14 in Algorithm \ref{alg:algoirthm 1} and $\chi=\infty$, we have
$$
\begin{aligned}
0 & \leq\left(\mu^{m+1}\right)^{2} \\
&=\sum_{m=1}^{M+1}\left(\left(\mu^{m}\right)^{2}-\left(\mu^{m-1}\right)^{2}\right) \\
&=\sum_{m=1}^{M+1}\left(\left(\mu^{m-1}+\eta\left(b^m-\xi \mu^{m-1} -V_{g, 1}^{m-1}(x_{1}\right)\right)\right)^{2}-\left(\mu^{m-1}\right)^{2} \\
& \leq \sum_{m=1}^{M+1} 2 \eta \mu^{m-1}\left(V_{g, 1}^{\pi^{\star,m-1}}\left(x_{1}\right)-\xi \mu^{m-1} -V_{g, 1}^{m-1}\left(x_{1}\right)\right)\\
& +\eta^{2}\left(b^m-\xi \mu^{m-1} -V_{g, 1}^{m-1}\left(x_{1}\right)\right)^{2}
\end{aligned}
$$
where we use the feasibility of $\pi^{\star,m-1}$ in the last inequality. Since $\mu^{0}=0$ and $\left|b^m-V_{g, 1}^{m-1}\left(x_{1}\right)\right| \leq H$, the above inequality implies that
\begin{align} \label{eq: 28}
\nonumber& -\sum_{m=1}^{M} \mu^{m-1}\left(V_{g, 1}^{\pi^{\star,m},m}\left(x_{1}\right)-V_{g, 1}^{m}\left(x_{1}\right)\right)\\
\leq & \sum_{m=1}^{M+1} \frac{\eta}{2}\left(b^m-\xi \mu^{m-1} -V_{g, 1}^{m-1}\left(x_{1}\right)\right)^{2}-\sum_{m=1}^{M+1} \xi | \mu^{m-1}|^2 \\
\nonumber \leq & \sum_{m=1}^{M+1} {\eta}\left(b^m -V_{g, 1}^{m-1}\left(x_{1}\right)\right)^{2} +\sum_{m=1}^{M+1} (\eta \xi^2-\xi) | \mu^{m-1}|^2 \\
\leq &  {\eta H^{2}(M+1)} +\sum_{m=1}^{M+1} (\eta \xi^2-\xi) | \mu^{m-1}|^2.
\end{align}
This completes the proof.
\end{proof}

\begin{lemma}[Model prediction error bound for dynamic regret] \label{lemma: Model prediction difference bound with local knowledge}
Let   Assumption \ref{ass: linear fun approx} and \ref{ass: local budget} hold. Fix $p \in(0,1)$ and let $\mathcal{E}$ be the epoch that the episode $m$ belongs to.
If we set $\lambda=1$, $LV=B_{\mathbb{P},\mathcal{E}} H^2 d_1\sqrt{d_1 W} +B_{g,\mathcal{E}}\sqrt{d_2 W},$ and
\begin{align*}
&\Gamma_{h}^{m}(\cdot, \cdot)=\beta\left(\varphi(\cdot, \cdot)^{\top}\left(\Lambda_{h}^{m}\right)^{-1} \varphi(\cdot, \cdot)\right)^{1 / 2}, \\
&\Gamma_{r,h}^m =\beta \left( (\phi_{r,h}^m)^\top (\Lambda_{r, h}^m)^{-1} \phi_{r,h}^m \right)^{1/2}, \\
&\Gamma_{g,h}^m =\beta \left( (\phi_{g,h}^m)^\top (\Lambda_{g,h}^m)^{-1} \phi_{g,h}^m \right)^{1/2}
\end{align*}
with $\beta=$ $C_{1} \sqrt{d H^{2} \log (d W / p)}$ in Algorithm \ref{alg:algoirthm 2}
, then with probability at least $1-p / 2$ it holds that
\begin{align*}
&\sum_{m=1}^{M} \sum_{h=1}^{H}\left(\mathbb{E}_{\pi^{\star,m}, \mathbb{P}^{m}}\left[\iota_{r, h}^{m}\left(x_{h}, a_{h}\right)+\mu^m\iota_{g, h}^{m}\left(x_{h}, a_{h}\right)\right]\right.\\
& \left.-\iota_{r, h}^{m}\left(x_{h}^{m}, a_{h}^{m}\right)\right)\\
\leq & C_2 d H^2 M W^{-\frac{1}{2}}  \sqrt{\log \left(dH^2W+1 \right) \log \left(\frac{dW}{p} \right)} \\
& + B_{\mathbb{P}} H^3 d_1 W\sqrt{d_1 W} +B_{r} HW\sqrt{d_2 W}
\end{align*}
where $C_{1}$ and $C_2$ are absolute constants.
\end{lemma}
\begin{proof}
By Lemma \ref{lemma: bounds on model prediction error with local knowledge}, for every $(m, h) \in[M] \times[H]$ and $(x, a) \in \mathcal{S} \times \mathcal{A}$, the following inequality holds with probability at least 1-p/2:
\begin{align*}
-2\left(\Gamma_{h}^{m}+\Gamma_{r, h}^{m}\right)(x, a)-B_{\mathbb{P},\mathcal{E}} H^2 d_1\sqrt{d_1 W} -B_{r,\mathcal{E}}\sqrt{d_2 W}\\
\leq \iota_{r, h}^{m}(x, a) \leq B_{\mathbb{P},\mathcal{E}} H^2 d_1\sqrt{d_1 W} +B_{r,\mathcal{E}}\sqrt{d_2 W} .
\end{align*}

By the definition of $\iota_{r, h}^{m}(x, a)$, we have $\left|\iota_{r, h}^{m}(x, a)\right| \leq 2 H .$ Hence, it holds with probability at least $1-p / 2$ that
\begin{align*}
&\mathbb{E}_{\pi^{\star,m}, \mathbb{P}^{m}}\left[\iota_{r, h}^{m}\left(x_{h}, a_{h}\right)\right]-\iota_{r, h}^{m}(x, a)\\
\leq &2 \min \left(H, \left(\Gamma_{h}^{m}+\Gamma_{r, h}^{m}\right)(x, a)+B_{\mathbb{P},\mathcal{E}} H^2 d_1\sqrt{d_1 W} \right.\\
&\hspace{5cm}\left.+B_{r,\mathcal{E}}\sqrt{d_2 W}\right)
\end{align*}
for every $(m, h) \in[M] \times[H]$ and $(x, a) \in \mathcal{S} \times \mathcal{A}$, where $\Gamma_{h}^{m}(\cdot, \cdot)=\beta\left(\varphi(\cdot, \cdot)^{\top}\left(\Lambda_{h}^{m}\right)^{-1} \varphi(\cdot, \cdot)\right)^{1 / 2}$ and $\Gamma_{r, h}^{m}(\cdot, \cdot)=\beta\left(\phi_{r, h}^{m}(\cdot, \cdot)^{\top}\left(\Lambda_{r, h}^{m}\right)^{-1} \phi_{r, h}^{m}(\cdot, \cdot)\right)^{1 / 2}$. Therefore, we have
\begin{align*}
&\sum_{m=1}^{M} \sum_{h=1}^{H}\left(\mathbb{E}_{\pi^{\star,m}, \mathbb{P}^{m}}\left[\iota_{r, h}^{m}\left(x_{h}, a_{h}\right) \mid x_{1}\right]-\iota_{r, h}^{m}\left(x_{h}^{m}, a_{h}^{m}\right)\right) \\
\leq& 2 \sum_{m=1}^{M} \sum_{h=1}^{H} \min \left(H,\left(\Gamma_{h}^{m}+\Gamma_{r, h}^{m}\right)\left(x_{h}^{m}, a_{h}^{m}\right) \right.\\
&\left.+B_{\mathbb{P},\mathcal{E}} H^2 d_1\sqrt{d_1 W} +B_{r,\mathcal{E}}\sqrt{d_2 W}\right)\\
\leq& 2 \sum_{m=1}^{M} \sum_{h=1}^{H} \min \left(H,\left(\Gamma_{h}^{m}+\Gamma_{r, h}^{m}\right)\left(x_{h}^{m}, a_{h}^{m}\right)\right)  \\ 
&+ 2 \sum_{\mathcal{E}=1}^{\ceil{\frac{M}{W}}} \left(B_{\mathbb{P},\mathcal{E}} H^3 d_1 W\sqrt{d_1 W} +B_{r,\mathcal{E}}HW\sqrt{d_2 W}\right)\\
\leq& 2 \sum_{m=1}^{M} \sum_{h=1}^{H} \min \left(H,\left(\Gamma_{h}^{m}+\Gamma_{r, h}^{m}\right)\left(x_{h}^{m}, a_{h}^{m}\right)\right)\\
&+ 2B_{\mathbb{P}} H^3 d_1 W\sqrt{d_1 W} +2B_{r}HW\sqrt{d_2 W}
\end{align*}
where the last inequality follows from the definition of the variation budgets $B_{\mathbb{P}}\coloneqq \sum_{\mathcal{E}=1}^{\ceil{\frac{M}{W}}}B_{\mathbb{P},\mathcal{E}}$ and $B_{r}\coloneqq \sum_{\mathcal{E}=1}^{\ceil{\frac{M}{W}}} B_{r,\mathcal{E}}$. It results from the Cauchy-Schwartz inequality that
\begin{align}\label{eq: 31}
 \nonumber&\sum_{m=1}^{M} \sum_{h=1}^{H} \min \left(H,\left(\Gamma_{h}^{m}+\Gamma_{r, h}^{m}\right)\left(x_{h}^{m}, a_{h}^{m}\right)\right) \\
\nonumber &=  \sum_{\mathcal{E}=1}^{\ceil{\frac{M}{W}}}\sum_{m=(\mathcal{E}-1)W}^{\mathcal{E}W} \sum_{h=1}^{H} \min \left(H,\left(\Gamma_{h}^{m}+\Gamma_{r, h}^{m}\right)\left(x_{h}^{m}, a_{h}^{m}\right)\right) \\
 \nonumber&\leq \beta \sum_{\mathcal{E}=1}^{\ceil{\frac{M}{W}}}\sum_{m=(\mathcal{E}-1)W}^{\mathcal{E}W} \sum_{h=1}^{H} \min \left(H / \beta, \right.\\
&\left.\left(\varphi\left(x_{h}^{m}, a_{h}^{m}\right)^{\top}\left(\Lambda_{h}^{m}\right)^{-1} \varphi\left(x_{h}^{m}, a_{h}^{m}\right)\right)^{1 / 2}\right.\\
\nonumber &\left.+\left(\phi_{r, h}^{m}\left(x_{h}^{m}, a_{h}^{m}\right)^{\top}\left(\Lambda_{r, h}^{m}\right)^{-1} \phi_{r, h}^{m}\left(x_{h}^{m}, a_{h}^{m}\right)\right)^{1 / 2}\right)
\end{align}

Since we take $\beta=C_{1} \sqrt{d H^{2} \log (d W / p)}$ with $C_{1}>1$, we have $H / \beta \leq 1$. It remains to apply Lemma \ref{lemma: elliptical Potential Lemma}. First, for every $h \in[H]$ it holds that
\begin{align*}
&\sum_{m=1}^{M} \phi_{r, h}^{m}\left(x_{h}^{m}, a_{h}^{m}\right)^{\top}\left(\Lambda_{r, h}^{m}\right)^{-1} \phi_{r, h}^{m}\left(x_{h}^{m}, a_{h}^{m}\right) \\
&\leq 2 \log \left(\frac{\operatorname{det}\left(\Lambda_{r, h}^{M+1}\right)}{\operatorname{det}\left(\Lambda_{r, h}^{1}\right)}\right).
\end{align*}

Due to $\left\|\phi_{r, h}^{m}\right\| \leq \sqrt{d} H$ in Assumption \ref{ass: linear fun approx} and $\Lambda_{r, h}^{1}=\lambda I$ in Algorithm \ref{alg:algoirthm 2}, it is clear that for every $h \in[H]$,
$$
\Lambda_{r, h}^{M+1}=\sum_{m=1}^{M} \phi_{r, h}^{m}\left(x_{h}^{m}, a_{h}^{m}\right) \phi_{r, h}^{m}\left(x_{h}^{m}, a_{h}^{m}\right)^{\top}+\lambda I \preceq\left(d H^{2} M +\lambda\right) I.
$$

Thus,
\begin{align*}
\log \left(\frac{\operatorname{det}\left(\Lambda_{r, h}^{K+1}\right)}{\operatorname{det}\left(\Lambda_{r, h}^{1}\right)}\right)  &\leq \log \left(\frac{\operatorname{det}\left(\left(d H^{2} M +\lambda\right) I\right)}{\operatorname{det}(\lambda I)}\right) \\
&\leq d \log \left(\frac{d H^{2} K+\lambda}{\lambda}\right).
\end{align*}

Therefore,
\begin{align} \label{eq: 32}
\nonumber &\sum_{m=1}^{M} \phi_{r, h}^{m}\left(x_{h}^{m}, a_{h}^{m}\right)^{\top}\left(\Lambda_{r, h}^{m}\right)^{-1} \phi_{r, h}^{m}\left(x_{h}^{m}, a_{h}^{m}\right) \\
&\leq 2 d \log \left(\frac{d H^{2} K+\lambda}{\lambda}\right).
\end{align}

Similarly, one can show that
\begin{align} \label{eq: 33}
\sum_{m=1}^{M} \varphi\left(x_{h}^{m}, a_{h}^{m}\right)^{\top}\left(\Lambda_{h}^{m}\right)^{-1} \varphi\left(x_{h}^{m}, a_{h}^{m}\right) \leq 2 d \log \left(\frac{d K+\lambda}{\lambda}\right) .
\end{align}

Applying the Cauchy-Schwartz inequality and the inequalities \eqref{eq: 32} and \eqref{eq: 33} to \eqref{eq: 31} leads to
$$
\begin{aligned}
&\sum_{\mathcal{E}=1}^{\ceil{\frac{M}{W}}}\sum_{m=(\mathcal{E}-1)W}^{\mathcal{E}W} \sum_{h=1}^{H} \min \left(H,\left(\Gamma_{h}^{m}+\Gamma_{r, h}^{m}\right)\left(x_{h}^{m}, a_{h}^{m}\right)\right) \\
&\leq \beta \sum_{\mathcal{E}=1}^{\ceil{\frac{M}{W}}} \sum_{h=1}^{H} \min \left(W,\right.\\
&\left.\hspace{1cm}\sum_{m=(\mathcal{E}-1)W}^{\mathcal{E}W}\left(\varphi\left(x_{h}^{m}, a_{h}^{m}\right)^{\top}\left(\Lambda_{h}^{m}\right)^{-1} \varphi\left(x_{h}^{m}, a_{h}^{m}\right)\right)^{1 / 2}\right.\\
&\left.+\left(\phi_{r, h}^{m}\left(x_{h}^{m}, a_{h}^{m}\right)^{\top}\left(\Lambda_{r, h}^{m}\right)^{-1} \phi_{r, h}^{m}\left(x_{h}^{m}, a_{h}^{m}\right)\right)^{1 / 2}\right) \\
& \leq \beta \sum_{\mathcal{E}=1}^{\ceil{\frac{M}{W}}} \sum_{h=1}^{H}\\ 
& \left(\left(W\sum_{m=(\mathcal{E}-1)W}^{\mathcal{E}W} \varphi\left(x_{h}^{m}, a_{h}^{m}\right)^{\top}\left(\Lambda_{h}^{m}\right)^{-1} \varphi\left(x_{h}^{m}, a_{h}^{m}\right)\right)^{1 / 2}+ \right. \\
&\left.\left(W \sum_{m=(\mathcal{E}-1)W}^{\mathcal{E}W} \phi_{r, h}^{m}\left(x_{h}^{m}, a_{h}^{m}\right)^{\top}\left(\Lambda_{r, h}^{m}\right)^{-1} \phi_{r, h}^{m}\left(x_{h}^{m}, a_{h}^{m}\right)\right)^{1 / 2}\right) \\
&\leq \beta MW^{-\frac{1}{2}} H \left(\left(2 d \log \left(\frac{d W+\lambda}{\lambda}\right)\right)^{1 / 2} \right.\\
&\left.+\left(2 d \log \left(\frac{d H^{2} W+\lambda}{\lambda}\right)\right)^{1/2}\right).
\end{aligned}
$$
Therefore, by setting $\lambda=1$, we have
\begin{align*}
&\sum_{m=1}^{M} \sum_{h=1}^{H}\left(\mathbb{E}_{\pi^{\star,m}, \mathbb{P}^{m}}\left[\iota_{r, h}^{m}\left(x_{h}, a_{h}\right) \mid x_{1}\right]-\iota_{r, h}^{m}\left(x_{h}^{m}, a_{h}^{m}\right)\right) \\
\leq& C_2 d H^2 M W^{-\frac{1}{2}}  \sqrt{\log \left(dH^2W+1 \right) \log \left(\frac{dW}{p} \right)} \\
& + 2B_{\mathbb{P}} H^3 d_1 W\sqrt{d_1 W} + 2B_{r}HW\sqrt{d_2 W}
\end{align*}
where $C_2$ is some constant. In addition, by Lemma \ref{lemma: bounds on model prediction error with local knowledge}, for every $(m, h) \in[M] \times[H]$ and $(x, a) \in \mathcal{S} \times \mathcal{A}$, the following inequality holds with probability at least $1-p/2$:
\begin{align*}
&-2\left(\Gamma_{h}^{m}+\Gamma_{r, h}^{m}\right)(x, a)-2B_{\mathbb{P},\mathcal{E}} H^2 d_1\sqrt{d_1 W} -2B_{r,\mathcal{E}}\sqrt{d_2 W} \\
& \leq  \iota_{r, h}^{m}(x, a) \leq 0.
\end{align*}

Thus, it holds that 
\begin{align*}
&\sum_{m=1}^{M} \sum_{h=1}^{H}\mathbb{E}_{\pi^{\star,m}, \mathbb{P}^{m}}\left[\mu^m\iota_{g, h}^{m}\left(x_{h}, a_{h}\right)\right]
 \leq  0.
\end{align*}
Finally, by combining the above two inequalities, we obtain the desired result.
\end{proof}

\begin{lemma}[Martingale bound for dynamic regret]\label{lemma: Martingale Bound}
Fix $p \in(0,1)$. In Algorithm \ref{alg:algoirthm 1}, it holds with probability at least $1-p/2$ that
\begin{align}\label{eq: 34}
\left|S_{r, H, 2}^{M}\right| \leq 4 \sqrt{H^{2} T \log \left(\frac{4}{p}\right)}
\end{align}
where $T=H M$.
\end{lemma}

\begin{proof}
In the expansion of the term (R.\rom{3}) in \eqref{eq: R.3} and Lemma \ref{lemma: expansion of V m -V pi m}, we introduce the following martingale:
$$
S_{r, H, 2}^{M}=\sum_{m=1}^{M} \sum_{h=1}^{H}\left(D_{r, h, 1}^{m}+D_{r, h, 2}^{m}\right)
$$
where
$$
\begin{aligned}
D_{r, h, 1}^{m}=&\left(\mathcal{I}_{h}^{m}\left(Q_{r, h}^{m}-Q_{r, h}^{\pi^{m}, m}\right)\right)\left(x_{h}^{m}\right)\\
&-\left(Q_{r, h}^{m}-Q_{r, h}^{\pi^{m}, m}\right)\left(x_{h}^{m}, a_{h}^{m}\right), \\
D_{r, h, 2}^{m}=&\left(\mathbb{P}^m_{h} V_{r, h+1}^{m}-\mathbb{P}_{h}^m V_{r, h+1}^{\pi^{m}, m}\right)\left(x_{h}^{m}, a_{h}^{m}\right)\\
&-\left(V_{r, h+1}^{m}-V_{r, h+1}^{\pi^{m}, m}\right)\left(x_{h+1}^{m}\right)
\end{aligned}
$$
and $\left(\mathcal{I}_{h}^{m} f\right)(x):=\left\langle f(x, \cdot), \pi_{h}^{m}(\cdot \mid x)\right\rangle$.
Due to the truncation in line 10 of Algorithm \ref{alg:algoirthm 2}, we know that $Q_{r, h}^{m}, Q_{r, h}^{\pi^{m},m}, V_{r, h+1}^{m}, V_{r, h+1}^{\pi^{m},m} \in[0, H]$. This shows that $\left|D_{r, h, 1}^{m}\right|\leq 2 H,\left|D_{r, h, 2}^{m}\right| \leq 2 H$ for all $(m, h) \in[M] \times[H]$. The Azuma-Hoeffding inequality yields that,
$$
P\left(\left|S_{r, H, 2}^{M}\right| \geq s\right) \leq 2 \exp \left(\frac{-s^{2}}{16 H^{2} T}\right) .
$$
For $p \in(0,1)$, if we set $s=4 H \sqrt{T \log (4 / p)}$, then the inequality \eqref{eq: 34} holds with probability at least $1-p / 2$.
\end{proof}

\subsubsection{Proof of dynamic regret in Theorem \ref{thm: Linear Kernal MDP under local budget}}
By combining Lemmas \ref{lemma: Bound for equation  R.2-1} and \ref{lemma: Bound for equation  R.2-2}, we can conclude that
\begin{align*}
&\operatorname{DR}(M)\\
\leq& \frac{1}{\alpha} H M L^{-1} \log{|\mathcal{A}|}+ \alpha H^{2} \sum_{m=1}^M(1+|\mu^{m}|^{2}) \\
&+H^2L B_{\star} + 2 H^2L \left(\sqrt{d_1} B_\mathbb{P}+  B_\star \right)-\sum_{m=1}^M \mu^m \sum_{h=1}^H \\
& \EE_{\pi^{\star,m}, \mathbb{P}^m} \left[ \inner{Q_{g,h}^m(x_h,\cdot)}{\pi_h^{\star,m}(\cdot\mid x_h) - \pi_h^m(\cdot\mid x_h)}\right] \\
& +\sum_{m=1}^M\sum_{h=1}^H \EE_{{\pi}^{\star,m}, \mathbb{P}^m} \left[ \iota_{r, h}^{m}(x_h,a_h) \right]\\
&-\sum_{m=1}^{M} \sum_{h=1}^{H} \iota_{r, h}^{m}\left(x_{h}^{m}, a_{h}^{m}\right)+S_{r, H, 2}^{M}.
\end{align*}
Then, by Lemma \ref{lemma: expansion of V star m -V m}, the above inequality further implies that
\begin{align} \label{eq: dynamic regret for constraint under slater condition}
\nonumber &\operatorname{DR}(M)\\
\nonumber \leq& \frac{1}{\alpha} H M L^{-1} \log{|\mathcal{A}|}+  \alpha H^{2} \sum_{m=1}^M(1+|\mu^{m}|^{2}) +H^2L B_{\star} \\
\nonumber &+ 2 H^2L \left(\sqrt{d_1} B_\mathbb{P}+  B_\star \right)\\
\nonumber & -\sum_{m=1}^{M} \mu^{m}\left(V_{g, 1}^{\pi^{\star,m},m}\left(x_{1}\right)-V_{g, 1}^{m}\left(x_{1}\right)\right)\\
\nonumber & +\sum_{m=1}^M\sum_{h=1}^H \EE_{{\pi}^{\star,m}, \mathbb{P}^m} \left[ \iota_{r, h}^{m}(x_h,a_h) +\mu^m \iota_{g, h}^{m}(x_h,a_h)\right] \\
&-\sum_{m=1}^{M} \sum_{h=1}^{H} \iota_{r, h}^{m}\left(x_{h}^{m}, a_{h}^{m}\right)+S_{r, H, 2}^{M}.
\end{align}
Due to the dual update in Lemma \ref{lemma: dual step}, we obtain
\begin{align}
 &\operatorname{DR}(M)\label{eq: for the regret under slater condition}\\
\nonumber\leq& \frac{1}{\alpha} H M L^{-1} \log{|\mathcal{A}|}+  \alpha H^{2} M +H^2L B_{\star}\\
\nonumber &+ 2  H^2L \left(\sqrt{d_1} B_\mathbb{P}+  B_\star \right) \\
\nonumber &+{\eta H^{2}(M+1)} +\sum_{m=1}^{M+1} ( \alpha H^{2}+\eta \xi^2-\xi) ||\mu^{m}|^{2} \\
\nonumber & +\sum_{m=1}^M\sum_{h=1}^H \EE_{{\pi}^{\star,m}, \mathbb{P}^m} \left[ \iota_{r, h}^{m}(x_h,a_h) +\mu^m \iota_{g, h}^{m}(x_h,a_h)\right] \\
\nonumber& -\sum_{m=1}^{M} \sum_{h=1}^{H} \iota_{r, h}^{m}\left(x_{h}^{m}, a_{h}^{m}\right)+S_{r, H, 2}^{M}.
\end{align}
Then, by controlling the model prediction error in Lemma \ref{lemma: Model prediction difference bound with local knowledge} and the martingale bound in Lemma \ref{lemma: Martingale Bound}, we have
\begin{align*}
&\operatorname{DR}(M)\\
\leq& \frac{1}{\alpha} H M L^{-1} \log{|\mathcal{A}|}+  \alpha H^{2} M + H^2L \left(2\sqrt{d_1} B_\mathbb{P}+  3B_\star \right)\\
&+{\eta H^{2}(M+1)}\\
&+\sum_{m=1}^{M+1} ( \alpha H^{2}+\eta \xi^2-\xi) ||\mu^{m}|^{2}+4 \sqrt{H^{2} T \log \left(\frac{4}{p}\right)}    \\
&+ C_2 d H^2 M W^{-\frac{1}{2}}  \sqrt{\log \left(dH^2W+1 \right) \log \left(\frac{dW}{p} \right)} \\
&+ B_{\mathbb{P}} H^3 d_1 W\sqrt{d_1 W} +B_{r}HW\sqrt{d_2 W}
\end{align*}
with probability at least $1-p$. Finally, by setting
$\alpha=  H^{-1} M^{-\frac{1}{2}} (\sqrt{d}B_\Delta+B_\star)^{\frac{1}{3}} $, $L= {M^{\frac{3}{4}}} (\sqrt{d}B_\Delta+B_\star)^{-\frac{2}{3}} $, $\eta=M^{-\frac{1}{2}}$, $\xi=2H (\sqrt{d}B_\Delta+B_\star)^{\frac{1}{3}} M^{-\frac{1}{2}}$,
$W=d^{-\frac{1}{4}} H^{-1} {M}^{\frac{1}{2}} B_\Delta^{-\frac{1}{2}}$,
it holds that 
\begin{align*}
    \text{D-Regret(M)} \leq & \widetilde{\mathcal{O}}\left( d^{\frac{9}{8}} H^{\frac{5}{2}} M^{\frac{3}{4}} (\sqrt{d}B_\Delta +B_\ast)^{\frac{1}{3}}\right)
\end{align*}
with probability at least $1-p$. This completes the proof.

%% file: files/appe_constraint.tex
\subsection{Proof of constraint violation in Theorem \ref{thm: Linear Kernal MDP under local budget}}\label{sec:appe-constraint}
Similarly, to analyze the constraint violation in \eqref{eq: d regret}, we introduce a useful decomposition weighted by the dual variable $\mu^m$.
\begin{lemma}[Constraint violation decomposition] \label{lemma: Constraint violation decomposition}
The constraint violation in \eqref{eq: d regret} is bounded by
\begin{align}
\nonumber & \sum_{m=1}^M \mu \left(b_m- V_{g,1}^{\pi^m,m} \right) 
= {\sum_{m=1}^M \mu \left(b_m-V_{g,1}^{m}(x_1) \right)}\\
\nonumber &-\sum_{m=1}^{M}\mu \sum_{h=1}^{H} \iota_{g, h}^{m}\left(x_{h}^{m}, a_{h}^{m}\right)+ \mu  S_{g, H, 2}^{M}.
\end{align}
where $\left\{S_{g, h, k}^{m, ,\mu}\right\}_{(m, h, k) \in[M] \times[H] \times[2]}$ is a martingale defined in Lemma \ref{lemma: expansion of V m -V pi m}.
\end{lemma}
\begin{proof}
We have
\begin{align*}
\nonumber & \sum_{m=1}^M \mu \left(b_m- V_{g,1}^{\pi^m,m} \right) \\
=&{\sum_{m=1}^M \mu \left(b_m-V_{g,1}^{m}(x_1) \right)} +{\sum_{m=1}^M\mu \left(V_{g,1}^{m}(x_1) -V_{g,1}^{\pi^m,m}(x_1)\right)}\\
=&{\sum_{m=1}^M \mu \left(b_m-V_{g,1}^{m}(x_1) \right)} -\sum_{m=1}^{M}\mu \sum_{h=1}^{H} \iota_{g, h}^{m}\left(x_{h}^{m}, a_{h}^{m}\right)+ \mu S_{g, H, 2}^{M}
\end{align*}
where the last inequality follows from Lemma \ref{lemma: expansion of V m -V pi m}.  
\end{proof}

\begin{lemma}[Primal step for constraint violation]\label{lemma: Primal step for constraint violation}
Let Assumption \ref{ass: linear fun approx} hold. For the primal update rule in line 10 of Algorithm \ref{alg:algoirthm 1}, we have
\begin{align*}
&\sum_{m=1}^{M} \mu^{m}\left(b_m-V_{g, 1}^{m}\left(x_{1}\right) \right)\\
\leq& HM + \frac{1}{\alpha} H M L^{-1} \log{|\mathcal{A}|}+  \alpha H^{2}M+H^2L B_{\star} \\
&+ 2  H^2L \left(\sqrt{d_1} B_\mathbb{P}+  B_\star \right)\\
\nonumber &+{\eta H^{2}(M+1)} +\sum_{m=1}^{M+1} ( \alpha H^{2}+\eta \xi^2-\xi) ||\mu^{m}|^{2} \\
\nonumber & +\sum_{m=1}^M\sum_{h=1}^H \EE_{{\pi}^{\ast,m}, \mathbb{P}^m} \left[ \iota_{r, h}^{m}(x_h,a_h) +\mu^m \iota_{g, h}^{m}(x_h,a_h)\right] \\
&-\sum_{m=1}^{M} \sum_{h=1}^{H} \iota_{r, h}^{m}\left(x_{h}^{m}, a_{h}^{m}\right)+S_{r, H, 2}^{M}.
\end{align*}
\end{lemma}
\begin{proof}
This result follows from the feasibility of the optimal policy $\pi^{*,m}$, equation \eqref{eq: for the regret under slater condition}, and $V_{r, 1}^{\pi^{*,m},m}\left(x_{1}\right), V_{r, 1}^{m}\left(x_{1}\right) \in[0,H]$.
\end{proof}

\begin{lemma}[Dual step for constraint violation] \label{lemma: Dual step for constraint violation}
Let Assumption \ref{ass: local budget} hold. In Algorithm \ref{alg:algoirthm 1}, if we set $\xi\eta\leq \frac{1}{2}$ and $\chi=\infty$, then we have 
\begin{align*}
 \sum_{m=1}^M (\mu-\mu^m) \left(b_m-V_{g,1}^{m}  \right)-(\frac{\xi M}{2}+\frac{1}{2\eta})|\mu|^2    \leq \eta  H^2 M
\end{align*}
for every $\mu\geq 0$.
\end{lemma}
\begin{proof}
Since $\widetilde{\mathcal{L}}_\xi^m(\pi^m,\cdot)$ is convex in $\mu$ for every $\mu\geq 0$, it holds that
\begin{align*}
\widetilde{\mathcal{L}}_\xi^m(\pi^m,\mu)\geq \widetilde{\mathcal{L}}_\xi^m(\pi^m,\mu^m)+ \nabla_\mu \widetilde{\mathcal{L}}_\xi^m(\pi^m,\mu^m)\left(\mu-  \mu^m \right)^\top,
\end{align*}
which is equivalent to 
\begin{align} \label{eq: dual convex inequality}
\nonumber &(\mu-\mu^m) \left(b_m-V_{g,1}^{m}  \right)-\frac{\xi}{2}|\mu|^2\\
\leq & -\frac{\xi}{2}|\mu^m|^2 + \nabla_\mu \widetilde{\mathcal{L}}_\xi^m(\pi^m,\mu^m)\left( \mu^m-\mu \right).
\end{align}
Based on the update of the dual variable in Algorithm \ref{alg:algoirthm 2}, we have 
$$
\begin{aligned}
(\mu^{m+1}-\mu)^{2} =&\left|\mu^{m}-\eta \nabla_\mu \widetilde{\mathcal{L}}_\xi^m(\pi^m,\mu^m) -\mu\right|^{2} \\
 \leq&\left(\mu^{m}-\mu\right)^{2}-2 \eta\nabla_\mu \widetilde{\mathcal{L}}_\xi^m(\pi^m,\mu^m) \left(\mu^{m}-\mu\right)\\
 & +\eta^{2} (\nabla_\mu \widetilde{\mathcal{L}}_\xi^m(\pi^m,\mu^m) )^{2}\\
\end{aligned}
$$
which is equivalent to 
\begin{align*}
& \nabla_\mu \widetilde{\mathcal{L}}_\xi^m(\pi^m,\mu^m) \left(\mu^{m}-\mu\right) \\
\leq&  \frac{1}{2\eta} \left( \left(\mu^{m}-\mu\right)^{2}-\left(\mu^{m+1}-\mu\right)^{2}  \right)+\frac{\eta}{2}(\nabla_\mu \widetilde{\mathcal{L}}_\xi^m(\pi^m,\mu^m) )^{2}.
\end{align*}

Substituting the above inequality into \eqref{eq: dual convex inequality} yields that
\begin{align*}
&(\mu-\mu^m) \left(b_m-V_{g,1}^{m}   \right)-\frac{\xi}{2}|\mu|^2 \\
\leq & -\frac{\xi}{2}|\mu^m|^2 +\frac{\eta}{2} (\nabla_\mu \widetilde{\mathcal{L}}_\xi^m(\pi^m,\mu^m) )^{2}\\
&+\frac{1}{2\eta} \left( \left(\mu^{m}-\mu\right)^{2}-\left(\mu^{m+1}-\mu\right)^{2}  \right)\\
\leq & -\frac{\xi}{2}|\mu^m|^2 +\frac{\eta}{2} (H +\xi |\mu^m|)^{2}\\
&+\frac{1}{2\eta} \left( \left(\mu^{m}-\mu\right)^{2}-\left(\mu^{m+1}-\mu\right)^{2}  \right)\\
\leq & -\frac{\xi}{2}|\mu^m|^2 +\eta (H ^{2}   +\xi^{2} |\mu^m|^{2})\\
&+\frac{1}{2\eta} \left( \left(\mu^{m}-\mu\right)^{2}-\left(\mu^{m+1}-\mu\right)^{2}  \right)\\
= & \eta H^2  +(\xi^2\eta - \frac{\xi}{2})|\mu^m|^{2}+\frac{1}{2\eta} \left( \left(\mu^{m}-\mu\right)^{2}-\left(\mu^{m+1}-\mu\right)^{2}  \right)\\
\leq & \eta H^2  +\frac{1}{2\eta} \left( \left(\mu^{m}-\mu\right)^{2}-\left(\mu^{m+1}-\mu\right)^{2}  \right).
\end{align*}
where the second inequality follows from
\begin{align*}
\left|\nabla_\mu \widetilde{\mathcal{L}}_\xi^m(\pi^m,\mu^m) \right|=& \left| \left(b_m-V_{g,1}^{m}  \right) +\xi\mu^m \right|\leq  (H +\xi |\mu^m|),
\end{align*}
and the last inequality follows from $\xi\eta\leq \frac{1}{2}$.
Taking the summation from $m=1$ to $M$, we obtain the desired results.
\end{proof}

\begin{lemma}(\textbf{Model prediction error bound for constraint violation}) \label{lemma: Model prediction difference bound for constraint violation}
Let  \ref{ass: linear fun approx} and Assumption \ref{ass: local budget} hold. Fix $p \in(0,1)$ and let $\mathcal{E}$ be the epoch that the episode $m$ belongs to.
If we set $\lambda=1$, $LV=B_{\mathbb{P},\mathcal{E}} H^2 d_1\sqrt{d_1 W} +B_{g,\mathcal{E}}\sqrt{d_2 W},$ and
\begin{align*}
&\Gamma_{h}^{m}(\cdot, \cdot)=\beta\left(\varphi(\cdot, \cdot)^{\top}\left(\Lambda_{h}^{m}\right)^{-1} \varphi(\cdot, \cdot)\right)^{1 / 2}, \\
&\Gamma_{r,h}^m =\beta \left( (\phi_{r,h}^m)^\top (\Lambda_{r, h}^m)^{-1} \phi_{r,h}^m \right)^{1/2}, \\
&\Gamma_{g,h}^m =\beta \left( (\phi_{g,h}^m)^\top (\Lambda_{g,h}^m)^{-1} \phi_{g,h}^m \right)^{1/2}
\end{align*}
with $\beta=$ $C_{1} \sqrt{d H^{2} \log (d W / p)}$ in Algorithm \ref{alg:algoirthm 2}, then with probability at least $1-p / 2$ it holds that
\begin{align*}
& \sum_{m=1}^{M} \sum_{h=1}^{H}\left(- \mu \iota_{g, h}^{m}\left(x_{h}^{m}, a_{h}^{m}\right)\right)\\
\leq  & C_3 \mu d H^2 M W^{-\frac{1}{2}}  \sqrt{\log \left(dH^2W+1 \right) \log \left(\frac{dW}{p} \right)} \\
&+ 2\mu B_{\mathbb{P}} H^3 d_1 W\sqrt{d_1 W} +  2\mu B_{g}HW\sqrt{d_2 W}
\end{align*}
for every $\mu>0$, where $C_{1}, C_3$ are some absolute constants.
\end{lemma}

\begin{proof}
By Lemma \ref{lemma: bounds on model prediction error tabular}, for every $(m, h) \in[M] \times[H]$ and $(x, a) \in \mathcal{S} \times \mathcal{A}$, the inequality
\begin{align*}
 & -\iota_{g, h}^{m}(x, a) \leq 2\left(\Gamma_{h}^{m}+\Gamma_{g, h}^{m}\right)(x, a)\\
  &+2B_{\mathbb{P},\mathcal{E}} H^2 d_1\sqrt{d_1 W} +2B_{g,\mathcal{E}}\sqrt{d_2 W}
\end{align*}
holds with probability at least $1-p / 2$. The rest of the proof is similar to Lemma \ref{lemma: Model prediction difference bound with local knowledge} and is thus omitted.

\end{proof}

\begin{lemma}[Martingale bound for constraint violation]\label{lemma: Martingale Bound for constraint}
Fix $p \in(0,1)$. In Algorithm \ref{alg:algoirthm 1}, for every $\mu\geq 0$, it holds with probability at least $1-p/2$ that
\begin{align}\label{eq: 34 for constraint}
\left|\mu S_{g, H, 2}^{M}\right| \leq 4 H \mu\sqrt{ T \log \left(\frac{4}{p}\right)}
\end{align}
where $T=H M$.
\end{lemma}

\begin{proof}
In Lemmas \ref{lemma: Constraint violation decomposition} and \ref{lemma: expansion of V m -V pi m}, we introduce the following martingale:
$$
\mu S_{g, H, 2}^{M}= \mu \sum_{m=1}^{M} \sum_{h=1}^{H}\left(D_{r, h, 1}^{m}+D_{r, h, 2}^{m}\right)
$$
where
$$
\begin{aligned}
D_{r, h, 1}^{m}=&\left(\mathcal{I}_{h}^{m}\left(Q_{r, h}^{m}-Q_{r, h}^{\pi^{m}, m}\right)\right)\left(x_{h}^{m}\right)\\
&-\left(Q_{r, h}^{m}-Q_{r, h}^{\pi^{m}, m}\right)\left(x_{h}^{m}, a_{h}^{m}\right), \\
D_{r, h, 2}^{m}=&\left(\mathbb{P}^m_{h} V_{r, h+1}^{m}-\mathbb{P}_{h}^m V_{r, h+1}^{\pi^{m}, m}\right)\left(x_{h}^{m}, a_{h}^{m}\right)\\
&-\left(V_{r, h+1}^{m}-V_{r, h+1}^{\pi^{m}, m}\right)\left(x_{h+1}^{m}\right)
\end{aligned}
$$
and $\left(\mathcal{I}_{h}^{m} f\right)(x):=\left\langle f(x, \cdot), \pi_{h}^{m}(\cdot \mid x)\right\rangle$.
Due to the truncation in line 10 of Algorithm \ref{alg:algoirthm 2}, we know that $Q_{r, h}^{m}, Q_{r, h}^{\pi^{m},m}, V_{r, h+1}^{m}, V_{r, h+1}^{\pi^{m},m} \in[0, H]$. In addition, we have  $\left|\mu D_{r, h, 1}^{m}\right|\leq 2 H \mu,\left|\mu D_{r, h, 2}^{m}\right| \leq 2 H \mu$ for all $(m, h) \in[M] \times[H]$. The Azuma-Hoeffding inequality yields that
$$
P\left(\left| \mu S_{g, H, 2}^{M}\right| \geq s\right) \leq 2 \exp \left(\frac{-s^{2}}{16 H^{2} \mu^2T}\right) .
$$
For $p \in(0,1)$, if we set $s=4 H  \mu \sqrt{T \log (4 / p)}$, then the inequality \eqref{eq: 34 for constraint} holds with probability at least $1-p / 2$.
\end{proof}

\subsubsection{Proof of constraint violation in Theorem \ref{thm: Linear Kernal MDP under local budget}}
We are now ready to prove the desired constraint violation bound. 
By combining Lemmas \ref{lemma: Constraint violation decomposition}, \ref{lemma: Primal step for constraint violation} and \ref{lemma: Dual step for constraint violation}, one can conclude that
\begin{align} 
& \mu \sum_{m=1}^M  \left(b_m-V_{g, 1}^{\pi^m,m}\left(x_{1}\right) \right) -(\frac{\xi M}{2}+\frac{1}{2\eta})|\mu|^2  \label{eq: constraint vilation in linear MDP cited for tabular under local budget}\\
\nonumber\leq& HM + \frac{1}{\alpha} H M L^{-1} \log{|\mathcal{A}|}+ \alpha H^{2}M\\
\nonumber&+H^2L \left(2\sqrt{d_1} B_\mathbb{P}+  3B_\star \right)+{\eta H^{2}(M+1)}+   \eta H^2 M \\
\nonumber & +\sum_{m=1}^M\sum_{h=1}^H \EE_{{\pi}^{\ast,m}, \mathbb{P}^m} \left[ \iota_{r, h}^{m}(x_h,a_h) +\mu^m \iota_{g, h}^{m}(x_h,a_h)\right] \\
\nonumber&-\sum_{m=1}^{M} \sum_{h=1}^{H} \iota_{r, h}^{m}\left(x_{h}^{m}, a_{h}^{m}\right)\\
\nonumber&-\sum_{m=1}^{M}\mu \sum_{h=1}^{H} \iota_{g, h}^{m}\left(x_{h}^{m}, a_{h}^{m}\right)
+S_{r, H, 2}^{M}+ \mu S_{g, H, 2}^{M} .
\end{align}
Then, by controlling the model prediction errors in Lemmas \ref{lemma: Model prediction difference bound with local knowledge} and \ref{lemma: Model prediction difference bound for constraint violation} as well as the martingale bounds in Lemmas \ref{lemma: Martingale Bound} and \ref{lemma: Martingale Bound for constraint}, we have
\begin{align*}
& \mu \sum_{m=1}^M  \left(b_m-V_{g, 1}^{\pi^m,m}\left(x_{1}\right) \right) -(\frac{\xi M}{2}+\frac{1}{2\eta})|\mu|^2 \\
\leq& HM + \frac{1}{\alpha} H M L^{-1} \log{|\mathcal{A}|}+ \alpha H^{2}M\\
&+H^2L \left(2\sqrt{d_1} B_\mathbb{P}+  3B_\star \right)+{\eta H^{2}(M+1)}+   \eta H^2 M\\
&+ C_2 d H^2 M W^{-\frac{1}{2}}  \sqrt{\log \left(dH^2W+1 \right) \log \left(\frac{dW}{p} \right)}\\
&+ 2B_{\mathbb{P}} H^3 d_1 W\sqrt{d_1 W} +2B_{r}HW\sqrt{d_2 W}\\
& +C_3 \mu d H^2 M W^{-\frac{1}{2}}  \sqrt{\log \left(dH^2W+1 \right) \log \left(\frac{dW}{p} \right)} \\
& + 2\mu B_{\mathbb{P}} H^3 d_1 W\sqrt{d_1 W} +  2\mu B_{g}HW\sqrt{d_2 W}\\
& +4H \sqrt{ T \log \left(\frac{4}{p}\right)} +4 H \mu \sqrt{ T \log \left(\frac{4}{p}\right)}.
\end{align*}
Furthermore, by substituting the parameters $\alpha=  H^{-1} M^{-\frac{1}{2}} (\sqrt{d}B_\Delta+B_\star)^{\frac{1}{3}} $, $L= {M^{\frac{3}{4}}} (\sqrt{d}B_\Delta+B_\star)^{-\frac{2}{3}} $, $\eta=M^{-\frac{1}{2}}$, $\xi=2H (\sqrt{d}B_\Delta+B_\star)^{\frac{1}{3}} M^{-\frac{1}{2}}$,
$W=d^{-\frac{1}{4}} H^{-1} {M}^{\frac{1}{2}} B_\Delta^{-\frac{1}{2}}$, and rearranging all the terms related to $\mu$ to the left hand side of the inequality,
it holds that 
\begin{align*}
& \mu \sum_{m=1}^M \left[ \left(b_m-V_{g, 1}^{\pi^m,m}\left(x_{1}\right) \right) - \widetilde{O} \left(  d^{\frac{9}{8}} H^{\frac{5}{2}} M^{\frac{3}{4}} B_\Delta^{\frac{1}{4}} \right)\right]\\
&-H (\sqrt{d}B_\Delta+B_\star)^{\frac{1}{3}} M^{\frac{1}{2}}|\mu|^2 \\
\leq &  \widetilde{\mathcal{O}}\left( HM+ H^2 M^{\frac{3}{4}} (\sqrt{d}B_\Delta+B_\ast)^{\frac{1}{3}} + d^{\frac{9}{8}} H^{\frac{5}{2}} M^{\frac{3}{4}} B_\Delta^{\frac{1}{4}} \right).
\end{align*}

Note that the above inequality holds for every $\mu\geq 0$. By maximizing the both sides of above inequality over $\mu\geq 0$, i.e., by choosing 
$$\mu=\frac{\left[\sum_{m=1}^{M}\left(b_m-V_{g, 1}^{\pi^{m},m}\left(x_{1}\right) \right)- \widetilde{O} \left(  d^{\frac{9}{8}} H^{\frac{5}{2}} M^{\frac{3}{4}} B_\Delta^{\frac{1}{4}} \right)\right]_+}{2H (\sqrt{d}B_\Delta+B_\star)^{\frac{1}{3}} M^{\frac{1}{2}}},$$ 
we obtain
\begin{align*}
&\frac{\left( \left[\sum_{m=1}^{M}\left(b_m-V_{g, 1}^{\pi^{m},m}\left(x_{1}\right) \right)- \widetilde{O} \left(  d^{\frac{9}{8}} H^{\frac{5}{2}} M^{\frac{3}{4}} B_\Delta^{\frac{1}{4}} \right)\right]_+\right)^2}{4H (\sqrt{d}B_\Delta+B_\star)^{\frac{1}{3}} M^{\frac{1}{2}}} \\
\leq& \widetilde{\mathcal{O}}\left(HM+   H^3 M^{\frac{2}{3}} (\sqrt{d}B_\Delta +B_\ast)^{\frac{1}{3}} +  d^{\frac{9}{8}} H^{\frac{5}{2}} M^{\frac{3}{4}} B_\Delta^{\frac{1}{4}}\right).
\end{align*}

Finally, one can conclude that 
\begin{align*}
\left[\sum_{m=1}^{M}b_m- V_{g, 1}^{\pi^{m},m}\left(x_{1}\right)\right]_{+} 
\leq & \widetilde{\mathcal{O}}\left(d^{\frac{9}{8}} H^{\frac{5}{2}} M^{\frac{3}{4}} (\sqrt{d}B_\Delta +B_\ast)^{\frac{1}{3}}\right).
\end{align*}
This completes the proof.

%% file: files/appe_regret_uniform_slater.tex
\section{Proof for linear kernel CMDP case under Assumption \ref{ass: Feasibility}}\label{sec:appe-linear MDP under feasibility}

\subsection{Model prediction error}
\begin{lemma}(\textbf{Model prediction error bound for dynamic regret under uniform Slater condition}) \label{lemma: Model prediction difference bound under uniform Slater condition}
Let Assumption \ref{ass: linear fun approx} and \ref{ass: Feasibility} hold. Fix $p \in(0,1)$ and let $\mathcal{E}$ be the epoch that the episode $m$ belongs to.
If we set $\lambda=1$, $LV=0$ and
\begin{align*}
&\Gamma_{h}^{m}(\cdot, \cdot)=\beta\left(\varphi(\cdot, \cdot)^{\top}\left(\Lambda_{h}^{m}\right)^{-1} \varphi(\cdot, \cdot)\right)^{1 / 2}, \\
&\Gamma_{r,h}^m =\beta \left( (\phi_{r,h}^m)^\top (\Lambda_{r, h}^m)^{-1} \phi_{r,h}^m \right)^{1/2}, \\
&\Gamma_{g,h}^m =\beta \left( (\phi_{g,h}^m)^\top (\Lambda_{g, h}^m)^{-1} \phi_{g,h}^m \right)^{1/2},
\end{align*}
with $\beta=$ $C_{1} \sqrt{d H^{2} \log (d W / p)}$ in Algorithm \ref{alg:algoirthm 2}, then with probability at least $1-p/2$ it holds that
\begin{align*}
&\sum_{m=1}^{M} \sum_{h=1}^{H}\left(\mathbb{E}_{\pi^{\star,m}, \mathbb{P}^{m}}\left[\iota_{r, h}^{m}\left(x_{h}, a_{h}\right)+\mu^m\iota_{g, h}^{m}\left(x_{h}, a_{h}\right)\right]\right.\\
& \left.-\iota_{r, h}^{m}\left(x_{h}^{m}, a_{h}^{m}\right)\right)\\
\leq & C_2 d H^2 M W^{-\frac{1}{2}}  \sqrt{\log \left(dH^2W+1 \right) \log \left(\frac{dW}{p} \right)} \\
& + (2+\chi )B_{\mathbb{P}} H^3 d_1 W\sqrt{d_1 W} +(2B_{r}+\chi B_{g})HW\sqrt{d_2 W}
\end{align*}
where $C_{1}$ and $C_2$ are absolute constants and $\mu^m\leq \chi$.
\end{lemma}
\begin{proof}
By Lemma \ref{lemma: bounds on model prediction error}, for every $(m, h) \in[M] \times[H]$ and $(x, a) \in \mathcal{S} \times \mathcal{A}$, the following inequality holds with probability at least 1-p/2:
\begin{align*}
&-2\left(\Gamma_{h}^{m}+\Gamma_{\diamond, h}^{m}\right)(x, a)-B_{\mathbb{P},\mathcal{E}} H^2 d_1\sqrt{d_1 W} -B_{\diamond,\mathcal{E}}\sqrt{d_2 W} \\
&\leq \iota_{\diamond, h}^{m}(x, a) \leq B_{\mathbb{P},\mathcal{E}} H^2 d_1\sqrt{d_1 W} +B_{\diamond,\mathcal{E}}\sqrt{d_2 W} .
\end{align*}

for $\diamond=r$ or $g$.
The rest of the proof is similar to Lemma \ref{lemma: Model prediction difference bound with local knowledge} and is thus omitted.
\end{proof}
\subsection{Proof of dynamic regret in Theorem \ref{thm: Linear Kernal MDP under feasible}}
From equation \eqref{eq: for the regret under slater condition}, we have
\begin{align}
& \nonumber \operatorname{DR}(M) \\
\nonumber \leq& \frac{1}{\alpha} H M L^{-1} \log{|\mathcal{A}|}+ \alpha H^{2} \sum_{m=1}^M(1+|\mu^{m}|^{2}) \\
\nonumber & +  H^2L \left(2\sqrt{d_1} B_\mathbb{P}+  3B_\star \right)  \\
\nonumber &+{\eta H^{2}(M+1)} +\sum_{m=1}^{M+1} (\eta \xi^2-\xi) | \mu^{m-1}|^2 \\
\nonumber & +\sum_{m=1}^M\sum_{h=1}^H \EE_{{\pi}^{\star,m}, \mathbb{P}^m} \left[ \iota_{r, h}^{m}(x_h,a_h) +\mu^m \iota_{g, h}^{m}(x_h,a_h)\right] \\
\nonumber &-\sum_{m=1}^{M} \sum_{h=1}^{H} \iota_{r, h}^{m}\left(x_{h}^{m}, a_{h}^{m}\right)+S_{r, H, 2}^{M}\\
 \leq& \frac{1}{\alpha} H M L^{-1} \log{|\mathcal{A}|}+ \alpha H^{2} \sum_{m=1}^M(1+\chi ^{2}) \label{eq: dynamic regret in linear under feasibility for tabular} \\
\nonumber  &+  H^2L \left(2\sqrt{d_1} B_\mathbb{P}+  3B_\star \right) \\
\nonumber &+{\eta H^{2}(M+1)} +\sum_{m=1}^{M+1} (\eta \xi^2-\xi) \chi ^2 \\
\nonumber & +\sum_{m=1}^M\sum_{h=1}^H \EE_{{\pi}^{\star,m}, \mathbb{P}^m} \left[ \iota_{r, h}^{m}(x_h,a_h) +\mu^m \iota_{g, h}^{m}(x_h,a_h)\right]\\
\nonumber &-\sum_{m=1}^{M} \sum_{h=1}^{H} \iota_{r, h}^{m}\left(x_{h}^{m}, a_{h}^{m}\right)+S_{r, H, 2}^{M}
\end{align}
where the second inequality follows from the fact $\mu^m \leq \chi $ for all $m\in[M]$ under the uniform Slater condition. Then, by Lemma \ref{lemma: Model prediction difference bound under uniform Slater condition}, it holds that
\begin{align}
 & \operatorname{DR}(M)\\
\nonumber \leq& \frac{1}{\alpha} H M L^{-1} \log{|\mathcal{A}|}+ \alpha H^{2} \sum_{m=1}^M(1+\chi ^{2})  \\
\nonumber & +   H^2L \left(2\sqrt{d_1} B_\mathbb{P}+  3B_\star \right)  \\
\nonumber &+{\eta H^{2}(M+1)} +\sum_{m=1}^{M+1} (\eta \xi^2-\xi) \chi ^2 \\
\nonumber & + C_2 d H^2 M W^{-\frac{1}{2}}  \sqrt{\log \left(dH^2W+1 \right) \log \left(\frac{dW}{p} \right)} \\
\nonumber &+ (2+\chi )B_{\mathbb{P}} H^3 d_1 W\sqrt{d_1 W} \\
\nonumber&+(2B_{r}+\chi B_{g})HW\sqrt{d_2 W}.
\end{align}

Furthermore, by substituting the parameters $\alpha=\gamma H^{-\frac{3}{2}} M^{-\frac{1}{3}} (\sqrt{d}B_\Delta+B_\star)^{\frac{1}{3}} $, $L= {M^{\frac{2}{3}}} (\sqrt{d}B_\Delta+B_\star)^{-\frac{2}{3}} $, $\eta=M^{-\frac{1}{2}}$, $\xi=0$,
$W=d^{-\frac{1}{4}} H^{-1} {M}^{\frac{1}{2}} B_\Delta^{-\frac{1}{2}}$, we obtain
\begin{align}
\operatorname{DR}(M) \leq& \widetilde{\mathcal{O}} \left(\gamma^{-1} d^{\frac{9}{8}} H^{\frac{5}{2}}M^{\frac{3}{4}} (\sqrt{d}B_\Delta +B_\ast)^{\frac{1}{3}}\right).
\end{align}
This completes the proof.

\subsection{Proof of constraint violation in Theorem \ref{thm: Linear Kernal MDP under feasible}}

By the dual update in line 14 in Algorithm \ref{alg:algoirthm 1} and $\xi=0$ , for any $\mu \in[0, \chi]$ we have
$$
\begin{aligned}
&\left|\mu^{m+1}-\mu\right|^{2} \\
&=\left|\operatorname{Proj}_{[0, \chi]}\left(\mu^{m}+\eta\left(b_m-V_{g, 1}^{m}\left(x_{1}\right)\right)\right)-\operatorname{Proj}_{[0, \chi]}(\mu)\right|^{2} \\
& \leq\left|\mu^{m}+\eta\left(b_m-V_{g, 1}^{m}\left(x_{1}\right)\right)-\mu\right|^{2} \\
& \leq\left(\mu^{m}-\mu\right)^{2}+2 \eta\left(b_m -V_{g, 1}^{m}\left(x_{1}\right)\right)\left(\mu^{m}-\mu\right)+\eta^{2} H^{2}
\end{aligned}
$$
where we apply the non-expansiveness of projection in the first inequality and $\left|b_m-V_{g, 1}^{m}\left(x_{1}\right)\right| \leq H$ for the last inequality. By summing the above inequality from $m=1$ to $m=M$, we have
$$
\begin{aligned}
&0 \leq\left|\mu^{M+1}-\mu\right|^{2}=\left|\mu^{1}-\mu\right|^{2}\\
&+2 \eta \sum_{m=1}^{M}\left(b_m-V_{g, 1}^{m}\left(x_{1}\right)\right)\left(\mu^{m}-\mu \right)+\eta^{2} H^{2} M
\end{aligned}
$$
which implies that
\begin{align} \label{eq: change of mu}
\nonumber \sum_{m=1}^{M}\left(b_m-V_{g, 1}^{m}\left(x_{1}\right)\right)\left(\mu-\mu^{m}\right) &\leq \frac{1}{2 \eta}\left|\mu^{1}-\mu\right|^{2}+\frac{\eta}{2} H^{2} M \\
&\leq \frac{1}{2 \eta}\mu^2 +\frac{\eta}{2} H^{2} M.
\end{align}

In addition, from  equation \eqref{eq: dynamic regret for constraint under slater condition}, we obtain
\begin{align}
\nonumber & \sum_{m=1}^{M} \left(V_{r, 1}^{\pi^{\star,m},m}\left(x_{1}\right)-V_{r, 1}^{m}\left(x_{1}\right)\right) +\sum_{m=1}^{M} \mu^{m}\left(b_m-V_{g, 1}^{m}\left(x_{1}\right)\right)\\
\nonumber \leq& \frac{1}{\alpha} H M L^{-1} \log{|\mathcal{A}|} \\
& +  \alpha H^{2} \sum_{m=1}^M(1+|\mu^{m}|^{2})  + H^2L \left(2\sqrt{d_1} B_\mathbb{P}+  3B_\star \right) \label{eq: constraint violation in linear under feasibility for tabular}\\
\nonumber & +\sum_{m=1}^M\sum_{h=1}^H \EE_{{\pi}^{\star,m}, \mathbb{P}^m} \left[ \iota_{r, h}^{m}(x_h,a_h) +\mu^m \iota_{g, h}^{m}(x_h,a_h)\right] \\
\nonumber &-\sum_{m=1}^{M} \sum_{h=1}^{H} \iota_{r, h}^{m}\left(x_{h}^{m}, a_{h}^{m}\right)+S_{r, H, 2}^{M}\\
\nonumber \leq & \frac{1}{\alpha} H M L^{-1} \log{|\mathcal{A}|} \\
\nonumber  &+ \alpha H^{2} \sum_{m=1}^M(1+\chi ^{2}) +   H^2L \left(2\sqrt{d_1} B_\mathbb{P}+  3B_\star \right)  \\
\nonumber & + C_2 d H^2 M W^{-\frac{1}{2}}  \sqrt{\log \left(dH^2W+1 \right) \log \left(\frac{dW}{p} \right)} \\
\nonumber &+ (2+\chi )B_{\mathbb{P}} H^3 d_1 W\sqrt{d_1 W}\\ \nonumber &+(2B_{r}+\chi B_{g})HW\sqrt{d_2 W}\\
\nonumber \leq & \widetilde{\mathcal{O}} \left(\gamma^{-1}  d^{\frac{9}{8}} H^{\frac{5}{2}}M^{\frac{3}{4}} (\sqrt{d}B_\Delta +B_\ast)^{\frac{1}{3}}\right),
\end{align}
where the second inequality follow from Lemma  \ref{lemma: Model prediction difference bound under uniform Slater condition} and the last inequality follows  by substituting the parameters $\alpha=\gamma H^{-\frac{3}{2}} M^{-\frac{1}{3}} (\sqrt{d}B_\Delta+B_\star)^{\frac{1}{3}} $, $L= {M^{\frac{2}{3}}} (\sqrt{d}B_\Delta+B_\star)^{-\frac{2}{3}} $, $\eta=M^{-\frac{1}{2}}$, $\xi=0$, $W=d^{-\frac{1}{4}} H^{-1} {M}^{\frac{1}{2}} B_\Delta^{-\frac{1}{2}}$. 
Then, by combining the above inequality with \eqref{eq: change of mu} and setting $\mu=\chi$, it holds that
\begin{align}
\nonumber & \sum_{m=1}^{M} \left(V_{r, 1}^{\pi^{\star,m},m}\left(x_{1}\right)-V_{r, 1}^{m}\left(x_{1}\right)\right) +\sum_{m=1}^{M} \chi \left(b_m-V_{g, 1}^{m}\left(x_{1}\right)\right)\\
\nonumber \leq & \widetilde{\mathcal{O}} \left(\gamma^{-1} d^{\frac{9}{8}} H^{\frac{5}{2}}M^{\frac{3}{4}} (\sqrt{d}B_\Delta +B_\ast)^{\frac{1}{3}} + M^{\frac{1}{2}} \chi^2\right).
\end{align}
Finally, by Corollary \ref{corollary: Constraint Violation under Uniform Slater Condition}, we obtain
\begin{align*}
&\left[\sum_{m=1}^{M}b_m- V_{g, 1}^{\pi^{m},m}\left(x_{1}\right)\right]_{+} \\
\leq & \widetilde{\mathcal{O}}\left(\gamma^{-1}  d^{\frac{9}{8}} H^{\frac{5}{2}} M^{\frac{3}{4}} (\sqrt{d}B_\Delta +B_\ast)^{\frac{1}{3}}\right).
\end{align*}
This completes the proof.

%% file: files/appe_tabular_local_budget.tex
\section{Proof for tabular CMDP case under Assumption \ref{ass: local budget}}\label{sec:appe-tabular MDP under local budget}
The proof is similar to that of Theorem \ref{thm: Linear Kernal MDP under local budget}, and we will first prove the dynamic regret bound.
Since we only change the policy evaluation, all previous policy improvement results still hold.
\subsection{Model prediction error}
\begin{lemma}(\textbf{Model prediction error bound in tabular case under Assumption \ref{ass: local budget}}) \label{lemma: Model prediction difference bound tabular under local budget}
Let  Assumption  \ref{ass: local budget} hold. If we set $\lambda=1$, $LV=B_{\mathbb{P},\mathcal{E}} H  +B_{g,\mathcal{E}}$ and
\begin{align*}
\Gamma_{h}^m =\beta\left(n_{h}^m(x, a)+\lambda\right)^{-1 / 2},
\end{align*}
with $\beta=C_{4} H \sqrt{|\mathcal{S}| \log (|\mathcal{S}||\mathcal{A}| W / p)}$ in Algorithm \ref{alg:algoirthm 3}, then with probability at least $1-p / 2$ it holds that
\begin{align*}
&\sum_{m=1}^{M} \sum_{h=1}^{H}\left(\mathbb{E}_{\pi^{\star,m}, \mathbb{P}^{m}}\left[\iota_{r, h}^{m}\left(x_{h}, a_{h}\right)+\mu^m\iota_{g, h}^{m}\left(x_{h}, a_{h}\right)\right] \right.\\
& \left.-\iota_{r, h}^{m}\left(x_{h}^{m}, a_{h}^{m}\right)\right)\\
\leq & C_5 MH^2 |\mathcal{S}|\sqrt{ |\mathcal{A}| W^{-1}} \sqrt{\log (|\mathcal{S}||\mathcal{A}| W / p) \log \left(M+1\right)}\\
& +2B_{\mathbb{P}} H W + 2B_{r}W \sum_{m=1}^{M} \sum_{h=1}^{H}\left(- \mu\iota_{g, h}^{m}\left(x_{h}^{m}, a_{h}^{m}\right)\right)\\
\leq  & C_6\mu MH^2 |\mathcal{S}|\sqrt{ |\mathcal{A}| W^{-1}} \sqrt{\log (|\mathcal{S}||\mathcal{A}| W / p) \log \left(M+1\right)}\\
& + \mu B_{\mathbb{P}} H W + \mu B_g W
\end{align*}
where $C_{4}$, $C_5$ and $C_6$ are some absolute constants, and $\mu, \mu^m\geq0$.
\end{lemma}
\begin{proof}
By Lemma \ref{lemma: bounds on model prediction error tabular under local budget}, for every $(m, h) \in[M] \times[H]$ and $(x, a) \in \mathcal{S} \times \mathcal{A}$, the following inequality holds with probability at least 1-p/2:
\begin{align*}
&-4\Gamma_{h}^{m}(x, a)-B_{\mathbb{P},\mathcal{E}} H -B_{r,\mathcal{E}}\leq \iota_{r, h}^{m}(x, a) \leq B_{\mathbb{P},\mathcal{E}} H +B_{r,\mathcal{E}}\\    
&-4\Gamma_{h}^{m}(x, a)-2B_{\mathbb{P},\mathcal{E}} H -2B_{g,\mathcal{E}}\leq \iota_{g, h}^{m}(x, a) \leq 0.
\end{align*}
The rest of the proof is similar to Lemma \ref{lemma: Model prediction difference bound with local knowledge} and is thus omitted.
\end{proof}

\subsection{Proof of dynamic regret in Theorem \ref{thm: Tabular Case MDP with local budget}}
From equation \eqref{eq: for the regret under slater condition} and Lemma \ref{lemma: probability difference in expectation tabular}, we obtain
\begin{align}
 &\operatorname{DR}(M) \\
\nonumber \leq& \frac{1}{\alpha} H M L^{-1} \log{|\mathcal{A}|}+  \alpha H^{2} M +H^2L B_{\star} + 2  H^2L \left(B_\mathbb{P}+  B_\star \right) \\
\nonumber &+{\eta H^{2}(M+1)} +\sum_{m=1}^{M+1} (\alpha H^2+\eta \xi^2-\xi) | \mu^{m-1}|^2 \\
\nonumber & +\sum_{m=1}^M\sum_{h=1}^H \EE_{{\pi}^{\star,m}, \mathbb{P}^m} \left[ \iota_{r, h}^{m}(x_h,a_h) +\mu^m \iota_{g, h}^{m}(x_h,a_h)\right] \\
\nonumber & -\sum_{m=1}^{M} \sum_{h=1}^{H} \iota_{r, h}^{m}\left(x_{h}^{m}, a_{h}^{m}\right)+S_{r, H, 2}^{M}.
\end{align}
Then, by controlling the model prediction error in Lemma \ref{lemma: Model prediction difference bound tabular under local budget} and the martingale bound in Lemma \ref{lemma: Martingale Bound}, we have
\begin{align*}
&\operatorname{DR}(M)\\
\leq& \frac{1}{\alpha} H M L^{-1} \log{|\mathcal{A}|}+  \alpha H^{2} M + H^2L \left(2 B_\mathbb{P}+  3B_\star \right)\\
&+{\eta H^{2}(M+1)} +\sum_{m=1}^{M+1} (\alpha H^2+\eta \xi^2-\xi) | \mu^{m-1}|^2\\
&+4 \sqrt{H^{2} T \log \left(\frac{4}{p}\right)}+2B_{\mathbb{P}} H W + 2B_{r}W\\
& + C_5 MH^2 |\mathcal{S}|\sqrt{ |\mathcal{A}| W^{-1}} \sqrt{\log (|\mathcal{S}||\mathcal{A}| W / p) \log \left(M+1\right)}
\end{align*}
with probability at least $1-p$. Finally, by setting
$\alpha=  H^{-\frac{1}{3}} M^{-\rho} (B_\Delta+B_\star)^{\frac{1}{3}} $, $L= H^{-\frac{1}{3}} {M^{\frac{1+\rho}{2}}} (B_\Delta+B_\star)^{-\frac{2}{3}} $, $\eta=H^{-\frac{1}{3}} M^{-\frac{1}{2}}$, $\xi=2H^{\frac{5}{3}} (B_\Delta+B_\star)^{\frac{1}{3}} M^{-\rho}$,
$W= H^{\frac{2}{3}}|\mathcal{S}|^{\frac{2}{3}} |\mathcal{A}|^{\frac{1}{3}} \left(\frac{M}{B_\Delta}\right)^{\frac{2}{3}}$ with $\rho\in[\frac{1}{3},\frac{1}{2}]$, it holds that 
\begin{align*}
    \operatorname{DR}(M) \leq & \widetilde{\mathcal{O}}\left( H^{\frac{5}{3}}|\mathcal{S}|^{\frac{2}{3}} |\mathcal{A}|^{\frac{1}{3}} M^{\frac{1+\rho}{2}} (B_\Delta +B_\ast)^{\frac{1}{3}}\right)
\end{align*}
with probability at least $1-p$. This completes the proof.

\subsection{Proof of constraint violation in Theorem \ref{thm: Tabular Case MDP with local budget}}
From equation \eqref{eq: constraint vilation in linear MDP cited for tabular under local budget} and Lemma \ref{lemma: probability difference in expectation tabular}, one can conclude that
\begin{align} \label{eq: constraint violation before model prediction error}
& \mu \sum_{m=1}^M  \left(b_m-V_{g, 1}^{\pi^m,m}\left(x_{1}\right) \right) -(\frac{\xi M}{2}+\frac{1}{2\eta})|\mu|^2 \\
\nonumber\leq& HM + \frac{1}{\alpha} H M L^{-1} \log{|\mathcal{A}|}+ \alpha H^{2}M+H^2L \left(2 B_\mathbb{P}+  3B_\star \right)\\
\nonumber &+{\eta H^{2}(M+1)}+   \eta H^2 M \\
\nonumber & +\sum_{m=1}^M\sum_{h=1}^H \EE_{{\pi}^{\ast,m}, \mathbb{P}^m} \left[ \iota_{r, h}^{m}(x_h,a_h) +\mu^m \iota_{g, h}^{m}(x_h,a_h)\right]\\
\nonumber&-\sum_{m=1}^{M} \sum_{h=1}^{H} \iota_{r, h}^{m}\left(x_{h}^{m}, a_{h}^{m}\right)\\
\nonumber &-\sum_{m=1}^{M}\mu \sum_{h=1}^{H} \iota_{g, h}^{m}\left(x_{h}^{m}, a_{h}^{m}\right)
+S_{r, H, 2}^{M}+ \mu S_{g, H, 2}^{M} .
\end{align}
Then, by controlling the model prediction errors in Lemmas \ref{lemma: Model prediction difference bound tabular under local budget} as well as the martingale bounds in Lemmas \ref{lemma: Martingale Bound} and \ref{lemma: Martingale Bound for constraint}, we have
\begin{align*}
& \mu \sum_{m=1}^M  \left(b_m-V_{g, 1}^{\pi^m,m}\left(x_{1}\right) \right) -(\frac{\xi M}{2}+\frac{1}{2\eta})|\mu|^2 \\
\leq& HM + \frac{1}{\alpha} H M L^{-1} \log{|\mathcal{A}|}+ \alpha H^{2}M+H^2L \left(2 B_\mathbb{P}+  3B_\star \right)\\
& +{\eta H^{2}(M+1)}+   \eta H^2 M\\
&+  C_5 MH^2 |\mathcal{S}|\sqrt{ |\mathcal{A}| W^{-1}} \sqrt{\log (|\mathcal{S}||\mathcal{A}| W / p) \log \left(M+1\right)}\\
&+2B_{\mathbb{P}} H W + 2B_{r}W\\
& + C_6\mu MH^2 |\mathcal{S}|\sqrt{ |\mathcal{A}| W^{-1}} \sqrt{\log (|\mathcal{S}||\mathcal{A}| W / p) \log \left(M+1\right)}\\
&+ \mu B_{\mathbb{P}} H W + \mu B_g W \\
& +4H \sqrt{ T \log \left(\frac{4}{p}\right)} +4 H \mu \sqrt{ T \log \left(\frac{4}{p}\right)}.
\end{align*}
Furthermore, by substituting the parameters $\alpha=  H^{-\frac{1}{3}} M^{-\rho} (B_\Delta+B_\star)^{\frac{1}{3}} $, $L= H^{-\frac{1}{3}} {M^{\frac{1+\rho}{2}}} (B_\Delta+B_\star)^{-\frac{2}{3}} $, $\eta=H^{-\frac{1}{3}} M^{-\frac{1}{2}}$, $\xi=2H^{\frac{5}{3}} (B_\Delta+B_\star)^{\frac{1}{3}} M^{-\rho}$,
$W= H^{\frac{2}{3}}|\mathcal{S}|^{\frac{2}{3}} |\mathcal{A}|^{\frac{1}{3}} \left(\frac{M}{B_\Delta}\right)^{\frac{2}{3}}$ with $\rho\in[\frac{1}{3},\frac{1}{2}]$, and rearranging all the terms related to $\mu$ to the left hand side of the inequality,
it holds that 
\begin{align*}
& \mu \sum_{m=1}^M \left[ \left(b_m-V_{g, 1}^{\pi^m,m}\left(x_{1}\right) \right) - \widetilde{O} \left(  H^{\frac{5}{3}}|\mathcal{S}|^{\frac{2}{3}} |\mathcal{A}|^{\frac{1}{3}} M^{\frac{2}{3}} B_\Delta^{\frac{1}{3}}  \right)\right]\\
&-H^{\frac{5}{3}} (B_\Delta+B_\star)^{\frac{1}{3}} M^{1-\rho}|\mu|^2 \\
\leq &  \widetilde{\mathcal{O}}\left( HM+ H^{\frac{5}{3}}|\mathcal{S}|^{\frac{2}{3}} |\mathcal{A}|^{\frac{1}{3}} M^{\frac{1+\rho}{2}} (B_\Delta +B_\ast)^{\frac{1}{3}} \right).
\end{align*}

Note that the above inequality holds for every $\mu\geq 0$. By maximizing the both sides of above inequality over $\mu\geq 0$, i.e., by choosing 
$$\mu=\frac{\left[\sum_{m=1}^{M}\left(b_m-V_{g, 1}^{\pi^{m},m}\left(x_{1}\right) \right)- \widetilde{O} \left(  H^{\frac{5}{3}}|\mathcal{S}|^{\frac{2}{3}} |\mathcal{A}|^{\frac{1}{3}} M^{\frac{2}{3}} B_\Delta^{\frac{1}{3}}  \right)\right]_+}{2H^{\frac{5}{3}} (B_\Delta+B_\star)^{\frac{1}{3}} M^{1-\rho}},$$ 
we obtain
\begin{align*}
&\frac{\left( \left[\sum_{m=1}^{M}\left(b_m-V_{g, 1}^{\pi^{m},m}\left(x_{1}\right) \right)- \widetilde{O} \left(   H^{\frac{5}{3}}|\mathcal{S}|^{\frac{2}{3}} |\mathcal{A}|^{\frac{1}{3}} M^{\frac{2}{3}} B_\Delta^{\frac{1}{3}}  \right)\right]_+\right)^2}{4H^{\frac{5}{3}} ( B_\Delta+B_\star)^{\frac{1}{3}} M^{1-\rho}} \\
&\leq \widetilde{\mathcal{O}}\left( HM+ H^{\frac{5}{3}}|\mathcal{S}|^{\frac{2}{3}} |\mathcal{A}|^{\frac{1}{3}} M^{\frac{1+\rho}{2}} (B_\Delta +B_\ast)^{\frac{1}{3}} \right).
\end{align*}

Finally, one can conclude that 
\begin{align*}
\left[\sum_{m=1}^{M}b_m- V_{g, 1}^{\pi^{m},m}\left(x_{1}\right)\right]_{+} 
\leq & \widetilde{\mathcal{O}}\left(H^{\frac{5}{3}}|\mathcal{S}|^{\frac{2}{3}} |\mathcal{A}|^{\frac{1}{3}} M^{\frac{2-\rho}{2}}  (B_\Delta +B_\ast)^{\frac{1}{3}}\right)
\end{align*}
for $\rho\in[\frac{1}{3},\frac{1}{2}]$. This completes the proof.

%% file: files/appe_tabular_uniform_slater.tex
\section{Proof for tabular CMDP case under Assumption \ref{ass: Feasibility}}\label{sec:appe-tabular MDP under Feasibility}
The proof is similar to that of Theorem \ref{thm: Linear Kernal MDP under feasible}, and we will first prove the dynamic regret bound.
Since we only change the policy evaluation, all previous policy improvement results still hold.
\subsection{Model prediction error in the tabular case}
\begin{lemma}[Model prediction error bound in tabular case] \label{lemma: Model prediction difference bound tabular}
Let Assumption \ref{ass: Feasibility} hold. Fix $p \in(0,1)$ and let $\mathcal{E}$ be the epoch that the episode $m$ belongs to.
If we set $\lambda=1$, $LV=0$ and
\begin{align*}
\Gamma_{h}^m =\beta\left(n_{h}^m(x, a)+\lambda\right)^{-1 / 2},
\end{align*}
with $\beta=C_{4} H \sqrt{|\mathcal{S}| \log (|\mathcal{S}||\mathcal{A}| W / p)}$ in Algorithm \ref{alg:algoirthm 3} , then with probability at least $1-p / 2$ it holds that
\begin{align*}
&\sum_{m=1}^{M} \sum_{h=1}^{H}\left(\mathbb{E}_{\pi^{\star,m}, \mathbb{P}^{m}}\left[\iota_{r, h}^{m}\left(x_{h}, a_{h}\right)+\mu^m\iota_{g, h}^{m}\left(x_{h}, a_{h}\right)\right] \right.\\
&\left.-\iota_{r, h}^{m}\left(x_{h}^{m}, a_{h}^{m}\right)\right)\\
\leq & C_5 MH^2 |\mathcal{S}|\sqrt{ |\mathcal{A}| W^{-1}} \sqrt{\log (|\mathcal{S}||\mathcal{A}| W / p) \log \left(M+1\right)}\\
&+(2+\chi  )B_{\mathbb{P}} H W + (2B_{r}+\chi  B_g)W,
\end{align*}
where $C_{4}$, $C_5$ are some absolute constants, and $\mu^m\leq \chi$.
\end{lemma}

\begin{proof}
By Lemma \ref{lemma: bounds on model prediction error tabular}, for every $(m, h) \in[M] \times[H]$ and $(x, a) \in \mathcal{S} \times \mathcal{A}$, the following inequality holds with probability at least 1-p/2:
$$
-4\Gamma_{h}^{m}(x, a)-B_{\mathbb{P},\mathcal{E}} H -B_{\diamond,\mathcal{E}}\leq \iota_{\diamond, h}^{m}(x, a) \leq B_{\mathbb{P},\mathcal{E}} H  +B_{\diamond,\mathcal{E}}
$$
for $\diamond=r$ or $g$.
The rest of the proof is similar to Lemma \ref{lemma: Model prediction difference bound with local knowledge} and is thus omitted.
\end{proof}

\subsection{Proof of dynamic regret in Theorem \ref{thm: Tabular Case MDP under feasible}}
From equation \eqref{eq: dynamic regret in linear under feasibility for tabular} and Lemma \ref{lemma: probability difference in expectation tabular}, we have
\begin{align}
&\nonumber \operatorname{DR}(M) \\
\nonumber\leq& \frac{1}{\alpha} H M L^{-1} \log{|\mathcal{A}|}+ \alpha H^{2} \sum_{m=1}^M(1+\chi ^{2}) \\
\nonumber&+ 2  H^2L \left( 2B_\mathbb{P}+  3B_\star \right)  
+{\eta H^{2}(M+1)} \\
\nonumber&+\sum_{m=1}^{M+1} (\eta \xi^2-\xi) \chi ^2  \\
\nonumber &+\sum_{m=1}^M\sum_{h=1}^H \EE_{{\pi}^{\star,m}, \mathbb{P}^m} \left[ \iota_{r, h}^{m}(x_h,a_h) +\mu^m \iota_{g, h}^{m}(x_h,a_h)\right] \\
\nonumber&-\sum_{m=1}^{M} \sum_{h=1}^{H} \iota_{r, h}^{m}\left(x_{h}^{m}, a_{h}^{m}\right)+S_{r, H, 2}^{M}.
\end{align}
 Then, by Lemma \ref{lemma: Model prediction difference bound tabular}, it holds that
\begin{align}
&\nonumber \operatorname{DR}(M)\\
\nonumber\leq& \frac{1}{\alpha} H M L^{-1} \log{|\mathcal{A}|}+ \alpha H^{2} \sum_{m=1}^M(1+\chi ^{2})\\
\nonumber&+ H^2L \left( 2B_\mathbb{P}+  3B_\star \right) +{\eta H^{2}(M+1)} +\sum_{m=1}^{M+1} (\eta \xi^2-\xi) \chi ^2 \\
\nonumber & + C_5 MH^2 |\mathcal{S}|\sqrt{ |\mathcal{A}| W^{-1}} \sqrt{\log (|\mathcal{S}||\mathcal{A}| W / p) \log \left(M+1\right)}\\
\nonumber &+(2+\chi  )B_{\mathbb{P}} H W + (2B_{r}+\chi  B_g)W,
\end{align}
where $\chi=\frac{H}{\gamma}$.
Furthermore, by substituting the parameters $\alpha= \gamma H^{-\frac{3}{2}} M^{-\frac{1}{3}} (B_\Delta+B_\star)^{\frac{1}{3}} $, $L= {M^{\frac{2}{3}}} (B_\Delta+B_\star)^{-\frac{2}{3}} $, $\eta=M^{-\frac{1}{2}}$, $\xi=0$,
$W= |\mathcal{S}|^{\frac{2}{3}} |\mathcal{A}|^{\frac{1}{3}} \left(\frac{M}{B_\Delta}\right)^{\frac{2}{3}}$, we obtain
\begin{align*}
\operatorname{D-Regret}(M)\leq  \widetilde{\mathcal{O}}\left(\gamma^{-1} |\mathcal{S}|^{\frac{2}{3}} |\mathcal{A}|^{\frac{1}{3}}   H^{\frac{5}{2}} M^{\frac{2}{3}} (B_\Delta+B_\star)^{\frac{1}{3}} \right)
\end{align*}
with probability at least $1-p$.
Thus, we conclude the desired regret bound.

\subsection{Proof of constraint violation in Theorem \ref{thm: Tabular Case MDP under feasible}}
From equation \eqref{eq: constraint violation in linear under feasibility for tabular} and Lemma \ref{lemma: probability difference in expectation tabular}, it holds that
\begin{align}
\nonumber & \sum_{m=1}^{M} \left(V_{r, 1}^{\pi^{\star,m},m}\left(x_{1}\right)-V_{r, 1}^{m}\left(x_{1}\right)\right) +\sum_{m=1}^{M} \mu^{m}\left(b_m-V_{g, 1}^{m}\left(x_{1}\right)\right)\\
\nonumber\leq& \frac{1}{\alpha} H M L^{-1} \log{|\mathcal{A}|}+  \alpha H^{2} \sum_{m=1}^M(1+|\mu^{m}|^{2})  + H^2L \left(2 B_\mathbb{P}+  3B_\star \right) \\
\nonumber & +\sum_{m=1}^M\sum_{h=1}^H \EE_{{\pi}^{\star,m}, \mathbb{P}^m} \left[ \iota_{r, h}^{m}(x_h,a_h) +\mu^m \iota_{g, h}^{m}(x_h,a_h)\right]\\
\nonumber &-\sum_{m=1}^{M} \sum_{h=1}^{H} \iota_{r, h}^{m}\left(x_{h}^{m}, a_{h}^{m}\right)+S_{r, H, 2}^{M}.
\end{align}
Then, by controlling the model prediction errors in Lemmas \ref{lemma: Model prediction difference bound tabular}, and the martingale bounds in Lemmas \ref{lemma: Martingale Bound} and \ref{lemma: Martingale Bound for constraint}, we have
\begin{align}
& \sum_{m=1}^{M} \left(V_{r, 1}^{\pi^{\star,m},m}\left(x_{1}\right)-V_{r, 1}^{m}\left(x_{1}\right)\right) +\sum_{m=1}^{M} \mu^{m}\left(b_m-V_{g, 1}^{m}\left(x_{1}\right)\right) \\
\nonumber\leq& \frac{1}{\alpha} H M L^{-1} \log{|\mathcal{A}|}+ {\alpha(1+\chi^{2}) H^{3} M} +  H^2L \left( 2B_\mathbb{P}+ 3B_\star \right) \\
\nonumber & +C_5 MH^2 |\mathcal{S}| \sqrt{ |\mathcal{A}| W^{-1}} \sqrt{\log (|\mathcal{S}||\mathcal{A}| W / p) \log \left(M+1\right)}\\
\nonumber &+(2+\chi  )B_{\mathbb{P}} H W\\
\nonumber &+ (2B_{r}+\chi  B_g)W+4 H \sqrt{ HM \log (4 / p)}\\
\nonumber\leq&  \widetilde{\mathcal{O}}\left(\gamma^{-1} |\mathcal{S}|^{\frac{2}{3}} |\mathcal{A}|^{\frac{1}{3}}   H^{\frac{5}{2}} M^{\frac{2}{3}} (B_\Delta+B_\star)^{\frac{1}{3}} \right).
\end{align}
where the last inequality follows by substituting the parameters $\alpha=\gamma H^{-\frac{3}{2}} M^{-\frac{1}{3}} (\sqrt{d}B_\Delta+B_\star)^{\frac{1}{3}} $, $L= {M^{\frac{2}{3}}} (\sqrt{d}B_\Delta+B_\star)^{-\frac{2}{3}} $, $\eta=M^{-\frac{1}{2}}$, $\xi=0$, $W=d^{-\frac{1}{4}} H^{-1} {M}^{\frac{1}{2}} B_\Delta^{-\frac{1}{2}}$. 
Then, by combining the above inequality with \eqref{eq: change of mu} and setting $\mu=\chi$, it holds that
\begin{align}
\nonumber & \sum_{m=1}^{M} \left(V_{r, 1}^{\pi^{\star,m},m}\left(x_{1}\right)-V_{r, 1}^{m}\left(x_{1}\right)\right) +\sum_{m=1}^{M} \chi \left(b_m-V_{g, 1}^{m}\left(x_{1}\right)\right)\\
\nonumber& \leq  \widetilde{\mathcal{O}} \left(\gamma^{-1} |\mathcal{S}|^{\frac{2}{3}} |\mathcal{A}|^{\frac{1}{3}}   H^{\frac{5}{2}} M^{\frac{2}{3}} (B_\Delta+B_\star)^{\frac{1}{3}} + M^{\frac{1}{2}} \chi^2\right).
\end{align}
Finally, by Corollary \ref{corollary: Constraint Violation under Uniform Slater Condition}, we obtain
\begin{align*}
& \left[\sum_{m=1}^{M}b_m- V_{g, 1}^{\pi^{m},m}\left(x_{1}\right)\right]_{+} \\
\leq & \widetilde{\mathcal{O}}\left(\gamma^{-1} |\mathcal{S}|^{\frac{2}{3}} |\mathcal{A}|^{\frac{1}{3}}   H^{\frac{5}{2}} M^{\frac{2}{3}} (B_\Delta+B_\star)^{\frac{1}{3}}\right).
\end{align*}
This completes the proof.

%% file: files/appe_model_pred_error.tex
\section{Model prediction error}\label{sec:model predic error}
We first show that the prediction error in the value function can be expanded as the summation of the model prediction error and a martingale.
\begin{lemma}[Value prediction error expansion] \label{lemma: expansion of V m -V pi m}
It holds that
\begin{align*}
&   \sum_{m=1}^M \left(V_{r, 1}^{m}\left(x_{1}\right)-V_{r, 1}^{\pi^{m},m}\left(x_{1}\right) \right)\\
=&-\sum_{m=1}^M \sum_{h=1}^{H} \iota_{r, h}^{m}\left(x_{h}^{m}, a_{h}^{m}\right) + S_{r,H,2}^M.
\end{align*}
\end{lemma}
\begin{proof}
We recall the definition of $V_{r, h}^{\pi^{m},m}$ and define the operator $\mathcal{I}_{h}^{m}$ for function $f: \mathcal{S} \times \mathcal{A} \rightarrow \mathbb{R}$:
\begin{align*}
V_{r, h}^{\pi^{m},m}(x)&=\left\langle Q_{h}^{\pi^{m},m}(x, \cdot), \pi_{h}^{m}(\cdot \mid x)\right\rangle, \\
\left(\mathcal{I}_{h}^{m} f\right)(x)&=\left\langle f(x, \cdot), \pi_{h}^{m}(\cdot \mid x)\right\rangle .
\end{align*}
We expand the model prediction error $\iota_{r, h}^{m}$ into,
$$
\begin{aligned}
&\iota_{r, h}^{m}\left(x_{h}^{m}, a_{h}^{m}\right) \\
&=r_{h}^m\left(x_{h}^{m}, a_{h}^{m}\right)+\left(\mathbb{P}_{h}^m V_{r, h+1}^{m}\right)\left(x_{h}^{m}, a_{h}^{m}\right)-Q_{r, h}^{m}\left(x_{h}^{m}, a_{h}^{m}\right) \\
&=\left(r_{h}^m\left(x_{h}^{m}, a_{h}^{m}\right)+\left(\mathbb{P}_{h}^m V_{r, h+1}^{m}\right)\left(x_{h}^{m}, a_{h}^{m}\right)-Q_{r, h}^{\pi^m,m}\left(x_{h}^{m}, a_{h}^{m}\right)\right)\\
&\hspace{0.4cm}+\left(Q_{r, h}^{\pi^m,m}\left(x_{h}^{m}, a_{h}^{m}\right)-Q_{r, h}^{m}\left(x_{h}^{m}, a_{h}^{m}\right)\right) \\
&=\left(\mathbb{P}_{h}^m V_{r, h+1}^{m}-\mathbb{P}_{h}^m V_{r, h+1}^{\pi^m,m}\right)\left(x_{h}^{m}, a_{h}^{m}\right)\\
&+\left(Q_{r, h}^{\pi^{m},m}\left(x_{h}^{m}, a_{h}^{m}\right)-Q_{r, h}^{m}\left(x_{h}^{m}, a_{h}^{m}\right)\right),
\end{aligned}
$$
where we have used the Bellman equation $Q_{r, h}^{\pi^{m},m}\left(x_{h}^{m}, a_{h}^{m}\right)=r_{h}^m\left(x_{h}^{m}, a_{h}^{m}\right)+\left(\mathbb{P}_{h}^m V_{r, h+1}^{\pi^{m},m}\right)\left(x_{h}^{m}, a_{h}^{m}\right)$ in the last equality. With the above formula, we expand the difference $V_{r, 1}^{m}\left(x_{1}\right)-V_{r, 1}^{\pi^{m},m}\left(x_{1}\right)$ into
$$
\begin{aligned}
&V_{r, h}^{m}\left(x_{h}^{m}\right)-V_{r, h}^{\pi^m, m}\left(x_{h}^{m}\right)\\
=&\left(\mathcal{I}_{h}^{m}\left(Q_{r, h}^{m}-Q_{r, h}^{\pi^{m},m}\right)\right)\left(x_{h}^{m}\right)\\
=&\left(\mathcal{I}_{h}^{m}\left(Q_{r, h}^{m}-Q_{r, h}^{\pi^{m},m}\right)\right)\left(x_{h}^{m}\right)-\iota_{r, h}^{m}\left(x_{h}^{m}, a_{h}^{m}\right) \\
&+\left(\mathbb{P}_{h}^m V_{r, h+1}^{m}-\mathbb{P}_{h}^m V_{r, h+1}^{\pi^{m},m}\right)\left(x_{h}^{m}, a_{h}^{m}\right)\\
&+\left(Q_{r, h}^{\pi^{m},m}-Q_{r, h}^{m}\right)\left(x_{h}^{m}, a_{h}^{m}\right).
\end{aligned}
$$
Let
$$
\begin{aligned}
D_{r, h, 1}^{m}:=&\left(\mathcal{I}_{h}^{m}\left(Q_{r, h}^{m}-Q_{r, h}^{\pi^{m},m}\right)\right)\left(x_{h}^{m}\right)\\
&-\left(Q_{r, h}^{m}-Q_{r, h}^{\pi^m,m}\right)\left(x_{h}^{m}, a_{h}^{m}\right), \\
D_{r, h, 2}^{m}:=&\left(\mathbb{P}_{h}^m V_{r, h+1}^{m}-\mathbb{P}_{h}^m V_{r, h+1}^{\pi^{m},m}\right)\left(x_{h}^{m}, a_{h}^{m}\right)\\
&-\left(V_{r, h+1}^{m}-V_{r, h+1}^{\pi^{m},m}\right)\left(x_{h+1}^{m}\right)
\end{aligned}
$$
Therefore, we have the following recursive formula over $h$:
$$
\begin{aligned}
&V_{r, h}^{m}\left(x_{h}^{m}\right)-V_{r, h}^{\pi^{m},m}\left(x_{h}^{m}\right)\\
=&D_{r, h, 1}^{m}+D_{r, h, 2}^{m}+\left(V_{r, h+1}^{m}-V_{r, h+1}^{\pi^{m},m}\right)\left(x_{h+1}^{m}\right)-\iota_{r, h}^{m}\left(x_{h}^{m}, a_{h}^{m}\right) .
\end{aligned}
$$
Notice that $V_{r, H+1}^{\pi^{m},m}=V_{r, H+1}^{m}=0 .$ Summing the above equality over $h \in[H]$ yields that
\begin{align}\label{eq: 43}
\nonumber &V_{r, 1}^{m}\left(x_{1}\right)-V_{r, 1}^{\pi^{m},m}\left(x_{1}\right)\\
=&\sum_{h=1}^{H}\left(D_{r, h, 1}^{m}+D_{r, h, 2}^{m}\right)-\sum_{h=1}^{H} \iota_{r, h}^{m}\left(x_{h}^{m}, a_{h}^{m}\right) .
\end{align}

Following the definitions of $\mathcal{F}_{h, 1}^{m}$ and $\mathcal{F}_{h, 2}^{m}$, we know that $D_{r, h, 1}^{m} \in \mathcal{F}_{h, 1}^{m}$ and $D_{r, h, 2}^{m} \in \mathcal{F}_{h, 2}^{m}$. Thus, for every $(m, h) \in$ $[M] \times[H]$,
$$
\mathbb{E}\left[D_{r, h, 1}^{m} \mid \mathcal{F}_{h-1,2}^{m}\right]=0 \text { and } \mathbb{E}\left[D_{r, h, 2}^{m} \mid \mathcal{F}_{h, 1}^{m}\right]=0
$$
Notice that $t(m, 0,2)=t(m-1, H, 2)=2 H(m-1) .$ Clearly, $\mathcal{F}_{0,2}^{m}=\mathcal{F}_{H, 2}^{m-1}$ for all $m \geq 2.$ Let $\mathcal{F}_{0,2}^{1}$ be empty. We define the martingale sequence:
$$
\begin{aligned}
&S_{r, h, k}^{m}\\
&=\sum_{ \tau=1}^{m-1} \sum_{i=1}^{H}\left(D_{r, i, 1}^{\tau}+D_{r, i, 2}^{\tau}\right)+\sum_{i=1}^{h-1}\left(D_{r, i, 1}^{m}+D_{r, i, 2}^{m}\right)+\sum_{j=1}^{m} D_{r, h, \ell}^{m} \\
&=\sum_{(\tau, i, j) \in[M] \times[H] \times[2], t(\tau, i, j) \leq t(m, h, k)} D_{r, i, \ell}^{\tau}
\end{aligned}
$$
where $t(m, h, k):=2(m-1) H+2(h-1)+k$ is the time index. Clearly, this martingale is adapted to the filtration $\left\{\mathcal{F}_{h, k}^{m}\right\}_{(m, h, k) \in[M] \times[H] \times[2]}$, and particularly,
$$
\sum_{m=1}^{M} \sum_{h=1}^{H}\left(D_{r, h, 1}^{m}+D_{r, h, 2}^{m}\right)=S_{r, H, 2}^{M} .
$$
Finally, we combine the above martingale with \eqref{eq: 43} to obtain the desired result.

\end{proof}

\subsection{Linear Kernel MDP case}
\begin{lemma} \label{lemma: concentrarion of V-PV}
Let $\lambda=1$ in Algorithm \ref{alg:algoirthm 2}. Fix $p \in(0,1)$. Then, for every $(m, h) \in[M] \times[H]$ it holds for $\diamond=r$ or $g$ that
$$
\begin{aligned}
& \left\|\sum_{\tau=\ell^m_Q}^{m-1} \phi_{\diamond, h}^{\tau}\left(x_{h}^{\tau}, a_{h}^{\tau}\right)^{\top}\left(V_{\diamond, h+1}^{\tau}\left(x_{h+1}^{\tau}\right) \right.\right.\\ &\hspace{3cm} \left.\left.-\left(\mathbb{P}_{h}^{\tau} V_{\diamond, h+1}^{\tau}\right)\left(x_{h}^{\tau}, a_{h}^{\tau}\right)\right)\right\|_{\left(\Lambda_{\diamond, h}^{m}\right)^{-1}} \\
\leq& C_1 \sqrt{d_1 H^{2} \log \left(\frac{d_1 W}{p}\right)}
\end{aligned}
$$
with probability at least $1-p / 2$, where $C_1>0$ is an absolute constant.
\end{lemma}
\begin{proof}
The lemma is slightly different than \cite{cai2020provably}[Lemma D.1], since they assume that MDPs are stationary and $\mathbb{P}_{h}$ is fixed over different episodes. It can be verified that the proof for the stationary case still holds in our non-stationary case without any modifications since the results in \cite{cai2020provably} holds for the least-squares value iteration for all value functions that are determined by $Q^\tau_{\diamond,h+1}$ and $\pi^\tau_{h+1}$, which are further determined by the historical data in $\mathcal{F}^\tau_{h,1}$.
\end{proof}

\begin{lemma} \label{lemma: bound Gamma diamond}
If we set $C_1>1$, $\lambda=1$, $\beta=C_1\sqrt{d H^{2} \log (d W / p)}$ and
$\Gamma_{\diamond, h}^{m}(\cdot, \cdot)=\beta\left(\phi_{\diamond, h}^{m}(\cdot, \cdot)^{\top}\left(\Lambda_{\diamond, h}^{m}\right)^{-1} \phi_{\diamond, h}^{m}(\cdot, \cdot)\right)^{1 / 2}$  in line 6 of Algorithm \ref{alg:algoirthm 2}, it holds that
$$
\begin{aligned}
&\left|\phi_{\diamond, h}^{m}(x, a)^{\top} w_{\diamond, h}^{m}-\left(\mathbb{P}_{h}^m V_{\diamond, h+1}^{m}\right)(x, a)\right| \\
&\leq \Gamma_{\diamond, h}^{m}(x, a)+B_{\mathbb{P},\mathcal{E}} H^2 d_1\sqrt{d_1 W} 
\end{aligned}
$$
with probability at least $1-p / 2$ for all $(m, h) \in[M] \times[H]$ and $(x, a) \in \mathcal{S} \times \mathcal{A}$, where the symbol $\diamond$ is equal to $r$ or $g$. 
\end{lemma}
\begin{proof}
We recall the definition of the feature map $\phi_{r, h}^{m}$:
$$
\phi_{r, h}^{m}(x, a)=\int_{\mathcal{S}} \psi\left(x, a, x^{\prime}\right) V_{r, h+1}^{m}\left(x^{\prime}\right) d x^{\prime}
$$
for all $(m, h) \in[M] \times[H]$ and $(x, a) \in \mathcal{S} \times \mathcal{A}$. By Assumption 2, we have
\begin{align}
&\nonumber \left(\mathbb{P}_{h}^m V_{r, h+1}^{m}\right)(x, a) \\
\nonumber &=\int_{\mathcal{S}} \psi\left(x, a, x^{\prime}\right)^{\top} \theta_{h}^m \cdot V_{r, h+1}^{m}\left(x^{\prime}\right) d x^{\prime} \\
\nonumber &=\phi_{r, h}^{m}(x, a)^{\top} \theta_{h}^m \\
\nonumber &=\phi_{r, h}^{m}(x, a)^{\top}\left(\Lambda_{r, h}^{m}\right)^{-1} \Lambda_{r, h}^{m} \theta_{h}^m \\
\nonumber &=\phi_{r, h}^{m}(x, a)^{\top}\left(\Lambda_{r, h}^{m}\right)^{-1} \\
\nonumber&\hspace{2cm}\left(\sum_{\tau=\ell^m_Q}^{m-1} \phi_{r, h}^{\tau}\left(x_{h}^{\tau}, a_{h}^{\tau}\right) \phi_{r, h}^{\tau}\left(x_{h}^{\tau}, a_{h}^{\tau}\right)^{\top} \theta_{h}^m+\lambda \theta_{h}^m\right) \\
\nonumber &=\phi_{r, h}^{m}(x, a)^{\top}\left(\Lambda_{r, h}^{m}\right)^{-1} \\
\nonumber&\hspace{2cm}\left(\sum_{\tau=\ell^m_Q}^{m-1} \phi_{r, h}^{\tau}\left(x_{h}^{\tau}, a_{h}^{\tau}\right) \phi_{r, h}^{\tau}\left(x_{h}^{\tau}, a_{h}^{\tau}\right)^{\top} \theta_{h}^\tau+\lambda \theta_{h}^m\right) \\
\nonumber &+\phi_{r, h}^{m}(x, a)^{\top}\left(\Lambda_{r, h}^{m}\right)^{-1}\\
\nonumber&\hspace{2cm}\left(\sum_{\tau=\ell^m_Q}^{m-1} \phi_{r, h}^{\tau}\left(x_{h}^{\tau}, a_{h}^{\tau}\right) \phi_{r, h}^{\tau}\left(x_{h}^{\tau}, a_{h}^{\tau}\right)^{\top} \left(\theta_{h}^m-\theta_h^\tau\right)\right) \\
\nonumber &=\phi_{r, h}^{m}(x, a)^{\top}\left(\Lambda_{r, h}^{m}\right)^{-1}\\
\nonumber&\hspace{1.5cm}\left(\sum_{\tau=\ell^m_Q}^{m-1} \phi_{r, h}^{\tau}\left(x_{h}^{\tau}, a_{h}^{\tau}\right) \cdot\left(\mathbb{P}_{h}^\tau V_{r, h+1}^{\tau}\right)\left(x_{h}^{\tau}, a_{h}^{\tau}\right)+\lambda \theta_{h}^m\right)\\
\nonumber &+\phi_{r, h}^{m}(x, a)^{\top}\left(\Lambda_{r, h}^{m}\right)^{-1}\\
&\hspace{2cm} \left(\sum_{\tau=\ell^m_Q}^{m-1} \phi_{r, h}^{\tau}\left(x_{h}^{\tau}, a_{h}^{\tau}\right) \phi_{r, h}^{\tau}\left(x_{h}^{\tau}, a_{h}^{\tau}\right)^{\top} \left(\theta_{h}^m-\theta_h^\tau\right)\right) \label{eq: variation P term}
\end{align}
where the second equality is due to the definition of $\phi_{r, h}^{m}$, we use $\Lambda_{r, h}^{m}=\sum_{\tau=\ell^m_Q}^{m-1} \phi_{r, h}^{\tau}\left(x_{h}^{\tau}, a_{h}^{\tau}\right) \phi_{r, h}^{\tau}\left(x_{h}^{\tau}, a_{h}^{\tau}\right)^{\top}+\lambda I$ from line 3
of Algorithm \ref{alg:algoirthm 2} in the fourth equality, and we recursively replace $\phi_{r, h}^{\tau}\left(x_{h}^{\tau}, a_{h}^{\tau}\right)^{\top} \theta_{h}^{\tau}$ by $\left(\mathbb{P}_{h}^{\tau} V_{r, h+1}^{\tau}\right)\left(x_{h}^{\tau}, a_{h}^{\tau}\right)$ for all $\tau \in[\ell^m_Q, m-1]$ in the last equality. For the term in \eqref{eq: variation P term}, it holds that

\begin{align*}
&\left|\phi_{r, h}^{m}(x, a)^{\top}\left(\Lambda_{r, h}^{m}\right)^{-1} \right.\\
&\hspace{1cm}\left.\left(\sum_{\tau=\ell^m_Q}^{m-1} \phi_{r, h}^{\tau}\left(x_{h}^{\tau}, a_{h}^{\tau}\right) \phi_{r, h}^{\tau}\left(x_{h}^{\tau}, a_{h}^{\tau}\right)^{\top} \left(\theta_{h}^m-\theta_h^\tau\right)\right)\right|\\
\leq &\sum_{\tau=\ell^m_Q}^{m-1} \left|\phi_{r, h}^{m}(x, a)^{\top}\left(\Lambda_{r, h}^{m}\right)^{-1} \phi_{r, h}^{\tau}\left(x_{h}^{\tau}, a_{h}^{\tau}\right) \right|\\
&\hspace{3cm}\left| \phi_{r, h}^{\tau}\left(x_{h}^{\tau}, a_{h}^{\tau}\right)^{\top} \left(\theta_{h}^m-\theta_h^\tau\right)\right|\\
\leq &B_{\mathbb{P},\mathcal{E}} \sqrt{d_1}H \sum_{\tau=\ell^m_Q}^{m-1} \left|\phi_{r, h}^{m}(x, a)^{\top}\left(\Lambda_{r, h}^{m}\right)^{-1} \phi_{r, h}^{\tau}\left(x_{h}^{\tau}, a_{h}^{\tau}\right) \right|\\
\leq &B_{\mathbb{P},\mathcal{E}} \sqrt{d_1}H \sqrt{\sum_{\tau=\ell^m_Q}^{m-1} \norm{\phi_{r, h}^{m}}^2_{(\Lambda_{r, h}^{m})^{-1}}} \\
&\hspace{3cm}\sqrt{\sum_{\tau=\ell^m_Q}^{m-1} (\phi_{r, h}^{\tau})^\top (\Lambda_{r, h}^{m})^{-1} \phi_{r, h}^{\tau}}\\
\leq & d_1 \sqrt{m-\ell^m_Q} B_{\mathbb{P},\mathcal{E}} H \norm{\phi_{r, h}^{m}}_{(\Lambda_{r, h}^{m})^{-1}}\\
\leq & B_{\mathbb{P},\mathcal{E}} H^2 d_1\sqrt{d_1W/\lambda} 
\end{align*}
where the first and third inequalities are due to Cauchy-Schwarz inequality, the second inequality is due to the boundedness of $\phi_{r, h}^{\tau}$ and the definition of variation budget of $B_{\mathbb{P},\mathcal{E}}$, the fourth inequality is due to Lemma \ref{lemma: lemma D1 in Jin2020}, and the last inequality is due to $\left\|\phi_{r, h}^{m}\right\|_{\left(\Lambda_{r, h}^{m}\right)^{-1}} \leq \sqrt{d_1/\lambda}H$ by noticing that $\Lambda_{r, h}^{m} \succeq \lambda I$ and $\left\|\phi_{r, h}^{m}\right\| \leq \sqrt{d_1}H$. 

We recall the update $w_{r, h}^{m}=\left(\Lambda_{r, h}^{m}\right)^{-1} \sum_{\tau=\ell^m_Q}^{m-1} \phi_{r, h}^{\tau}\left(x_{h}^{\tau}, a_{h}^{\tau}\right) V_{r, h+1}^{\tau}\left(x_{h+1}^{\tau}\right)$ from line 4 of Algorithm \ref{alg:algoirthm 2}. Therefore,
$$
\begin{aligned}
&\left|\phi_{r, h}^{m}(x, a)^{\top} w_{r, h}^{m}-\left(\mathbb{P}_{h}^m V_{r, h+1}^{m}\right)(x, a)\right|  \\
\leq&\left|\phi_{r, h}^{m}(x, a)^{\top}\left(\Lambda_{r, h}^{m}\right)^{-1} \sum_{\tau=\ell^m_Q}^{m-1} \phi_{r, h}^{\tau}\left(x_{h}^{\tau}, a_{h}^{\tau}\right) \right.\\
&\hspace{2cm}\left.\cdot\left(V_{r, h+1}^{\tau}\left(x_{h+1}^{\tau}\right)-\left(\mathbb{P}_{h}^\tau V_{r, h+1}^{\tau}\right)\left(x_{h}^{\tau}, a_{h}^{\tau}\right)\right)\right| \\
&+\left|\lambda \cdot \phi_{r, h}^{m}(x, a)^{\top}\left(\Lambda_{r, h}^{m}\right)^{-1} \theta_{h}^m\right|\\
&+ d_1 \sqrt{m-\ell^m_Q} B_{\mathbb{P},\mathcal{E}} H \norm{\phi_{r, h}^{m}}_{(\Lambda_{r, h}^{m})^{-1}} \\
\leq &\norm{\phi_{r, h}^{m}}_{(\Lambda_{r, h}^{m})^{-1}}\left\|\sum_{\tau=\ell^m_Q}^{m-1} \phi_{r, h}^{\tau}\left(x_{h}^{\tau}, a_{h}^{\tau}\right) \right.\\
& \left.\cdot\left(V_{r, h+1}^{\tau}\left(x_{h+1}^{\tau}\right)-\left(\mathbb{P}_{h}^\tau V_{r, h+1}^{\tau}\right)\left(x_{h}^{\tau}, a_{h}^{\tau}\right)\right)\right\|_{\left(\Lambda_{r, h}^{m}\right)^{-1}} \\
&+\lambda \norm{\phi_{r, h}^{m}}_{(\Lambda_{r, h}^{m})^{-1}} \left\|\theta_{h}^m\right\|_{\left(\Lambda_{r, h}^{m}\right)-1}+ B_{\mathbb{P},\mathcal{E}} H^2 d_1\sqrt{d_1W/\lambda} 
\end{aligned}
$$
for all $(m, h) \in[M] \times[H]$ and $(x, a) \in \mathcal{S} \times \mathcal{A}$, where we apply the Cauchy-Schwarz inequality twice in the inequality. By Lemma \ref{lemma: concentrarion of V-PV}, by setting $\lambda=1$, with probability at least $1-p / 2$ it holds that
\begin{align*}
&\left\|\sum_{\tau=\ell^m_Q}^{m-1} \phi_{r, h}^{\tau}\left(x_{h}^{\tau}, a_{h}^{\tau}\right) \right.\\
&\hspace{1cm}\left.\cdot\left(V_{r, h+1}^{\tau}\left(x_{h+1}^{\tau}\right)-\left(\mathbb{P}_{h}^\tau V_{r, h+1}^{\tau}\right)\left(x_{h}^{\tau}, a_{h}^{\tau}\right)\right)\right\|_{\left(\Lambda_{r, h}^{m}\right)^{-1}}\\
&\leq C_1 \sqrt{d_1 H^{2} \log \left(\frac{d_1 W}{p}\right)} .
\end{align*}

Moreover, notice that $\Lambda_{r, h}^{m} \succeq \lambda I$ and $\left\|\theta_{h}^m\right\| \leq \sqrt{d_1}$. Thus, $\left\|\theta_{h}^m\right\|_{\left(\Lambda_{r, h}^{m}\right)^{-1}} \leq \sqrt{d_1/\lambda}$. In addition, by taking an appropriate absolute constant $C_1$, we obtain that
\begin{align*}
&\left|\phi_{r, h}^{m}(x, a)^{\top} w_{r, h}^{m}-\left(\mathbb{P}^m_{h} V_{r, h+1}^{m}\right)(x, a)\right| \\
&\leq C_1 \sqrt{d_1 H^{2} \log \left(\frac{d_1 W}{p}\right)} \norm{\phi_{r, h}^{m}}_{(\Lambda_{r, h}^{m})^{-1}} + B_{\mathbb{P},\mathcal{E}} H^2 d_1\sqrt{d_1W}
\end{align*}

for all $(m, h) \in[M] \times[H]$ and $(x, a) \in \mathcal{S} \times \mathcal{A}$ under the event of Lemma 12.
We now set $C_1>1$ and $\beta=C_1\sqrt{d H^{2} \log \left(\frac{d W}{p}\right)} $. By the exploration bonus $\Gamma_{r, h}^{m}$ in line 6 of Algorithm \ref{alg:algoirthm 2}, with probability at least $1-p / 2$ it holds that
\begin{align*}
&\left|\phi_{r, h}^{m}(x, a)^{\top} w_{r, h}^{m}-\left(\mathbb{P}_{h}^m V_{r, h+1}^{m}\right)(x, a)\right| \\
\leq &\Gamma_{r, h}^{m}(x, a)+B_{\mathbb{P},\mathcal{E}} H^2 d_1\sqrt{d_1W} 
\end{align*}

for all $(m, h) \in[M] \times[H]$ and $(x, a) \in \mathcal{S} \times \mathcal{A}$. Similarly, one can derive the inequality $\left|\phi_{g, h}^{m}(x, a)^{\top} w_{g, h}^{m}-\left(\mathbb{P}_{h}^m V_{g, h+1}^{m}\right)(x, a)\right| \leq \Gamma_{g, h}^{m}(x, a)+B_{\mathbb{P},\mathcal{E}} H^2 d_1\sqrt{d_1 W}$.
\end{proof}

\begin{lemma}\label{lemma: bound Gamma}
If we set $C_1>1$, $\lambda=1$, $\beta=C_1\sqrt{d H^{2} \log (d W / p)}$ and $\Gamma_{h}^{m}(\cdot, \cdot)=\beta\left(\varphi(\cdot, \cdot)^{\top}\left(\Lambda_{h}^{m}\right)^{-1} \varphi(\cdot, \cdot)\right)^{1 / 2}$ in line 9 of Algorithm \ref{alg:algoirthm 2},  then it holds that
$$
\left|\varphi(x, a)^{\top} u_{\diamond, h}^{m}-\diamond_{h}^{m}(x, a)\right| \leq \Gamma_{h}^{m}(x, a) +B_{r,\mathcal{E}}\sqrt{d_2 W}
$$
with probability at least $1-p / 2$ for every $(m, h) \in[M] \times[H]$ and $(x, a) \in \mathcal{S} \times \mathcal{A}$ , where the symbol $\diamond$ is equal to $r$  or $g$. 
\end{lemma}

\begin{proof}
We note that the reward functions vary over episodes, namely, $r_{h}^\tau\left(x_{h}^{\tau}, a_{h}^{\tau}\right):=\varphi\left(x_{h}^{\tau}, a_{h}^{\tau}\right)^{\top} \theta_{r, h}^\tau$. For the difference $\varphi(x, a)^{\top} u_{r, h}^{m}-r^m_{h}(x, a)$, we have
$$
\begin{aligned}
&\left|\varphi(x, a)^{\top} u_{r, h}^{m}-r_{h}^m(x, a)\right| \\
&=\left|\varphi(x, a)^{\top} u_{r, h}^{m}-\varphi(x, a)^{\top} \theta^m_{r, h}\right| \\
&=\left|\varphi(x, a)^{\top}\left(\Lambda_{h}^{m}\right)^{-1} \right.\\
&\hspace{2cm} \left.\left(\sum_{ \tau=\ell^m_Q}^{m-1} \varphi\left(x_{h}^{\tau}, a_{h}^{\tau}\right) r^\tau_{h}\left(x_{h}^{\tau}, a_{h}^{\tau}\right)-\Lambda_{h}^{m} \theta_{r, h}^m\right)\right| \\
&=\left|\varphi(x, a)^{\top}\left(\Lambda_{h}^{m}\right)^{-1} \right.\\
&\left.\left(\sum_{ \tau=\ell^m_Q}^{m-1} \varphi\left(x_{h}^{\tau}, a_{h}^{\tau}\right)\left(r_{h}^\tau\left(x_{h}^{\tau}, a_{h}^{\tau}\right)-\varphi\left(x_{h}^{\tau}, a_{h}^{\tau}\right)^{\top} \theta_{r, h}^m\right)-\lambda \theta_{r, h}^m\right)\right| \\
&=\left|\varphi(x, a)^{\top}\left(\Lambda_{h}^{m}\right)^{-1}\right.\\
&\left.\left(\sum_{ \tau=\ell^m_Q}^{m-1} \varphi\left(x_{h}^{\tau}, a_{h}^{\tau}\right)\left(r_{h}^\tau\left(x_{h}^{\tau}, a_{h}^{\tau}\right)-r_{h}^m \left(x_{h}^{\tau}, a_{h}^{\tau}\right)\right)-\lambda \theta_{r, h}^m\right)\right| \\
&\leq \sum_{ \tau=\ell^m_Q}^{m-1} \left|\varphi(x, a)^{\top}\left(\Lambda_{h}^{m}\right)^{-1} \varphi\left(x_{h}^{\tau}, a_{h}^{\tau}\right)\right| \\
&\hspace{0.5cm}\left|r_{h}^m\left(x_{h}^{\tau}, a_{h}^{\tau}\right)-r_{h}^\tau \left(x_{h}^{\tau}, a_{h}^{\tau}\right)\right|+\lambda\left|\varphi(x, a)^{\top}\left(\Lambda_{h}^{m}\right)^{-1} \theta_{r, h}^m\right| \\
&\leq B_{r,\mathcal{E}}\sum_{ \tau=\ell^m_Q}^{m-1} \left|\varphi(x, a)^{\top}\left(\Lambda_{h}^{m}\right)^{-1} \varphi\left(x_{h}^{\tau}, a_{h}^{\tau}\right)\right| \\
&\hspace{2cm}+\lambda\left|\varphi(x, a)^{\top}\left(\Lambda_{h}^{m}\right)^{-1} \theta_{r, h}^m\right| \\
&\leq B_{r,\mathcal{E}} \sqrt{\sum_{ \tau=\ell^m_Q}^{m-1} \norm{\varphi(x, a)}^2_{\left(\Lambda_{h}^{m}\right)^{-1}}} \sqrt{\sum_{ \tau=\ell^m_Q}^{m-1} \norm{\varphi(x_h^\tau, a_h^\tau)}^2_{\left(\Lambda_{h}^{m}\right)^{-1}}} \\
&\hspace{1cm} + \lambda\left\|\theta^m_{r, h}\right\|_{\left(\Lambda_{h}^{m}\right)^{-1}} \norm{\varphi(x, a)}_{\left(\Lambda_{h}^{m}\right)^{-1}}\\
&\leq  B_{r,\mathcal{E}}\sqrt{d_2 W}\norm{\varphi(x, a)}_{\left(\Lambda_{h}^{m}\right)^{-1}} +\\
&\hspace{1cm}\lambda\left\|\theta^m_{r, h}\right\|_{\left(\Lambda_{h}^{m}\right)^{-1}} \norm{\varphi(x, a)}_{\left(\Lambda_{h}^{m}\right)^{-1}}
\end{aligned}
$$
where we use the definition of $B_{r,\mathcal{E}}$ in the second inequality, the Cauchy-Schwartz inequality in the third inequality, and Lemma \ref{lemma: lemma D1 in Jin2020} in the last inequality. 

Notice that $\Lambda_{h}^{m} \succeq \lambda I$, $\left\|\varphi(\cdot,\cdot)\right\| \leq 1$ and $\left\|\theta_{r, h}\right\| \leq \sqrt{d_2}$. Thus, $\norm{\varphi(\cdot, \cdot)}_{\left(\Lambda_{h}^{m}\right)^{-1}}\leq \sqrt{1/\lambda}$ and $\left\|\theta_{r, h}\right\|_{\left(\Lambda_{h}^{m}\right)^{-1}} \leq \sqrt{d_2/\lambda }$. Hence, if we set $\lambda=1$ and $\beta=C_1\sqrt{d H^{2} \log (d W / p)}$, then every $(m, h) \in[M] \times[H]$ and $(x, a) \in \mathcal{S} \times \mathcal{A}$, we have
$\left|\varphi(x, a)^{\top} u_{r, h}^{m}-r_{h}^{m}(x, a)\right| \leq \Gamma_{h}^{m}(x, a)+B_{r,\mathcal{E}}\sqrt{d_2 W}.$
Similarly,one can derive the inequality $\left|\varphi(x, a)^{\top} u_{g, h}^{m}-g_{h}^{m}(x, a)\right| \leq \Gamma_{h}^{m}(x, a)+B_{r,\mathcal{E}}\sqrt{d_2 W}$.
\end{proof}

\begin{lemma}\label{lemma: bounds on model prediction error}
Let   Assumption \ref{ass: linear fun approx} hold. Fix $p \in(0,1)$ and let $\mathcal{E}$ be the epoch that the episode $m$ belongs to.
If we set $\lambda=1$, $LV=0$ and
\begin{align*}
&\Gamma_{h}^{m}(\cdot, \cdot)=\beta\left(\varphi(\cdot, \cdot)^{\top}\left(\Lambda_{h}^{m}\right)^{-1} \varphi(\cdot, \cdot)\right)^{1 / 2}, \\
&\Gamma_{r,h}^m =\beta \left( (\phi_{r,h}^m)^\top (\Lambda_{r, h}^m)^{-1} \phi_{r,h}^m \right)^{1/2}, \\
&\Gamma_{g,h}^m =\beta \left( (\phi_{g,h}^m)^\top (\Lambda_{g, h}^m)^{-1} \phi_{g,h}^m \right)^{1/2},
\end{align*}
with $\beta=$ $C_{1} \sqrt{d H^{2} \log (d W / p)}$ in Algorithm \ref{alg:algoirthm 2}, then it holds that
$$
\begin{aligned}
&-2\left(\Gamma_{h}^{m}+\Gamma_{\diamond, h}^{m}\right)(x, a)-B_{\mathbb{P},\mathcal{E}} H^2 d_1\sqrt{d_1 W} -B_{\diamond,\mathcal{E}}\sqrt{d_2 W}\\
&\leq \iota_{\diamond, h}^{m}(x, a) \leq B_{\mathbb{P},\mathcal{E}} H^2 d_1\sqrt{d_1 W} +B_{\diamond,\mathcal{E}}\sqrt{d_2 W}
\end{aligned}
$$
with probability at least $1-p / 2$ for every $(m, h) \in[M] \times[H]$ and $(x, a) \in \mathcal{S} \times \mathcal{A}$, where the symbol $\diamond$ is equal to $r$ or $g$. 
\end{lemma}

\begin{proof}
We recall the model prediction error $\iota_{r, h}^{m}:=r_{h}^m+\mathbb{P}_{h}^m V_{r, h+1}^{m}-Q_{r, h}^{m}$ and the estimated state-action value function $Q_{r, h}^{m}$ in line 10 of Algorithm \ref{alg:algoirthm 2}:
$$
\begin{aligned}
Q_{r, h}^{m}(x, a)=&\min \left(\varphi(x, a)^{\top} u_{r, h}^{m}+\phi_{r, h}^{m}(x, a)^{\top} w_{r, h}^{m}\right. \\
\hspace{2cm}&\left.+\left(\Gamma_{h}^{m}+\Gamma_{r, h}^{m}\right)(x, a), H-h+1\right)_{+}
\end{aligned}
$$
for all $(m, h) \in[M] \times[H]$ and $(x, a) \in \mathcal{S} \times \mathcal{A}$. 
Then, for every $(m, h) \in[M] \times[H]$ and $(x, a) \in \mathcal{S} \times \mathcal{A}$, one can write that
$$
\begin{aligned}
&-\iota_{r, h}^{m}(x, a) \\
&=Q_{r, h}^{m}(x, a)-\left(r_{h}^m+\mathbb{P}_{h}^m V_{r, h+1}^{m}\right)(x, a) \\
&\leq \varphi(x, a)^{\top} u_{r, h}^{m}+\phi_{r, h}^{m}(x, a)^{\top} w_{r, h}^{m}+\left(\Gamma_{h}^{m}+\Gamma_{r, h}^{m}\right)(x, a)\\
&-\left(r_{h}^{m}+\mathbb{P}_{h}^m V_{r, h+1}^{m}\right)(x, a) \\
&=\left(\varphi(x, a)^{\top} u_{r, h}^{m }-r_{h}^m(x, a)\right)\\
&+ \left(\phi_{r, h}^{m}(x, a)^{\top} w_{r, h}^{m}-\mathbb{P}_{h}^m V_{r, h+1}^{m}(x,a) \right)+\Gamma_{h}^{m}(x, a)+\Gamma_{r, h}^{m}(x, a)\\
&\leq 2\Gamma_{h}^{m}(x, a)+2\Gamma_{r, h}^{m}(x, a)+B_{\mathbb{P},\mathcal{E}} H^2 d_1\sqrt{d_1 W} +B_{r,\mathcal{E}}\sqrt{d_2 W}
\end{aligned}
$$
where the last inequality holds due to Lemmas \ref{lemma: bound Gamma diamond} and \ref{lemma: bound Gamma}.

On the other hand, notice that $\left(r_{h}^{m}+\mathbb{P}_{h}^{m} V_{r, h+1}^{m}\right)(x, a) \leq H-h+1$ Thus, for every $(m, h) \in[M] \times[H]$ and $(x, a) \in \mathcal{S} \times \mathcal{A}$, it holds that
$$
\begin{aligned}
&\iota_{r, h}^{m}(x, a) \\
&=\left(r_{h}^{m}+\mathbb{P}_{h}^{m} V_{r, h+1}^{m}\right)(x, a)-Q_{r, h}^{m}(x, a) \\
&=\left(r_{h}^{m}+\mathbb{P}_{h}^{m} V_{r, h+1}^{m}\right)(x, a)-\min \left(\varphi(x, a)^{\top} u_{r, h}^{m}\right.\\
&\left.+\phi_{r, h}^{m}(x, a)^{\top} w_{r, h}^{m}+\left(\Gamma_{h}^{m}+\Gamma_{r, h}^{m}\right)(x, a), H-h+1\right)^{+} \\
&\leq \max \left(r_{h}^m(x, a)-\varphi(x, a)^{\top} u_{r, h}^{m}-\Gamma_{h}^{m}(x, a)\right.\\
&\left.+\left(\mathbb{P}_{h}^m V_{r, h+1}^{m}\right)(x, a)-\phi_{r, h}^{m}(x, a)^{\top} w_{r, h}^{m}-\Gamma_{r, h}^{m}(x, a), 0\right)^{+} \\
&\leq B_{\mathbb{P},\mathcal{E}} H^2 d_1\sqrt{d_1 W} +B_{r,\mathcal{E}}\sqrt{d_2 W}
\end{aligned}
$$
where the last inequality holds due to Lemmas \ref{lemma: bound Gamma diamond} and \ref{lemma: bound Gamma}.
Therefore, we have proved that with probability at least $1-p / 2$ it holds that
$$
\begin{aligned}
&-2\left(\Gamma_{h}^{m}+\Gamma_{r, h}^{m}\right)(x, a)-B_{\mathbb{P},\mathcal{E}} H^2 d_1\sqrt{d_1 W} -B_{r,\mathcal{E}}\sqrt{d_2 W}\\
&\leq \iota_{r, h}^{m}(x, a) \leq B_{\mathbb{P},\mathcal{E}} H^2 d_1\sqrt{d_1 W} +B_{r,\mathcal{E}}\sqrt{d_2 W}
\end{aligned}
$$
for all $(m, h) \in[M] \times[H]$ and $(x, a) \in \mathcal{S} \times \mathcal{A}$.
Similarly, it can be shown that 
$$
\begin{aligned}
&-2\left(\Gamma_{h}^{m}+\Gamma_{g, h}^{m}\right)(x, a)-B_{\mathbb{P},\mathcal{E}} H^2 d_1\sqrt{d_1 W} -B_{g,\mathcal{E}}\sqrt{d_2 W}\\
&\leq \iota_{g, h}^{m}(x, a) \leq B_{\mathbb{P},\mathcal{E}} H^2 d_1\sqrt{d_1 W} +B_{g,\mathcal{E}}\sqrt{d_2 W}.
\end{aligned}
$$
\end{proof}

\begin{lemma}\label{lemma: bounds on model prediction error with local knowledge}
Let   Assumptions \ref{ass: linear fun approx}  and \ref{ass: local budget} hold. Fix $p \in(0,1)$ and let $\mathcal{E}$ be the epoch that the episode $m$ belongs to.
If we set $\lambda=1$, $LV=B_{\mathbb{P},\mathcal{E}} H^2 d_1\sqrt{d_1 W} +B_{g,\mathcal{E}}\sqrt{d_2 W},$ and
\begin{align*}
&\Gamma_{h}^{m}(\cdot, \cdot)=\beta\left(\varphi(\cdot, \cdot)^{\top}\left(\Lambda_{h}^{m}\right)^{-1} \varphi(\cdot, \cdot)\right)^{1 / 2}, \\
&\Gamma_{r,h}^m =\beta \left( (\phi_{r,h}^m)^\top (\Lambda_{r, h}^m)^{-1} \phi_{r,h}^m \right)^{1/2}, \\
&\Gamma_{g,h}^m =\beta \left( (\phi_{g,h}^m)^\top (\Lambda_{g,h}^m)^{-1} \phi_{g,h}^m \right)^{1/2}
\end{align*}
with $\beta=$ $C_{1} \sqrt{d H^{2} \log (d W / p)}$ in Algorithm \ref{alg:algoirthm 2}, then it holds that
\begin{align*}
  &-2\left(\Gamma_{h}^{m}+\Gamma_{r, h}^{m}\right)(x, a)-B_{\mathbb{P},\mathcal{E}} H^2 d_1\sqrt{d_1 W} - B_{r,\mathcal{E}}\sqrt{d_2 W} \\
  &\leq \iota_{r, h}^{m}(x, a) \leq B_{\mathbb{P},\mathcal{E}} H^2 d_1\sqrt{d_1 W} + B_{r,\mathcal{E}}\sqrt{d_2 W}\\
  &-2\left(\Gamma_{h}^{m}+\Gamma_{g, h}^{m}\right)(x, a)-2 B_{\mathbb{P},\mathcal{E}} H^2 d_1\sqrt{d_1 W} -2 B_{g,\mathcal{E}}\sqrt{d_2 W} \\
  &\leq \iota_{g, h}^{m}(x, a) \leq 0
\end{align*}

with probability at least $1-p / 2$ for every $(m, h) \in[M] \times[H]$ and $(x, a) \in \mathcal{S} \times \mathcal{A}$.
\end{lemma}
\begin{proof}
The proof is similar to Lemma \ref{lemma: bounds on model prediction error} and thus is omitted.
\end{proof}

\begin{lemma}[Lemma D.1 in \cite{jin2020provably}] \label{lemma: lemma D1 in Jin2020}
Let $\Lambda_{t}=\lambda \mathbf{I}+\sum_{i=1}^{t} \phi_{i} \phi_{i}^{\top}$, where $\phi_{i} \in \mathbb{R}^{d}$ and $\lambda>0 .$ Then,
$$
\sum_{i=1}^{t} \phi_{i}^{\top}\left(\Lambda_{t}\right)^{-1} \phi_{i} \leq d.
$$
\end{lemma}

\begin{lemma}(\textbf{Elliptical Potential Lemma, Lemma D.2 in \cite{jin2020provably} or \cite{cai2020provably}})
\label{lemma: elliptical Potential Lemma}
Let $\left\{\phi_{t}\right\}_{t=1}^{\infty}$ be a sequence of functions in $\mathbb{R}^{d}$ and $\Lambda_{0} \in \mathbb{R}^{d \times d}$ be a positive definite matrix. Let $\Lambda_{t}=\Lambda_{0}+\sum_{i=1}^{t-1} \phi_{i} \phi_{i}^{\top} .$ Assume that $\left\|\phi_{t}\right\|_{2} \leq 1$ and $\lambda_{\min }\left(\Lambda_{0}\right) \geq 1 .$ For every $t \geq 1$, it holds that
$$
\log \left(\frac{\operatorname{det}\left(\Lambda_{t+1}\right)}{\operatorname{det}\left(\Lambda_{1}\right)}\right) \leq \sum_{i=1}^{t} \phi_{i}^{\top} \Lambda_{i}^{-1} \phi_{i} \leq 2 \log \left(\frac{\operatorname{det}\left(\Lambda_{t+1}\right)}{\operatorname{det}\left(\Lambda_{1}\right)}\right).
$$
\end{lemma}

\subsubsection{Tabular MDP case}
\begin{lemma} \label{lemma: bound Gamma diamond tabular}
If we set $C_4>1$, $\lambda=1$, $\beta=C_4 H \sqrt{|\mathcal{S}| \log (|\mathcal{S}||\mathcal{A}| W / p)}$ and $\Gamma_{h}^{m}=\beta\left(n_{h}^{m}(x, a)+\lambda\right)^{-1 / 2}$ in line 5 of Algorithm \ref{alg:algoirthm 3}, then  it holds that
$$
\begin{aligned}
&\left|\sum_{x^{\prime} \in \mathcal{S}}\left(\widehat{\mathbb{P}}_{h}^{m}\left(x^{\prime} \mid x, a\right) V\left(x^{\prime}\right)-\mathbb{P}_{h}^{m}\left(x^{\prime} \mid x, a\right) V\left(x^{\prime}\right)\right)\right|\\
&\leq \Gamma_{h}^{m}(x, a) + B_{\mathbb{P},\mathcal{E}} H
\end{aligned}
$$
with probability at least $1-p / 2$ for every $(m, h) \in[M] \times[H]$ and $(x, a) \in \mathcal{S} \times \mathcal{A}$.
\end{lemma}
\begin{proof}
Let $\mathcal{V}=\{V: \mathcal{S} \rightarrow[0, H]\}$ be a set of bounded functions on $\mathcal{S}$. For every $V \in \mathcal{V}$, we consider the difference between $\sum_{x^{\prime} \in \mathcal{S}} \widehat{\mathbb{P}}_{h}^{m}\left(x^{\prime} \mid \cdot, \cdot\right) V\left(x^{\prime}\right)$ and $\sum_{x^{\prime} \in \mathcal{S}} \mathbb{P}_{h}^{m}\left(x^{\prime} \mid \cdot, \cdot\right) V\left(x^{\prime}\right)$ as follows:
\begin{align} 
\nonumber&\left(n_{h}^{m}(x, a)+\lambda\right)^{1 / 2}\\
&\left|\sum_{x^{\prime} \in \mathcal{S}}\left(\widehat{\mathbb{P}}_{h}^{m}\left(x^{\prime} \mid x, a\right) V\left(x^{\prime}\right)-\mathbb{P}_{h}^{m}\left(x^{\prime} \mid x, a\right) V\left(x^{\prime}\right)\right)\right|  \label{eq: hat P V -P V tabular}\\
\nonumber =&\left(n_{h}^{m}(x, a)+\lambda\right)^{-1 / 2}\\
\nonumber &\left|\sum_{x^{\prime} \in \mathcal{S}} n_{h}^{m}\left(x, a, x^{\prime}\right) V\left(x^{\prime}\right)-\left(n_{h}^{m}(x, a)+\lambda\right)\left(\mathbb{P}_{h}^{m} V\right)(x, a)\right| \\
\nonumber \leq &\left(n_{h}^{m}(x, a)+\lambda\right)^{-1 / 2}\\
\nonumber&\left|\sum_{x^{\prime} \in \mathcal{S}} n_{h}^{m}\left(x, a, x^{\prime}\right) V\left(x^{\prime}\right)-n_{h}^{m}(x, a)\left(\mathbb{P}_{h}^{m} V\right)(x, a)\right| \\
\nonumber &+\left(n_{h}^{m}(x, a)+\lambda\right)^{-1 / 2}\left|\lambda\left(\mathbb{P}_{h}^{m} V\right)(x, a)\right| \\
\nonumber =&\left(n_{h}^{m}(x, a)+\lambda\right)^{-1 / 2}\\
&\left|\sum_{\tau=\ell^m_Q}^{m-1} 1\left\{(x, a)=\left(x_{h}^{\tau}, a_{h}^{\tau}\right)\right\}\left(V\left(x_{h+1}^{\tau}\right)-\left(\mathbb{P}_{h}^{m} V\right)(x, a)\right)\right| \\
\nonumber &+\left(n_{h}^{m}(x, a)+\lambda\right)^{-1 / 2}\left|\lambda\left(\mathbb{P}_{h}^{m} V\right)(x, a)\right|\\
\nonumber\leq&\left(n_{h}^{m}(x, a)+\lambda\right)^{-1 / 2}\\
&\left|\sum_{\tau=\ell^m_Q}^{m-1} 1\left\{(x, a)=\left(x_{h}^{\tau}, a_{h}^{\tau}\right)\right\}\left(V\left(x_{h+1}^{\tau}\right)-\left(\mathbb{P}_{h}^{\tau} V\right)(x, a)\right)\right| \label{eq: hat P V -P V tabular 1}\\ 
&+\left(n_{h}^{m}(x, a)+\lambda\right)^{-1 / 2} \label{eq: hat P V -P V tabular 2} \\
\nonumber&\left|\sum_{\tau=\ell^m_Q}^{m-1} 1\left\{(x, a)=\left(x_{h}^{\tau}, a_{h}^{\tau}\right)\right\}\left(\left(\mathbb{P}_{h}^{\tau} V\right)(x, a)-\left(\mathbb{P}_{h}^{m} V\right)(x, a)\right)\right| \\
&+\left(n_{h}^{m}(x, a)+\lambda\right)^{-1 / 2}\left|\lambda\left(\mathbb{P}_{h}^{m} V\right)(x, a)\right| \label{eq: hat P V -P V tabular 3}
\end{align}
for every $(m, h) \in[M] \times[H]$ and $(x, a) \in \mathcal{S} \times \mathcal{A}$.

To analyze the term in \eqref{eq: hat P V -P V tabular 1}, we let $\eta_{h}^{\tau}:=V\left(x_{h+1}^{\tau}\right)-\left(\mathbb{P}_{h}^{\tau} V\right)\left(x_{h}^{\tau}, a_{h}^{\tau}\right) .$ Conditioning on the filtration $\mathcal{F}_{h, 1}^{m}$, the term $\eta_{h}^{\tau}$ is a zero-mean and $H / 2$-sub-Gaussian random variable. By Lemma \ref{lemma: Concentration of Self-normalized Processes}, we use $Y=\lambda I$ and $X_{\tau}=1\left\{(x, a)=\left(x_{h}^{\tau}, a_{h}^{\tau}\right)\right\}$ and thus with probability at least $1-\delta$ it holds that
$$
\begin{aligned}
&\left(n_{h}^{m}(x, a)+\lambda\right)^{-1 / 2}\\
&\left|\sum_{\tau=\ell^m_Q}^{m-1} 1\left\{(x, a)=\left(x_{h}^{\tau}, a_{h}^{\tau}\right)\right\}\left(V\left(x_{h+1}^{\tau}\right)-\left(\mathbb{P}_{h}^{\tau} V\right)(x, a)\right)\right| \\
&\leq \sqrt{\frac{H^{2}}{2} \log \left(\frac{\left(n_{h}^{m}(x, a)+\lambda\right)^{1 / 2} \lambda^{-1 / 2}}{\delta / H}\right)} \\
&\leq \sqrt{\frac{H^{2}}{2} \log \left(\frac{W}{\delta}\right)}
\end{aligned}
$$
for every $(m, h) \in[M] \times[H] .$ 

For the term in \eqref{eq: hat P V -P V tabular 2}, by the definition of $B_{\mathbb{P},\mathcal{E}}$ and $n_h^m$, we have
\begin{align*}
&\left(n_{h}^{m}(x, a)+\lambda\right)^{-1 / 2}\\
&\left|\sum_{\tau=\ell^m_Q}^{m-1} 1\left\{(x, a)=\left(x_{h}^{\tau}, a_{h}^{\tau}\right)\right\}\left(\left(\mathbb{P}_{h}^{\tau} V\right)(x, a)-\left(\mathbb{P}_{h}^{m} V\right)(x, a)\right)\right| \\
\leq & \left|\left(n_{h}^{m}(x, a)+\lambda\right)^{-1 / 2} \sum_{\tau=\ell^m_Q}^{m-1} 1\left\{(x, a)=\left(x_{h}^{\tau}, a_{h}^{\tau}\right)\right\}\right|  B_{\mathbb{P},\mathcal{E}} H \\
\leq & \left(n_{h}^{m}(x, a)+\lambda\right)^{1 / 2}  B_{\mathbb{P},\mathcal{E}} H.
\end{align*}

For the term in \eqref{eq: hat P V -P V tabular 3}, since $0 \leq V \leq H$, we have
$$
\left(n_{h}^{m}(x, a)+\lambda\right)^{-1 / 2}\left|\lambda\left(\mathbb{P}_{h}^{m} V\right)(x, a)\right| \leq \sqrt{\lambda} H .
$$

By returning to \eqref{eq: hat P V -P V tabular} and setting $\lambda=1$, with probability at least $1-\delta$ it holds that
\begin{align*}
&\left(n_{h}^{m}(x, a)+\lambda\right)^{\frac{1}{2}}\\
&\left|\sum_{x^{\prime} \in \mathcal{S}}\left(\widehat{\mathbb{P}}_{h}^{m}\left(x^{\prime} \mid x, a\right) V\left(x^{\prime}\right)-\mathbb{P}_{h}^{m}\left(x^{\prime} \mid x, a\right) V\left(x^{\prime}\right)\right)\right| \\
\leq &\sqrt{H^{2}\left(\log \left(\frac{W}{\delta}\right)+2\right)} + \left(n_{h}^{m}(x, a)+\lambda\right)^{1 / 2}  B_{\mathbb{P},\mathcal{E}} H
\end{align*}
for all $m \geq 1$.
Let $d\left(V, V^{\prime}\right)=\max _{x \in \mathcal{S}}\left|V(x)-V^{\prime}(x)\right|$ be a distance on $\mathcal{V} .$ For every $\epsilon$, an $\epsilon$-covering $\mathcal{V}_{\epsilon}$ of $\mathcal{V}$ with respect to distance $d(\cdot, \cdot)$ satisfies
$$
\left|\mathcal{V}_{\epsilon}\right| \leq\left(1+\frac{2 \sqrt{|\mathcal{S}|} H}{\epsilon}\right)^{|\mathcal{S}|}
$$
Thus, for every $V \in \mathcal{V}$, there exists $V^{\prime} \in \mathcal{V}_{\epsilon}$ such that $\max _{x \in \mathcal{S}}\left|V(x)-V^{\prime}(x)\right| \leq \epsilon$. By the triangle inequality, we have
$$
\begin{aligned}
&\left(n_{h}^{m}(x, a)+\lambda\right)^{1 / 2}\\
&\left|\sum_{x^{\prime} \in \mathcal{S}}\left(\widehat{\mathbb{P}}_{h}^{m}\left(x^{\prime} \mid x, a\right) V\left(x^{\prime}\right)-\mathbb{P}_{h}^{m}\left(x^{\prime} \mid x, a\right) V\left(x^{\prime}\right)\right)\right| \\
=&\left(n_{h}^{m}(x, a)+\lambda\right)^{1 / 2}\\
&\left|\sum_{x^{\prime} \in \mathcal{S}}\left(\widehat{\mathbb{P}}_{h}^{m}\left(x^{\prime} \mid x, a\right) V^{\prime}\left(x^{\prime}\right)-\mathbb{P}_{h}^{m}\left(x^{\prime} \mid x, a\right) V^{\prime}\left(x^{\prime}\right)\right)\right| \\
&+\left(n_{h}^{m}(x, a)+\lambda\right)^{1 / 2}\\
&\left|\sum_{x^{\prime} \in \mathcal{S}}\left(\widehat{\mathbb{P}}_{h}^{m}\left(x^{\prime} \mid x, a\right)\left(V\left(x^{\prime}\right)-V^{\prime}\left(x^{\prime}\right)\right)\right.\right.\\
&\left.\left.-\mathbb{P}_{h}^{m}\left(x^{\prime} \mid x, a\right)\left(V\left(x^{\prime}\right)-V^{\prime}\left(x^{\prime}\right)\right)\right)\right| \\
\leq &\left(n_{h}^{m}(x, a)+\lambda\right)^{1 / 2}\\
&\left|\sum_{x^{\prime} \in \mathcal{S}}\left(\widehat{\mathbb{P}}_{h}^{m}\left(x^{\prime} \mid x, a\right) V^{\prime}\left(x^{\prime}\right)-\mathbb{P}_{h}^{m}\left(x^{\prime} \mid x, a\right) V^{\prime}\left(x^{\prime}\right)\right)\right| \\
&+2\left(n_{h}^{m}(x, a)+\lambda\right)^{-1 / 2} \epsilon
\end{aligned}
$$
Furthermore, we choose $\delta=(p / 3) /\left(\left|\mathcal{V}_{\epsilon}\|\mathcal{S}\| \mathcal{A}\right|\right)$ and take a union bound over $V \in \mathcal{V}_{\epsilon}$ and $(x, a) \in \mathcal{S} \times \mathcal{A}$. By $(48)$, with probability at least $1-p / 2$ it holds that
$$
\begin{aligned}
&\sup _{V \in \mathcal{V}}\left\{\left(n_{h}^{m}(x, a)+\lambda\right)^{1 / 2} \right.\\
&\left.\left|\sum_{x^{\prime} \in \mathcal{S}}\left(\widehat{\mathbb{P}}_{h}^{m}\left(x^{\prime} \mid x, a\right) V\left(x^{\prime}\right)-\mathbb{P}_{h}^{m}\left(x^{\prime} \mid x, a\right) V\left(x^{\prime}\right)\right)\right|\right\} \\
&\leq \sqrt{H^{2}\left(\log \left(\frac{W}{\delta}\right)+2\right)} + \left(n_{h}^{m}(x, a)+\lambda\right)^{1 / 2}  B_{\mathbb{P},\mathcal{E}} H\\
&+2\left(n_{h}^{m}(x, a)+\lambda\right)^{-1 / 2} \frac{H}{K} \\
&\leq \sqrt{2 H^{2}\left(\log \left|\mathcal{V}_{\epsilon}\right|+\log \left(\frac{2|\mathcal{S}||\mathcal{A}| W}{p}\right)+2\right)}\\
&+ \left(n_{h}^{m}(x, a)+\lambda\right)^{1 / 2}  B_{\mathbb{P},\mathcal{E}} H+2\left(n_{h}^{m}(x, a)+\lambda\right)^{-1 / 2} \frac{H}{K} \\
&\leq C_4 H \sqrt{|\mathcal{S}| \log \left(\frac{|\mathcal{S}||\mathcal{A}| W}{p}\right)}+ \left(n_{h}^{m}(x, a)+\lambda\right)^{1 / 2}  B_{\mathbb{P},\mathcal{E}} H
\end{aligned}
$$
for every $(m, h)$ and $(x, a)$, where $C_4$ is an absolute constant. We recall our choice of $\Gamma_{h}^{m}$ and $\beta$. Hence, with probability at least $1-p / 2$ it holds that
$$
\begin{aligned}
&\left|\sum_{x^{\prime} \in \mathcal{S}}\left(\widehat{\mathbb{P}}_{h}^{m}\left(x^{\prime} \mid x, a\right) V\left(x^{\prime}\right)-\mathbb{P}_{h}^{m}\left(x^{\prime} \mid x, a\right) V\left(x^{\prime}\right)\right)\right| \\
&\leq \beta\left(n_{h}^{m}(x, a)+\lambda\right)^{-1 / 2}+  B_{\mathbb{P},\mathcal{E}} H
\end{aligned}
$$
for any $(m, h) \in[M] \times[H]$ and $(x, a) \in|\mathcal{S}| \times|\mathcal{A}|$, where $\beta:=C_{3} H \sqrt{|\mathcal{S}| \log (|\mathcal{S}||\mathcal{A}| W / p)}$.
\end{proof}

\begin{lemma}\label{lemma: bound Gamma tabular}
If we set $C_4>1$, $\lambda=1$ and $\beta:=C_{4} H \sqrt{|\mathcal{S}| \log (|\mathcal{S}||\mathcal{A}| W / p)}$ and $\Gamma_{h}^{m}=\beta\left(n_{h}^{m}(x, a)+\lambda\right)^{-1 / 2}$  in line 5 of Algorithm \ref{alg:algoirthm 3},  then it holds that
$$
\left|\widehat{\diamond}_{h}^{m}(x, a)-\diamond_{h}^{m}(x, a)\right| \leq \Gamma_{h}^{m}(x, a) +B_{\diamond,\mathcal{E}}.
$$
with probability at least $1-p / 2$ for every $(m, h) \in[M] \times[H]$ and $(x, a) \in \mathcal{S} \times \mathcal{A}$, where the symbol $\diamond$ is equal to $r$ or $g$ and $\widehat{\diamond}_{h}^{m}(x, a)$ is defined in  \eqref{eq: hat diamond}. 
\end{lemma}

\begin{proof}
We recall the definition $r_{h}^m(x, a)=\mathbf{e}_{(x, a)}^{\top} \theta^m_{r, h}$. By our estimation $\widehat{r}_{h}^{m}(x, a)$ in Algorithm 3 , we have
$$
\widehat{r}_{h}^{m}(x, a)=\frac{1}{n_{h}^{m}(x, a)+\lambda} \sum_{\tau=\ell^m_Q}^{m-1} 1\left\{(x, a)=\left(x_{h}^{\tau}, a_{h}^{\tau}\right)\right\}\left[\theta^\tau_{r, h}\right]_{\left(x_{h}^{\tau}, a_{h}^{\tau}\right)}
$$
and thus
$$
\begin{aligned}
&\left|\widehat{r}_{h}^{m}(x, a)-r^m_{h}(x, a)\right| \\
&=\left|\widehat{r}_{h}^{m}(x, a)-\left[\theta^m_{r, h}\right]_{(x, a)}\right| \\
&=\left(n_{h}^{m}(x, a)+\lambda\right)^{-1}\left|\sum_{\tau=\ell^m_Q}^{m-1} 1\left\{(x, a)=\left(x_{h}^{\tau}, a_{h}^{\tau}\right)\right\} \right.\\
&\hspace{1.5cm}\left.\left(\left[\theta^\tau_{r, h}\right]_{\left(x_{h}^{\tau}, a_{h}^{\tau}\right)}-\left[\theta^m_{r, h}\right]_{(x, a)}\right)-\lambda\left[\theta^m_{r, h}\right]_{(x, a)}\right| \\
&\leq B_{r,\mathcal{E}}+\left(n_{h}^{m}(x, a)+\lambda\right)^{-1}\left|\lambda\left[\theta^m_{r, h}\right]_{(x, a)}\right| \\
&\leq B_{r,\mathcal{E}}+\left(n_{h}^{m}(x, a)+\lambda\right)^{-1} \lambda \\
&\leq B_{r,\mathcal{E}}+\left(n_{h}^{m}(x, a)+\lambda\right)^{-1 / 2} \lambda \\
&\leq B_{r,\mathcal{E}}+\Gamma_{h}^{m}(x, a)
\end{aligned}
$$
where we utilize $\lambda=1$ in the last inequality. This completes the proof.
\end{proof}

\begin{lemma}\label{lemma: bounds on model prediction error tabular}
Fix $p \in(0,1)$ and let $\mathcal{E}$ be the epoch that the episode $m$ belongs to.
If we set $C_4>1$, $\lambda=1$, $LV=0$ and
\begin{align*}
\Gamma_{h}^m =\beta\left(n_{h}^m(x, a)+\lambda\right)^{-1 / 2},
\end{align*}
with $\beta:=C_{3} H \sqrt{|\mathcal{S}| \log (|\mathcal{S}||\mathcal{A}| W / p)}$ in Algorithm \ref{alg:algoirthm 3}, then  we have
$$
-4\Gamma_{h}^{m}(x, a)-B_{\mathbb{P},\mathcal{E}} H -B_{\diamond,\mathcal{E}}\leq \iota_{\diamond, h}^{m}(x, a) \leq B_{\mathbb{P},\mathcal{E}} H  +B_{\diamond,\mathcal{E}}
$$
with probability at least $1-p / 2$ for every $(m, h) \in[M] \times[H]$ and $(x, a) \in \mathcal{S} \times \mathcal{A}$, where the symbol $\diamond$ is equal to $r$ or $g$. 
\end{lemma}
\begin{proof}
The proof is similar to Lemma \ref{lemma: bounds on model prediction error} and thus is omitted.
\end{proof}

\begin{lemma}\label{lemma: bounds on model prediction error tabular under local budget}
Fix $p \in(0,1)$ and let $\mathcal{E}$ be the epoch that the episode $m$ belongs to.
If we set $C_4>1$, $\lambda=1$, $LV=B_{\mathbb{P},\mathcal{E}} H  +B_{g,\mathcal{E}}$ and
\begin{align*}
\Gamma_{h}^m =\beta\left(n_{h}^m(x, a)+\lambda\right)^{-1 / 2},
\end{align*}
with $\beta:=C_{4} H \sqrt{|\mathcal{S}| \log (|\mathcal{S}||\mathcal{A}| W / p)}$ in Algorithm \ref{alg:algoirthm 3}, then  we have
\begin{align*}
&-4\Gamma_{h}^{m}(x, a)-B_{\mathbb{P},\mathcal{E}} H -B_{r,\mathcal{E}}\leq \iota_{r, h}^{m}(x, a) \leq B_{\mathbb{P},\mathcal{E}} H +B_{r,\mathcal{E}}\\    
&-4\Gamma_{h}^{m}(x, a)-2B_{\mathbb{P},\mathcal{E}} H -2B_{g,\mathcal{E}}\leq \iota_{g, h}^{m}(x, a) \leq 0 
\end{align*}
with probability at least $1-p / 2$ for every $(m, h) \in[M] \times[H]$ and $(x, a) \in \mathcal{S} \times \mathcal{A}$. 
\end{lemma}
\begin{proof}
The proof is similar to Lemma \ref{lemma: bounds on model prediction error} and thus is omitted.
\end{proof}

%% file: files/appe_other.tex
\section{Auxiliary lemmas}\label{sec:appe-useful lemmas}

\subsection{Performance difference lemmas}
We first introduce a ``one-step descent'' result.

\begin{lemma}(\textbf{One-step descent lemma, Lemma 3.3 in \cite{cai2020provably}})
\label{lemma: One-step descent lemma}
For every two distributions $\pi^\star$ and $\pi$ supported on $\mathcal{A}$, state $s\in\mathcal{S}$ and function $Q: \mathcal{S}\times\mathcal{A} \rightarrow [0,H]$, it holds that for a distribution $\pi^\prime$ supported on $\mathcal{A}$ with $\pi^\prime (\cdot) \propto \pi (\cdot) \cdot \exp\{\alpha Q(s,\cdot)\}$ we have
\begin{align*}
&\inner{Q(s,\cdot)}{\pi^\star(\cdot) - \pi(\cdot)}\\
&\leq \frac{1}{2}\alpha H^2 +\frac{1}{\alpha} \left[D(\pi^\star (\cdot)\mid  \pi(\cdot)) -D(\pi^\star (\cdot)\mid \pi^\prime (\cdot) ) \right].
\end{align*}
\end{lemma}

We then introduce a variation of the performance difference lemma with the model prediction error.

\begin{lemma}(\textbf{Performance difference lemma with model prediction error}) \label{lemma: expansion of V star m -V m}
For $\diamond =r$ or $g$, it holds that
\begin{align*}
& V_{\diamond,1}^{\pi^{\star,m}}(x_1)-V_{\diamond,1}^{m}(x_1)\\
 =&\sum_{h=1}^{H} \mathbb{E}_{\pi^{\star,m},\mathbb{P}^m}\left[\left\langle Q_{\diamond, h}^{m}\left(x_{h}, \cdot\right), \pi_{h}^{\star,m}\left(\cdot \mid x_{h}\right)-\pi_{h}^{m}\left(\cdot \mid x_{h}\right)\right\rangle\right]\\
 &+ \sum_{h=1}^{H} \mathbb{E}_{\pi^{\star,m},\mathbb{P}^m}\left[\iota_{\diamond, h}^{m}\left(x_{h}, a_{h}\right)\right].
\end{align*}
\end{lemma}
\begin{proof}
For every $(h,m) \in[H]\times[M]$, we recall the definition of $V_{r,h}^{\pi^{\star,m},m}$ in the Bellman equation \eqref{eq: belmman equation} and the definition of $V_{r,h}^m$:
\begin{align*}
    &V_{r,h}^{\pi^{\star,m},m}(x)=\inner{Q_{r,h}^{\pi^{\star,m},m} (x,\cdot)}{\pi_h^{\star,m}(\cdot\mid x)}, \\
    &V_{r,h}^{m}(x)=\inner{Q_{r,h}^{m} (x,\cdot)}{\pi_h^m(\cdot\mid x)}.
\end{align*}
We can expand the difference $V_{r,h}^{\pi^{\star,m},m}(x)-V_{r,h}^{m}(x)$ as
\begin{align} \label{eq: V*-Vm}
&\nonumber V_{r,h}^{\pi^{\star,m},m}(x)-V_{r,h}^{m}(x)\\
=& \inner{Q_{r,h}^{\pi^{\star,m},m} (x,\cdot)}{\pi_h^{\star,m}(\cdot\mid x)} -\inner{Q_{r,h}^{m} (x,\cdot)}{\pi_h^m(\cdot\mid x)}\\
\nonumber=& \inner{Q_{r,h}^{\pi^{\star,m},m} (x,\cdot)-Q_{r,h}^{m} (x,\cdot)}{\pi_h^{\star,m}(\cdot\mid x)} \\
\nonumber &+ \inner{Q_{r,h}^{m} (x,\cdot)}{\pi_h^{\star,m}(\cdot\mid x)-\pi_h^m(\cdot\mid x)}\\
 =& \inner{Q_{r,h}^{\pi^{\star,m},m} (x,\cdot)-Q_{r,h}^{m} (x,\cdot)}{\pi_h^{\star,m}(\cdot\mid x)} + \xi_h^m(x),
\end{align}
where $\xi_h^m(x)\coloneqq \inner{Q_{r,h}^{m} (x,\cdot)}{\pi_h^{\star,m}(\cdot\mid x)-\pi_h^m(\cdot\mid x)}$. As a result of the Bellman equation \eqref{eq: belmman equation} and the definition of the model prediction error, we have
\begin{align*}
&Q_{r,h}^{\pi^{\star,m},m}(x,a)=r_h^m(x,a) + \mathbb{P}_h^m V_{r,h+1}^{\pi^{\star,m},m}(x,a), \\ &\iota_{r,h}^m=r_h^m+\mathbb{P}_h^m V_{r,h+1}^{m}-Q_{r,h}^{m}.
\end{align*}
As a result, we have
\begin{align} \label{eq: Q*-Qm}
Q_{r,h}^{\pi^{\star,m},m}-Q_{r,h}^{m}=\iota_{r, h}^{m}+ \mathbb{P}_h^m\left( V_{r,h+1}^{\pi^{\star,m},m}-V_{r,h+1}^m \right).
\end{align}

Substituting \eqref{eq: Q*-Qm} into the right-hand side of \eqref{eq: V*-Vm} yields that
\begin{align*}
 &V_{r,h}^{\pi^{\star,m},m}(x)-V_{r,h}^{m}(x)\\
 =&\inner{\mathbb{P}_h^m\left( V_{r,h+1}^{\pi^{\star,m},m}-V_{r,h+1}^m \right)(x,\cdot)}{\pi_h^{\star,m}(\cdot\mid x)}\\
 &+ \inner{\iota_{r, h}^{m}(x,\cdot)}{\pi_h^{\star,m}(\cdot\mid x)} + \xi_h^m(x).
\end{align*}
Using the above formula and expanding $V_{r,1}^{\pi^{\star,m}}(x)-V_{r,1}^{m}$ recursively at $x_1$, one can obtain
\begin{align} \label{eq: recursive formula 1}
\nonumber &V_{r,1}^{\pi^{\star,m}}(x_1)-V_{r,1}^{m}(x_1)\\
\nonumber=&\inner{\mathbb{P}_1^m\left( V_{r,2}^{\pi^{\star,m},m}-V_{r,2}^m \right)(x_1,\cdot)}{\pi_1^{\star,m}(\cdot\mid x_1)}\\
\nonumber&+ \inner{\iota_{r, 1}^{m}(x_1,\cdot)}{\pi_1^{\star,m}(\cdot\mid x_1)} + \xi_1^m(x_1)\\
\nonumber =&\inner{\mathbb{P}_1^m \inner{\mathbb{P}_2^m\left( V_{r,3}^{\pi^{\star,m},m}-V_{r,3}^m \right) (x_2,\cdot) }{\pi_2^{\star,m}(\cdot\mid x_2)}}{\pi_1^{\star,m}(\cdot\mid x_1) }\\
\nonumber & + \inner{\iota_{r, 1}^{m}(x_1,\cdot)}{\pi_1^{\star,m}(\cdot\mid x_1)} \\
\nonumber & + \inner{\mathbb{P}_1^m \inner{\iota_{r, 2}^{m}(x_2,\cdot) }{\pi_2^{\star,m}(\cdot\mid x_2)}}{\pi_1^{\star,m}(\cdot\mid x_1) } \\
&+ \inner{\mathbb{P}_1^m\xi_2^m (x_1,\cdot)}{\pi_1^{\star,m}(\cdot\mid x_1)}+ \xi_1^m(x_1).
\end{align}
For notational simplicity, for every $(m, h) \in[M] \times[H]$, we define an operator $\widetilde{\mathcal{I}}_{h}^{\star,m}$ for function $f: \mathcal{S} \times \mathcal{A} \rightarrow \mathbb{R}$:
$$\left(\widetilde{\mathcal{I}}_{h}^{\star,m} f\right)(x)=\left\langle f(x, \cdot), \pi_{h}^{\star,m}(\cdot \mid x)\right\rangle .$$

Repeating the recursion \eqref{eq: recursive formula 1} over $h\in[H]$ yields that
\begin{align*}
 &V_{r,1}^{\pi^{\star,m}}(x_1)-V_{r,1}^{m}(x_1)\\
 =&  \widetilde{\mathcal{I}}_1^{\star,m} \mathbb{P}_1^m \widetilde{\mathcal{I}}_2^{\star,m} \mathbb{P}_2^m \left(V_{r,3}^{\pi^{\star,m},m}-V_{r,3}^m\right) \\
 &+  \widetilde{\mathcal{I}}_1^{\star,m} \mathbb{P}_1^m \widetilde{\mathcal{I}}_2^{\star,m} \iota_{r,2}^m + \widetilde{\mathcal{I}}_1^{\star,m}\iota_{r,1}^m+ \widetilde{\mathcal{I}}_1^{\star,m} \mathbb{P}_1^m\xi_{2}^m + \xi_1^m\\
=& \widetilde{\mathcal{I}}_1^{\star,m} \mathbb{P}_1^m \widetilde{\mathcal{I}}_2^{\star,m} \mathbb{P}_2^m \widetilde{\mathcal{I}}_3^{\star,m} \mathbb{P}_3^m \left(V_{r,4}^{\pi^{\star,m},m}-V_{r,4}^m\right) \\
&+ \widetilde{\mathcal{I}}_1^{\star,m} \mathbb{P}_1^m \widetilde{\mathcal{I}}_2^{\star,m} \iota_{r,2}^m + \widetilde{\mathcal{I}}_1^{\star,m} \mathbb{P}_1^m \widetilde{\mathcal{I}}_2^{\star,m} \mathbb{P}_2^m \widetilde{\mathcal{I}}_3^{\star,m} \iota_{r,3}^m\\
&+ \widetilde{\mathcal{I}}_1^{\star,m}\iota_{r,1}^m+ \widetilde{\mathcal{I}}_1^{\star,m}\mathbb{P}_1^m \widetilde{\mathcal{I}}_2^{\star,m} \mathbb{P}_2^m\xi_{3}^m + \widetilde{\mathcal{I}}_1^{\star,m} \mathbb{P}_1^m\xi_{2}^m + \xi_1^m\\
&\ldots\\
=& \left( \prod_{h=1}^H \widetilde{\mathcal{I}}_h^{\star,m} \mathbb{P}^m_h\right) \left(V_{r,H+1}^{\pi^{\star,m},m}-V_{r,H+1}^m \right) \\
& + \sum_{h=1}^H \left(\prod_{i=1}^{h-1} \widetilde{\mathcal{I}}_i^{\star,m}\mathbb{P}_i^m \right) \widetilde{\mathcal{I}}_h^{\star,m} \iota_{r,h}^m +\sum_{h=1}^H \left(\prod_{h=1}^{h-1} \widetilde{\mathcal{I}}_i^{\star,m} \mathbb{P}_i^m \right) \xi_h^m\\
=& \sum_{h=1}^H \left(\prod_{i=1}^{h-1} \widetilde{\mathcal{I}}_i^{\star,m}\mathbb{P}_i^m \right) \widetilde{\mathcal{I}}_h^{\star,m} \iota_{r,h}^m +\sum_{h=1}^H \left(\prod_{h=1}^{h-1} \widetilde{\mathcal{I}}_i^{\star,m} \mathbb{P}_i^m \right) \xi_h^m.
\end{align*}
where the last inequality is due to $V_{r,H+1}^{\pi^{\star,m},m}-V_{r,H+1}^m =0$. Finally, the proof is completed by the definitions of $\mathbb{P}_h^m$ and $\widetilde{\mathcal{I}}_{h}^{\star,m}$. Similarly, we can also use the same argument to expand $V_{g,1}^{\pi^{\star,m},m}(x_1)-V_{g,1}^{m}(x_1)$.
\end{proof}

\subsection{Smoothness property for the visitation measure}
We recall the operator 
$\left(\widetilde{\mathcal{I}}_{h}^{\star,m} f\right)(x)=\left\langle f(x, \cdot), \pi_{h}^{\star,m}(\cdot \mid x)\right\rangle$ and note that $\left(\widetilde{\mathcal{I}}_{h}^{\star,m} \mathbb{P}_h^m\right)(x^\prime\mid x)= \sum_{a \in \mathcal{A}} \mathbb{P}_h^m(x^\prime \mid x,a) \pi_h(a\mid x) $
is the transition kernel in step $h$ under policy $\pi$ at the episode $m$. We fix $h\in[H]$. Under policies $\{\pi_h^m\}_{h=1}^H$, the distribution of $x_h$ conditional on $x_1$ is given by 
\begin{align*}
&\widetilde{\mathcal{I}}_1^{\star,m} \mathbb{P}_1^m \widetilde{\mathcal{I}}_2^{\star,m} \mathbb{P}_2^m \cdots \widetilde{\mathcal{I}}_{h-1}^{\star,m} \mathbb{P}_{h-1}^m (x_h\mid x_1)\\
\coloneqq & \sum_{x_2,\ldots,x_{h-1}} \prod_{i\in [h-1]} \left(\widetilde{\mathcal{I}}_i^{\star,m} \mathbb{P}_i^m \right)(x_{i+1}\mid x_i).
\end{align*}

We have the following smoothness property for the visitation measure $\widetilde{\mathcal{I}}_1^{\star,m} \mathbb{P}_1^m \widetilde{\mathcal{I}}_2^{\star,m} \mathbb{P}_2^m \cdots \widetilde{\mathcal{I}}_{h-1}^{\star,m} \mathbb{P}_{h-1}^m (x_h\mid x_1)$.

\begin{lemma} \label{lemma: prob difference in pi}
Under Assumption \eqref{ass: linear fun approx}, for every $h \in[H], j \in[h-1], x_{h}, x_{1} \in \mathcal{S}$, and policies $\left\{\pi^m_{i}\right\}_{i \in[H]} \cup\left\{\pi^{m-1}_{j}\right\}$, we have
\begin{align*}
& \left|\widetilde{\mathcal{I}}_1^{\star,m} \mathbb{P}_1^m \cdots \widetilde{\mathcal{I}}_j^{\star,m} \mathbb{P}_j^m \cdots \widetilde{\mathcal{I}}_{h-1}^{\star,m} \mathbb{P}_{h-1}^m (x_h\mid x_1)  \right.\\
& \left.- \widetilde{\mathcal{I}}_1^{\star,m} \mathbb{P}_1^m \cdots \widetilde{\mathcal{I}}_j^{\star,m-1} \mathbb{P}_j^{m} \cdots \widetilde{\mathcal{I}}_{h-1}^{\star,m} \mathbb{P}_{h-1}^m (x_h\mid x_1)\right| \\
\leq &\max_{x\in\mathcal{S}}\norm{\pi_j^{\star,m}(\cdot \mid x) -\pi_j^{\star,m-1}(\cdot \mid x)}_1,\\
& \left|\widetilde{\mathcal{I}}_1^{\star,m} \mathbb{P}_1^m \cdots \widetilde{\mathcal{I}}_j^{\star,m} \mathbb{P}_j^m \cdots \widetilde{\mathcal{I}}_{h-1}^{\star,m} \mathbb{P}_{h-1}^m(x_h\mid x_1) \right. \\
&\left.- \widetilde{\mathcal{I}}_1^{\star,m} \mathbb{P}_1^m \cdots \widetilde{\mathcal{I}}_j^{\star,m} \mathbb{P}_j^{m-1} \cdots \widetilde{\mathcal{I}}_{h-1}^{\star,m} \mathbb{P}_{h-1}^m (x_h\mid x_1) \right| \\
\leq& \sqrt{d_1} \norm{\theta_j^m-\theta_j^{m-1}}_2 
\end{align*}
where $d_1>0$ is the constant defined in Assumption \ref{ass: linear fun approx}.
\end{lemma}
\begin{proof}
From Holder's inequality, we know
\begin{align} \label{eq: prob difference in pi}
 \nonumber  & \left|\widetilde{\mathcal{I}}_1^{\star,m} \mathbb{P}_1^m \cdots \widetilde{\mathcal{I}}_j^{\star,m} \mathbb{P}_j^m \cdots \widetilde{\mathcal{I}}_{h-1}^{\star,m} \mathbb{P}_{h-1}^m (x_h\mid x_1) \right.\\
 \nonumber &\left.- \widetilde{\mathcal{I}}_1^{\star,m} \mathbb{P}_1^m \cdots \widetilde{\mathcal{I}}_j^{\star,m-1} \mathbb{P}_j^{m} \cdots \widetilde{\mathcal{I}}_{h-1}^{\star,m} \mathbb{P}_{h-1}^m (x_h\mid x_1) \right| \\
 \nonumber    \leq & \sum_{x_{2}, x_{3}, \ldots, x_{h-1}} \left|\left(\widetilde{\mathcal{I}}_j^{\star,m} \mathbb{P}_j^m \right)\left(x_{j+1} \mid x_{j}\right) -\left(\widetilde{\mathcal{I}}_j^{\star,m-1} \mathbb{P}_j^m \right) \left(x_{j+1} \mid x_{j}\right)\right| \\
 \nonumber & \cdot \prod_{i \in[h-1] \backslash\{j\}} \widetilde{\mathcal{I}}_i^{\star,m} \mathbb{P}_i^m \left(x_{i+1} \mid x_{i}\right) \\
 \nonumber    \leq & \sum_{x_{2}, \ldots, x_{j},x_{j+2},\ldots, x_{h-1}} \max_{x_{j+1}\in\mathcal{S}} \prod_{i \in[h-1] \backslash\{j\}} \widetilde{\mathcal{I}}_i^{\star,m} \mathbb{P}_i^m \left(x_{i+1} \mid x_{i}\right) \\
 \nonumber &\cdot \sum_{x_{j+1}} \left|\left(\widetilde{\mathcal{I}}_j^{\star,m} \mathbb{P}_j^m \right)\left(x_{j+1} \mid x_{j}\right) -\left(\widetilde{\mathcal{I}}_j^{\star,m-1} \mathbb{P}_j^m \right) \left(x_{j+1} \mid x_{j}\right)\right|  \\
\leq & \sum_{x_{2}, \ldots, x_{j},x_{j+2},\ldots, x_{h-1}} \max_{x_{j+1}\in\mathcal{S}} \prod_{i \in[h-1] \backslash\{j\}} \widetilde{\mathcal{I}}_i^{\star,m} \mathbb{P}_i^m \left(x_{i+1} \mid x_{i}\right) \\
 \nonumber & \cdot \max_{x_j \in \mathcal{S}} \sum_{x_{j+1}} \left|\left(\widetilde{\mathcal{I}}_j^{\star,m} \mathbb{P}_j^m \right)\left(x_{j+1} \mid x_{j}\right) -\left(\widetilde{\mathcal{I}}_j^{\star,m-1} \mathbb{P}_j^m \right) \left(x_{j+1} \mid x_{j}\right)\right|.
\end{align}

By the definition of $ \mathbb{P}_j^m$ and the boundedness of $\psi(x_j,\cdot,x_{j+1})$ and $\theta_j^m$ in Assumption \ref{ass: linear fun approx}, for every $x_{j+1} \in \mathcal{S}$, it holds that
\begin{align*}
&\left|\left(\widetilde{\mathcal{I}}_j^{\star,m} \mathbb{P}_j^m \right)\left(x_{j+1} \mid x_{j}\right) -\left(\widetilde{\mathcal{I}}_j^{\star,m-1} \mathbb{P}_j^m \right) \left(x_{j+1} \mid x_{j}\right)\right|\\
=  & \inner{\mathbb{P}_j^m (x_{j+1} \mid x_j, \cdot)}{\pi_j^{\star,m}(\cdot\mid x_j) - \pi_j^{\star,m-1}(\cdot\mid x_j)}\\
\leq & \norm{\mathbb{P}_j^m (x_{j+1} \mid x_j, \cdot)}_\infty \norm{\pi_j^{\star,m}(\cdot\mid x_j) - \pi_j^{\star,m-1}(\cdot\mid x_j)}_1 \\
\leq &  \norm{\pi_j^{\star,m}(\cdot\mid x_j) - \pi_j^{\star,m-1}(\cdot\mid x_j)}_1.
\end{align*}
Thus, by applying the above inequality to \eqref{eq: prob difference in pi}, we obtain
\begin{align*}
& \left|\widetilde{\mathcal{I}}_1^{\star,m} \mathbb{P}_1^m \cdots \widetilde{\mathcal{I}}_j^{\star,m} \mathbb{P}_j^m \cdots \widetilde{\mathcal{I}}_{h-1}^{\star,m} \mathbb{P}_{h-1}^m (x_h\mid x_1) \right.\\
&\left.- \widetilde{\mathcal{I}}_1^{\star,m} \mathbb{P}_1^m \cdots \widetilde{\mathcal{I}}_j^{\star,m-1} \mathbb{P}_j^{m} \cdots \widetilde{\mathcal{I}}_{h-1}^{\star,m} \mathbb{P}_{h-1}^m (x_h\mid x_1) \right| \\
\leq & \max_{x\in\mathcal{S}}\norm{\pi_j^{\star,m}(\cdot \mid x) -\pi_j^{\star,m-1}(\cdot \mid x)}_1 \\
&\cdot \sum_{x_{2}, \ldots, x_{j},x_{j+2},\ldots, x_{h-1}} \max_{x_{j+1}\in\mathcal{S}} \prod_{i \in[h-1] \backslash\{j\}} \widetilde{\mathcal{I}}_i^{\star,m} \mathbb{P}_i^m \left(x_{i+1} \mid x_{i}\right) \\
=& \max_{x\in\mathcal{S}}\norm{\pi_j^{\star,m}(\cdot \mid x) -\pi_j^{\star,m-1}(\cdot \mid x)}_1 \\
 \nonumber &\cdot \underbrace{\sum_{x_{j+2}, \ldots, x_{h-1}} \max _{x_{j+1} \in \mathcal{S}} \prod_{i=j+1}^{h-1}  \widetilde{\mathcal{I}}_i^{\star,m} \mathbb{P}_i^m \left(x_{i+1} \mid x_{i}\right)}_{\leq 1} \\
 &\cdot \underbrace{\sum_{x_{2}, \ldots, x_{j}} \prod_{i=1}^{j-1}  \widetilde{\mathcal{I}}_i^{\star,m} \mathbb{P}_i^m\left(x_{i+1} \mid x_{i}\right)}_{=1}\\
\leq & \max_{x\in\mathcal{S}}\norm{\pi_j^{\star,m}(\cdot \mid x) -\pi_j^{\star,m-1}(\cdot \mid x)}_1.
\end{align*}

Similarly, by noting that 
\begin{align*}
&\left|\left(\widetilde{\mathcal{I}}_j^{\star,m} \mathbb{P}_j^m \right)\left(x_{j+1} \mid x_{j}\right) -\left(\widetilde{\mathcal{I}}_j^{\star,m} \mathbb{P}_j^{m-1} \right) \left(x_{j+1} \mid x_{j}\right)\right|\\
= &\left| \sum_{a\in\mathcal{A}} \pi^{\star,m}(a\mid x_j) \left( {\mathbb{P}_j^m (x_{j+1} \mid x_j, a)-\mathbb{P}_j^{m-1} (x_{j+1} \mid x_j, a)} \right) \right|\\
\leq & \max_{a\in\mathcal{A}} \left| {\mathbb{P}_j^m (x_{j+1} \mid x_j, a)-\mathbb{P}_j^{m-1} (x_{j+1} \mid x_j, a)} \right|\\
\leq & \max_{a\in\mathcal{A}} \inner{\psi(x_j,a,x_{j+1})}{\theta_j^m-\theta_j^{m-1}}\\
\leq & \max_{a\in\mathcal{A}} \norm{\psi(x_j,a,x_{j+1})}_2\norm{\theta_j^m-\theta_j^{m-1}}_2\\
\leq & \sqrt{d_1} \norm{\theta_j^m-\theta_j^{m-1}}_2
\end{align*}
holds for every $x_{j+1} \in \mathcal{S}$, 
one can conclude that  
\begin{align*}
&\left|\widetilde{\mathcal{I}}_1^{\star,m} \mathbb{P}_1^m \cdots \widetilde{\mathcal{I}}_j^{\star,m} \mathbb{P}_j^m \cdots \widetilde{\mathcal{I}}_{h-1}^{\star,m} \mathbb{P}_{h-1}^m(x_h\mid x_1) \right.\\
&\left.- \widetilde{\mathcal{I}}_1^{\star,m} \mathbb{P}_1^m \cdots \widetilde{\mathcal{I}}_j^{\star,m} \mathbb{P}_j^{m-1} \cdots \widetilde{\mathcal{I}}_{h-1}^{\star,m} \mathbb{P}_{h-1}^m (x_h\mid x_1) \right|\\
&\leq \sqrt{d_1} \norm{\theta_j^m-\theta_j^{m-1}}_2.
\end{align*}
This completes the proof.
\end{proof}

\begin{lemma} \label{lemma: probability difference in expectation}
Under Assumption \ref{ass: linear fun approx}, it holds that
\begin{align*}
\sum_{m=2}^M \sum_{h=1}^{H}  \left( \mathbb{E}_{\pi^{\star,m}, \mathbb{P}^{m}}- \mathbb{E}_{\pi^{\star,m-1}, \mathbb{P}^{m-1}} \right) \left[\mathbb{1}(x_h)\right] \leq & H \left(\sqrt{d_1} B_\mathbb{P}+  B_\star \right),
\end{align*}
where $ \mathbb{1}(x_h)$ denotes the indicator function for the state $x_h$ and $d_1>0$ is a constant defined in Assumption \ref{ass: linear fun approx}.
\end{lemma}

\begin{proof}
 For every $(m,h) \in [M] \text{ and } [H]$, we have
\begin{align} \label{eq: expectation difference 2}
\nonumber&\left|\left(\mathbb{E}_{\pi^{\star,m}, \mathbb{P}^{m}} - \mathbb{E}_{\pi^{\star,m-1}, \mathbb{P}^{m-1}}\right) \left[ \mathbb{1}(x_h)\right]\right|\\
\nonumber\leq & \left|\widetilde{\mathcal{I}}_1^{\star,m} \mathbb{P}_1^m  \widetilde{\mathcal{I}}_2^{\star,m} \mathbb{P}_2^m \cdots \widetilde{\mathcal{I}}_{h-1}^{\star,m} \mathbb{P}_{h-1}^m (x_h\mid x_1) \right.\\
\nonumber &\left.- \widetilde{\mathcal{I}}_1^{\star,m-1} \mathbb{P}_1^{m-1} \widetilde{\mathcal{I}}_2^{\star,m-1} \mathbb{P}_2^{m-1} \cdots \widetilde{\mathcal{I}}_{h-1}^{\star,m-1} \mathbb{P}_{h-1}^{m-1} (x_h\mid x_1)\right| \\
\nonumber\leq & \left|\widetilde{\mathcal{I}}_1^{\star,m} \mathbb{P}_1^m  \widetilde{\mathcal{I}}_2^{\star,m} \mathbb{P}_2^m \cdots \widetilde{\mathcal{I}}_{h-1}^{\star,m} \mathbb{P}_{h-1}^m (x_h\mid x_1)  \right.\\
\nonumber&\left.- \widetilde{\mathcal{I}}_1^{\star,m} \mathbb{P}_1^{m-1} \widetilde{\mathcal{I}}_2^{\star,m-1} \mathbb{P}_2^{m-1} \cdots \widetilde{\mathcal{I}}_{h-1}^{\star,m-1} \mathbb{P}_{h-1}^{m-1} (x_h\mid x_1)\right| +\\
\nonumber&\left| \widetilde{\mathcal{I}}_1^{\star,m} \mathbb{P}_1^{m-1} \widetilde{\mathcal{I}}_2^{\star,m-1} \mathbb{P}_2^{m-1} \cdots \widetilde{\mathcal{I}}_{h-1}^{\star,m-1} \mathbb{P}_{h-1}^{m-1} (x_h\mid x_1)\right.\\
\nonumber &\left.-\widetilde{\mathcal{I}}_1^{\star,m-1} \mathbb{P}_1^{m-1} \widetilde{\mathcal{I}}_2^{\star,m-1} \mathbb{P}_2^{m-1} \cdots \widetilde{\mathcal{I}}_{h-1}^{\star,m-1} \mathbb{P}_{h-1}^{m-1} (x_h\mid x_1)\right| \\
\nonumber \leq & \sqrt{d_1} \sum_{j=1}^h\norm{\theta_j^m-\theta_j^{m-1}}_2  \\
&+ \sum_{j=1}^h \max_{x\in\mathcal{S}}\norm{\pi_j^{\star,m}(\cdot \mid x) -\pi_j^{\star,m-1}(\cdot \mid x)}_1
\end{align}
where the last inequality follows from further telescoping the first term in the second inequality and then applying Lemma \ref{lemma: prob difference in pi}. Taking the summation of \eqref{eq: expectation difference 2} over $(m,h)\in[M]\times [H]$, we obtain the desired result.
\end{proof}

\begin{lemma} \label{lemma: probability difference in expectation tabular}
In the tabular setting,  it holds that
\begin{align*}
\sum_{m=2}^M \sum_{h=1}^{H}  \left( \mathbb{E}_{\pi^{\star,m}, \mathbb{P}^{m}}- \mathbb{E}_{\pi^{\star,m-1}, \mathbb{P}^{m-1}} \right) \left[\mathbb{1}(x_h)\right] \leq & H \left(B_\mathbb{P}+  B_\star \right),
\end{align*}
where $ \mathbb{1}(x_h)$ denotes the indicator function for the state $x_h$.
\end{lemma}

\begin{proof}
Note that in Section \nameref{sec:appe_PE}, we define 
\begin{align*}
   & d_1=|\mathcal{S}|^2|\mathcal{A}|, \ \psi(x,a,x^\prime)=\boldsymbol{e}_{x,a,x'} \in \mathbb{R}^{d_1}, \ \theta_h^m =\mathbb{P}_h^m(\cdot,\cdot,\cdot) \in \mathbb{R}^{d_1}
\end{align*}
for the tabular case. The rest of the proof follows from that of \ref{lemma: probability difference in expectation}.

\end{proof}

\subsection{Concentration of Self-normalized Processes}
\begin{lemma}(\textbf{Theorem 1 in \cite{abbasi2011improved}}) \label{lemma: Concentration of Self-normalized Processes}
Let $\left\{\mathcal{F}_{t}\right\}_{t=0}^{\infty}$ be a filtration and $\left\{\eta_{t}\right\}_{t=1}^{\infty}$ be a $\mathbb{R}$-valued stochastic process such that $\eta_{t}$ is $\mathcal{F}_{t}$-measurable for every $t \geq 0$. Assume that for every $t \geq 0$, conditioning on $\mathcal{F}_{t}, \eta_{t}$ is a zero-mean and $\sigma$-subGaussian random variable with the variance proxy $\sigma^{2}>0$, i.e., $\mathbb{E}\left[e^{\lambda \eta_{t}} \mid \mathcal{F}_{t}\right] \leq e^{\lambda^{2} \sigma^{2} / 2}$ for every $\lambda \in \mathbb{R}$. Let $\left\{X_{t}\right\}_{t=1}^{\infty}$ be an $\mathbb{R}^{d}$-valued stochastic process such that $X_{t}$ is $\mathcal{F}_{t}$-measurable for every $t \geq 0$. Let $Y \in \mathbb{R}^{d \times d}$ be a deterministic and positive-definite matrix. For every $t \geq 0$, we define
$$
\bar{Y}_{t}:=Y+\sum_{\tau=1}^{t} X_{\tau} X_{\tau}^{\top} \text { and } S_{t}=\sum_{\tau=1}^{t} \eta_{\tau} X_{\tau} .
$$
Then, for every fixed $\delta \in(0,1)$, it holds with probability at least $1-\delta$ that
$$
\left\|S_{t}\right\|_{\left(\bar{Y}_{t}\right)^{-1}}^{2} \leq 2 \sigma^{2} \log \left(\frac{\operatorname{det}\left(\bar{Y}_{t}\right)^{1 / 2} \operatorname{det}(Y)^{-1 / 2}}{\delta}\right)
$$
for every $t \geq 0$.
\end{lemma}


\input{files/constraints_using_dual}

%% file: files/constraints_using_dual.tex
\subsection{Constraints violation under uniform Slater condition}


\begin{lemma}(\textbf{Boundedness of Sublevel Sets of the Dual Function  \citep[Section 8.5]{beck2017first}}). Let the Slater condition hold. Fix $C \in \mathbb{R}$. For every $\mu^m \in\{\mu \geq 0 \mid \mathcal{D}^m(\mu) \leq C\}$, it holds that
$$
\mu^m \leq \frac{1}{\gamma^m}\left(C-V_{r, 1}^{\bar{\pi},m}\left(x_{1}\right)\right)
$$
\end{lemma}

\begin{corollary}[Boundedness of $\mu^{\star,m}$]
If we take $C^m=V_{r, 1}^{\pi^{\star,m},m}\left(x_{1}\right)=\mathcal{D}^m\left(\mu^{\star,m}\right)$, then $\Lambda^{\star,m}=\{\mu \geq 0 \mid \mathcal{D}^m(\mu) \leq C^m\}$. Thus, for every $\mu^{\star,m} \in \Lambda^{\star,m}$, we have
$$
\mu^{\star,m} \leq \frac{1}{\gamma}\left(V_{r, 1}^{\pi^{\star,m},m}\left(x_{1}\right)-V_{r, 1}^{\bar{\pi},m}\left(x_{1}\right)\right) .
$$
\end{corollary}

By slightly extending  \cite[Proposition 3.60]{beck2017first} and \cite[Lemma 10]{ding2021provably}, the constraint violation $b-V_{g, 1}^{\pi}\left(x_{1}\right)$ can be bounded under the Slater condition as follows.

\begin{lemma}[Constraint Violation under Slater Condition] \label{lemma: Constraint Violation under Slater Condition}
Let the Slater condition hold and $\mu^{\star,m} \in \Lambda^{\star,m}$. Let $C^{\star,m} \geq 2 \mu^{\star,m}$. Assume that $\pi^m \in \Delta(\mathcal{A} \mid \mathcal{S}, H)$ satisfies
\begin{align} \label{eq: Constraint Violation under Slater Condition}
V_{r, 1}^{\pi^{\star,m},m}\left(x_{1}\right)-V_{r, 1}^{\pi,m}\left(x_{1}\right)+C^{\star,m}\left(b_m-V_{g, 1}^{\pi^m,m}\left(x_{1}\right)\right) \leq \delta.
\end{align}
Then,
$$
b_m-V_{g, 1}^{\pi^m,m}\left(x_{1}\right) \leq \frac{2 \delta}{C^{\star,m}}.
$$
\end{lemma}
\begin{proof}
Let
$$
\begin{aligned}
&v^m(\tau)\\
&=\operatorname{maximize}_{\pi \in \Delta(\mathcal{A} \mid \mathcal{S}, H)}\left\{V_{r, 1}^{\pi^m,m}\left(x_{1}\right) \mid V_{g, 1}^{\pi^m,m}\left(x_{1}\right) \geq b_m+\tau\right\}
\end{aligned}
$$
By definition, $v^m(0)=V_{r, 1}^{\pi^{\star,m},m}\left(x_{1}\right)$. It has been shown as a special case of \cite[Proposition 1]{paternain2019safe} that $v(\tau)$ is concave. By the Lagrangian and the strong duality,
\begin{align*}
\mathcal{L}^m\left(\pi, \mu^{\star,m}\right) \leq& \operatorname{maximize}_{\pi \in \Delta(\mathcal{A} \mid \mathcal{S}, H)} \mathcal{L}^m\left(\pi, \mu^{\star,m}\right)\\
=&\mathcal{D}^m\left(\mu^{\star,m}\right)\\
=&V_{r, 1}^{\pi^{\star,m}, m}\left(x_{1}\right)\\
=&v^m(0), \quad \text { for all } \pi \in \Delta(\mathcal{A} \mid \mathcal{S}, H).
\end{align*}

For every $\pi \in\left\{\pi \in \Delta(\mathcal{A} \mid \mathcal{S}, H) \mid V_{g, 1}^{\pi,m}\left(x_{1}\right) \geq b_m+\tau\right\}$
$$
\begin{aligned}
v^m(0)-\tau \mu^{\star,m} & \geq \mathcal{L}^m\left(\pi, \mu^{\star,m}\right)-\tau \mu^{\star,m} \\
&=V_{r, 1}^{\pi,m}\left(x_{1}\right)+\mu^{\star,m}\left(V_{g, 1}^{\pi, m}\left(x_{1}\right)-b_m\right)-\tau \mu^{\star,m} \\
&=V_{r, 1}^{\pi, m}\left(x_{1}\right)+\mu^{\star, m}\left(V_{g, 1}^{\pi, m}\left(x_{1}\right)-b_m-\tau\right) \\
& \geq V_{r, 1}^{\pi, m}\left(x_{1}\right) .
\end{aligned}
$$
If we maximize the right-hand side of above inequality over $\pi \in\left\{\pi \in \Delta(\mathcal{A} \mid \mathcal{S}, H) \mid V_{g, 1}^{\pi}\left(x_{1}\right) \geq b+\tau\right\}$, then
$$
v^m(\tau)- v^m(0) \leq  -\tau \mu^{\star, m}.
$$
On the other hand, if we take $\tau=\bar{\tau}:=-\left(b_m-V_{g, 1}^{\bar{\pi},m}\left(x_{1}\right)\right)$, then
$$
V_{r, 1}^{\bar{\pi},m}\left(x_{1}\right) \leq  v^m(\bar{\tau}).
$$
Combing the above two yields
$$
V_{r, 1}^{\bar{\pi}, m}\left(x_{1}\right)-V_{r, 1}^{\pi^{\star,m}, m}\left(x_{1}\right) \leq v^m(\bar{\tau})- v^m(0)\leq -\bar{\tau} \mu^{\star, m}.
$$

Then,
$$
\begin{aligned}
&\left(C^{\star, m}-\mu^{\star, m}\right)(-\bar{\tau})\\
&=\bar{\tau} \mu^{\star, m}-C^{\star, m}\bar{\tau} \\
& \leq V_{r, 1}^{\pi^{\star,m}, m}\left(x_{1}\right)-V_{r, 1}^{\bar{\pi}, m}\left(x_{1}\right)+C^{\star, m}\left(b_m-V_{g, 1}^{\bar{\pi}, m}\left(x_{1}\right)\right)\\
& \leq \delta.
\end{aligned}
$$
Thus,
$$
b_m-V_{g, 1}^{\bar{\pi},m}\left(x_{1}\right) \leq \frac{\delta}{C^{\star,m}-\mu^{\star, m}} \leq \frac{2 \delta}{C^{\star, m}}.
$$
\end{proof}

\begin{lemma}(\textbf{Constraint Violation under Uniform Slater Condition}) \label{corollary: Constraint Violation under Uniform Slater Condition}
Let the uniform Slater condition hold and $\mu^{\star,m} \in \Lambda^{\star,m}$. Let $\bar{C}^{\star} \geq 2 \max_{m \in [M]}\mu^{\star,m}$. Assume that $\{\pi^m\}_{m=1}^M$ satisfies
$$
\sum_{m=1}^M V_{r, 1}^{\pi^{\star,m},m}\left(x_{1}\right)-V_{r, 1}^{\pi^m,m}\left(x_{1}\right)+\bar{C}^{\star} \sum_{m=1}^M \left(b_m-V_{g, 1}^{\pi^m,m}\left(x_{1}\right)\right) \leq \delta.
$$
Then,
$$
\sum_{m=1}^M \left(b_m-V_{g, 1}^{\pi^m,m}\left(x_{1}\right) \right) \leq \frac{2 \delta}{\bar{C}^{\star}}.
$$
\end{lemma}
\begin{proof}
Denote   $\delta^m \coloneqq V_{r, 1}^{\pi^{\star,m}, m}\left(x_{1}\right)-V_{r, 1}^{{\pi^m}, m}\left(x_{1}\right)+\bar{C}^{\star}\left(b_m-V_{g, 1}^{{\pi^m}, m}\left(x_{1}\right)\right)$ for all $m\in[M]$.
From Lemma \ref{lemma: Constraint Violation under Slater Condition}, we know that 
\[b_m-V_{g, 1}^{\pi^m,m}\left(x_{1}\right)  \leq \frac{2 \delta^m}{\bar{C}^{\star}}.\]
By taking the summation over $m\in[M]$ on both sides of the above inequality, we obtain 
\begin{align*}
 \sum_{m=1}^M \left(b_m-V_{g, 1}^{\pi^m,m}\left(x_{1}\right) \right) \leq \frac{2 \sum_{m=1}^M \delta^m}{\bar{C}^{\star}} \leq \frac{2 \delta}{\bar{C}^{\star}}.
\end{align*}
\end{proof}

%% file: main.bbl
\begin{thebibliography}{48}
\providecommand{\natexlab}[1]{#1}

\bibitem[{Abbasi-Yadkori, P{\'a}l, and
  Szepesv{\'a}ri(2011)}]{abbasi2011improved}
Abbasi-Yadkori, Y.; P{\'a}l, D.; and Szepesv{\'a}ri, C. 2011.
\newblock Improved algorithms for linear stochastic bandits.
\newblock \emph{Advances in neural information processing systems}, 24:
  2312--2320.

\bibitem[{Altman(1999)}]{altman1999constrained}
Altman, E. 1999.
\newblock \emph{Constrained {M}arkov decision processes}, volume~7.
\newblock CRC Press.

\bibitem[{Amodei et~al.(2016)Amodei, Olah, Steinhardt, Christiano, Schulman,
  and Man{\'e}}]{amodei2016concrete}
Amodei, D.; Olah, C.; Steinhardt, J.; Christiano, P.; Schulman, J.; and
  Man{\'e}, D. 2016.
\newblock Concrete problems in {AI} safety.
\newblock \emph{arXiv preprint arXiv:1606.06565}.

\bibitem[{Auer, Gajane, and Ortner(2019)}]{auer2019adaptively}
Auer, P.; Gajane, P.; and Ortner, R. 2019.
\newblock Adaptively tracking the best bandit arm with an unknown number of
  distribution changes.
\newblock In \emph{Conference on Learning Theory}, 138--158. PMLR.

\bibitem[{Ayoub et~al.(2020)Ayoub, Jia, Szepesvari, Wang, and
  Yang}]{ayoub2020model}
Ayoub, A.; Jia, Z.; Szepesvari, C.; Wang, M.; and Yang, L. 2020.
\newblock Model-based reinforcement learning with value-targeted regression.
\newblock In \emph{International Conference on Machine Learning}, 463--474.
  PMLR.

\bibitem[{Azar, Osband, and Munos(2017)}]{azar2017minimax}
Azar, M.~G.; Osband, I.; and Munos, R. 2017.
\newblock Minimax regret bounds for reinforcement learning.
\newblock In \emph{International Conference on Machine Learning}, 263--272.
  PMLR.

\bibitem[{Beck(2017)}]{beck2017first}
Beck, A. 2017.
\newblock \emph{First-order methods in optimization}.
\newblock SIAM.

\bibitem[{Besbes, Gur, and Zeevi(2015)}]{besbes2015non}
Besbes, O.; Gur, Y.; and Zeevi, A. 2015.
\newblock Non-stationary stochastic optimization.
\newblock \emph{Operations research}, 63(5): 1227--1244.

\bibitem[{Bradtke and Barto(1996)}]{bradtke1996linear}
Bradtke, S.~J.; and Barto, A.~G. 1996.
\newblock Linear least-squares algorithms for temporal difference learning.
\newblock \emph{Machine learning}, 22(1): 33--57.

\bibitem[{Cai et~al.(2020)Cai, Yang, Jin, and Wang}]{cai2020provably}
Cai, Q.; Yang, Z.; Jin, C.; and Wang, Z. 2020.
\newblock Provably efficient exploration in policy optimization.
\newblock In \emph{International Conference on Machine Learning}, 1283--1294.
  PMLR.

\bibitem[{Cao and Liu(2018)}]{cao2018online}
Cao, X.; and Liu, K.~R. 2018.
\newblock Online convex optimization with time-varying constraints and bandit
  feedback.
\newblock \emph{IEEE Transactions on automatic control}, 64(7): 2665--2680.

\bibitem[{Chandak et~al.(2020)Chandak, Theocharous, Shankar, White, Mahadevan,
  and Thomas}]{chandak2020optimizing}
Chandak, Y.; Theocharous, G.; Shankar, S.; White, M.; Mahadevan, S.; and
  Thomas, P. 2020.
\newblock Optimizing for the future in non-stationary {MDP}s.
\newblock In \emph{International Conference on Machine Learning}, 1414--1425.
  PMLR.

\bibitem[{Cheung, Simchi-Levi, and Zhu(2020)}]{cheung2020reinforcement}
Cheung, W.~C.; Simchi-Levi, D.; and Zhu, R. 2020.
\newblock Reinforcement learning for non-stationary {M}arkov decision
  processes: The blessing of (more) optimism.
\newblock In \emph{International Conference on Machine Learning}, 1843--1854.
  PMLR.

\bibitem[{Ding et~al.(2021)Ding, Wei, Yang, Wang, and
  Jovanovic}]{ding2021provably}
Ding, D.; Wei, X.; Yang, Z.; Wang, Z.; and Jovanovic, M. 2021.
\newblock Provably efficient safe exploration via primal-dual policy
  optimization.
\newblock In \emph{International Conference on Artificial Intelligence and
  Statistics}, 3304--3312. PMLR.

\bibitem[{Ding et~al.(2020)Ding, Zhang, Basar, and Jovanovic}]{ding2020natural}
Ding, D.; Zhang, K.; Basar, T.; and Jovanovic, M.~R. 2020.
\newblock Natural Policy Gradient Primal-Dual Method for Constrained {M}arkov
  Decision Processes.
\newblock In \emph{NeurIPS}.

\bibitem[{Domingues et~al.(2021)Domingues, M{\'e}nard, Pirotta, Kaufmann, and
  Valko}]{domingues2021kernel}
Domingues, O.~D.; M{\'e}nard, P.; Pirotta, M.; Kaufmann, E.; and Valko, M.
  2021.
\newblock A kernel-based approach to non-stationary reinforcement learning in
  metric spaces.
\newblock In \emph{International Conference on Artificial Intelligence and
  Statistics}, 3538--3546. PMLR.

\bibitem[{Dulac-Arnold, Mankowitz, and Hester(2019)}]{dulac2019challenges}
Dulac-Arnold, G.; Mankowitz, D.; and Hester, T. 2019.
\newblock Challenges of real-world reinforcement learning.
\newblock \emph{arXiv preprint arXiv:1904.12901}.

\bibitem[{Efroni, Mannor, and Pirotta(2020)}]{efroni2020exploration}
Efroni, Y.; Mannor, S.; and Pirotta, M. 2020.
\newblock Exploration-exploitation in constrained {MDP}s.
\newblock \emph{arXiv preprint arXiv:2003.02189}.

\bibitem[{Fei et~al.(2020)Fei, Yang, Wang, and Xie}]{fei2020dynamic}
Fei, Y.; Yang, Z.; Wang, Z.; and Xie, Q. 2020.
\newblock Dynamic regret of policy optimization in non-stationary environments.
\newblock \emph{arXiv preprint arXiv:2007.00148}.

\bibitem[{Garc{\i}a and Fern{\'a}ndez(2015)}]{garcia2015comprehensive}
Garc{\i}a, J.; and Fern{\'a}ndez, F. 2015.
\newblock A comprehensive survey on safe reinforcement learning.
\newblock \emph{Journal of Machine Learning Research}, 16(1): 1437--1480.

\bibitem[{Hall and Willett(2013)}]{hall2013dynamical}
Hall, E.; and Willett, R. 2013.
\newblock Dynamical models and tracking regret in online convex programming.
\newblock In \emph{International Conference on Machine Learning}, 579--587.
  PMLR.

\bibitem[{Hall and Willett(2015)}]{hall2015online}
Hall, E.~C.; and Willett, R.~M. 2015.
\newblock Online convex optimization in dynamic environments.
\newblock \emph{IEEE Journal of Selected Topics in Signal Processing}, 9(4):
  647--662.

\bibitem[{Jaksch, Ortner, and Auer(2010)}]{jaksch2010near}
Jaksch, T.; Ortner, R.; and Auer, P. 2010.
\newblock Near-optimal Regret Bounds for Reinforcement Learning.
\newblock \emph{Journal of Machine Learning Research}, 11(4).

\bibitem[{Jin et~al.(2018)Jin, Allen-Zhu, Bubeck, and Jordan}]{jin2018q}
Jin, C.; Allen-Zhu, Z.; Bubeck, S.; and Jordan, M.~I. 2018.
\newblock Is {Q}-learning provably efficient?
\newblock \emph{arXiv preprint arXiv:1807.03765}.

\bibitem[{Jin et~al.(2020)Jin, Yang, Wang, and Jordan}]{jin2020provably}
Jin, C.; Yang, Z.; Wang, Z.; and Jordan, M.~I. 2020.
\newblock Provably efficient reinforcement learning with linear function
  approximation.
\newblock In \emph{Conference on Learning Theory}, 2137--2143. PMLR.

\bibitem[{Kakade and Langford(2002)}]{kakade2002approximately}
Kakade, S.; and Langford, J. 2002.
\newblock Approximately optimal approximate reinforcement learning.
\newblock In \emph{In Proc. 19th International Conference on Machine Learning}.
  Citeseer.

\bibitem[{Kakade(2001)}]{kakade2001natural}
Kakade, S.~M. 2001.
\newblock A natural policy gradient.
\newblock \emph{Advances in neural information processing systems}, 14.

\bibitem[{Lazaric, Ghavamzadeh, and Munos(2010)}]{lazaric2010finite}
Lazaric, A.; Ghavamzadeh, M.; and Munos, R. 2010.
\newblock Finite-sample analysis of LSTD.
\newblock In \emph{ICML-27th International Conference on Machine Learning},
  615--622.

\bibitem[{Liu et~al.(2021)Liu, Zhou, Kalathil, Kumar, and
  Tian}]{liu2021learning}
Liu, T.; Zhou, R.; Kalathil, D.; Kumar, P.; and Tian, C. 2021.
\newblock Learning Policies with Zero or Bounded Constraint Violation for
  Constrained {MDP}s.
\newblock \emph{arXiv preprint arXiv:2106.02684}.

\bibitem[{Mao et~al.(2020)Mao, Zhang, Zhu, Simchi-Levi, and
  Basar}]{mao2020near}
Mao, W.; Zhang, K.; Zhu, R.; Simchi-Levi, D.; and Basar, T. 2020.
\newblock Model-Free Non-Stationary RL: Near-Optimal Regret and Applications in
  Multi-Agent RL and Inventory Control.
\newblock \emph{https://arxiv.org/abs/2010.03161}.

\bibitem[{Modi et~al.(2020)Modi, Jiang, Tewari, and Singh}]{modi2020sample}
Modi, A.; Jiang, N.; Tewari, A.; and Singh, S. 2020.
\newblock Sample complexity of reinforcement learning using linearly combined
  model ensembles.
\newblock In \emph{International Conference on Artificial Intelligence and
  Statistics}, 2010--2020. PMLR.

\bibitem[{Moore et~al.(2014)Moore, Pyeatt, Kulkarni, Panousis, Padrez, and
  Doufas}]{moore2014reinforcement}
Moore, B.~L.; Pyeatt, L.~D.; Kulkarni, V.; Panousis, P.; Padrez, K.; and
  Doufas, A.~G. 2014.
\newblock Reinforcement learning for closed-loop propofol anesthesia: a study
  in human volunteers.
\newblock \emph{The Journal of Machine Learning Research}, 15(1): 655--696.

\bibitem[{Ortner, Gajane, and Auer(2020)}]{ortner2020variational}
Ortner, R.; Gajane, P.; and Auer, P. 2020.
\newblock Variational regret bounds for reinforcement learning.
\newblock In \emph{Uncertainty in Artificial Intelligence}, 81--90. PMLR.

\bibitem[{Paternain et~al.(2019)Paternain, Calvo-Fullana, Chamon, and
  Ribeiro}]{paternain2019safe}
Paternain, S.; Calvo-Fullana, M.; Chamon, L.~F.; and Ribeiro, A. 2019.
\newblock Safe policies for reinforcement learning via primal-dual methods.
\newblock \emph{arXiv preprint arXiv:1911.09101}.

\bibitem[{Qiu et~al.(2020)Qiu, Wei, Yang, Ye, and Wang}]{qiu2020upper}
Qiu, S.; Wei, X.; Yang, Z.; Ye, J.; and Wang, Z. 2020.
\newblock Upper confidence primal-dual reinforcement learning for {CMDP} with
  adversarial loss.
\newblock \emph{arXiv preprint arXiv:2003.00660}.

\bibitem[{Sallab et~al.(2017)Sallab, Abdou, Perot, and
  Yogamani}]{sallab2017deep}
Sallab, A.~E.; Abdou, M.; Perot, E.; and Yogamani, S. 2017.
\newblock Deep reinforcement learning framework for autonomous driving.
\newblock \emph{Electronic Imaging}, 2017(19): 70--76.

\bibitem[{Schulman et~al.(2015)Schulman, Levine, Abbeel, Jordan, and
  Moritz}]{schulman2015trust}
Schulman, J.; Levine, S.; Abbeel, P.; Jordan, M.; and Moritz, P. 2015.
\newblock Trust region policy optimization.
\newblock In \emph{International conference on machine learning}, 1889--1897.
  PMLR.

\bibitem[{Schulman et~al.(2017)Schulman, Wolski, Dhariwal, Radford, and
  Klimov}]{schulman2017proximal}
Schulman, J.; Wolski, F.; Dhariwal, P.; Radford, A.; and Klimov, O. 2017.
\newblock Proximal policy optimization algorithms.
\newblock \emph{arXiv preprint arXiv:1707.06347}.

\bibitem[{Singh, Gupta, and Shroff(2020)}]{singh2020learning}
Singh, R.; Gupta, A.; and Shroff, N.~B. 2020.
\newblock Learning in {M}arkov decision processes under constraints.
\newblock \emph{arXiv preprint arXiv:2002.12435}.

\bibitem[{Touati and Vincent(2020)}]{touati2020efficient}
Touati, A.; and Vincent, P. 2020.
\newblock Efficient learning in non-stationary linear {M}arkov decision
  processes.
\newblock \emph{arXiv preprint arXiv:2010.12870}.

\bibitem[{Wei and Luo(2021)}]{wei2021non}
Wei, C.-Y.; and Luo, H. 2021.
\newblock Non-stationary Reinforcement Learning without Prior Knowledge: An
  Optimal Black-box Approach.
\newblock \emph{arXiv preprint arXiv:2102.05406}.

\bibitem[{Yang and Wang(2019)}]{yang2019sample}
Yang, L.; and Wang, M. 2019.
\newblock Sample-optimal parametric {Q}-learning using linearly additive
  features.
\newblock In \emph{International Conference on Machine Learning}, 6995--7004.
  PMLR.

\bibitem[{Yang and Wang(2020)}]{yang2020reinforcement}
Yang, L.; and Wang, M. 2020.
\newblock Reinforcement learning in feature space: Matrix bandit, kernels, and
  regret bound.
\newblock In \emph{International Conference on Machine Learning}, 10746--10756.
  PMLR.

\bibitem[{Ying, Ding, and Lavaei(2021)}]{ying2021dual}
Ying, D.; Ding, Y.; and Lavaei, J. 2021.
\newblock A Dual Approach to Constrained {M}arkov Decision Processes with
  Entropy Regularization.
\newblock \emph{arXiv preprint arXiv:2110.08923}.

\bibitem[{Yu et~al.(2019)Yu, Yang, Kolar, and Wang}]{yu2019convergent}
Yu, M.; Yang, Z.; Kolar, M.; and Wang, Z. 2019.
\newblock Convergent policy optimization for safe reinforcement learning.
\newblock \emph{Advances in Neural Information Processing Systems}, 32:
  3127--3139.

\bibitem[{Zhong, Yang, and Szepesv{\'a}ri(2021)}]{zhong2021optimistic}
Zhong, H.; Yang, Z.; and Szepesv{\'a}ri, Z. W.~C. 2021.
\newblock Optimistic Policy Optimization is Provably Efficient in
  Non-stationary {MDP}s.
\newblock \emph{arXiv preprint arXiv:2110.08984}.

\bibitem[{Zhou, He, and Gu(2021)}]{zhou2021provably}
Zhou, D.; He, J.; and Gu, Q. 2021.
\newblock Provably efficient reinforcement learning for discounted {MDP}s with
  feature mapping.
\newblock In \emph{International Conference on Machine Learning}, 12793--12802.
  PMLR.

\bibitem[{Zhou et~al.(2020)Zhou, Chen, Varshney, and
  Jagmohan}]{zhou2020nonstationary}
Zhou, H.; Chen, J.; Varshney, L.~R.; and Jagmohan, A. 2020.
\newblock Nonstationary reinforcement learning with linear function
  approximation.
\newblock \emph{arXiv preprint arXiv:2010.04244}.

\end{thebibliography}
